\documentclass[letterpaper,onecolumn]{IEEEtran}
\usepackage{amsmath,amsfonts}
\usepackage{algorithmic}
\usepackage{algorithm}
\usepackage{array}
\usepackage{textcomp}
\usepackage{stfloats}
\usepackage{bbold}
\usepackage{url}
\usepackage{caption}
\usepackage{subcaption}
\usepackage{verbatim}
\usepackage{graphicx}
\usepackage{cite}
\usepackage{algorithm}
\usepackage{algorithmic}

\usepackage{float}

\usepackage{amsfonts,amsmath,amsthm, mathtools}
\usepackage{amssymb}
\usepackage{caption}
\usepackage{verbatim}
\usepackage{mathrsfs}
\usepackage{algorithm}
\usepackage{algorithmic}
\usepackage{comment}
\usepackage{multirow, tabularx, booktabs}
\usepackage[dvipsnames]{xcolor}

\usepackage{hyperref}
\usepackage{dsfont}

\newtheorem{theorem}{Theorem}

\newtheorem{corollary}{Corollary}[theorem]

\usepackage{fullpage}
\usepackage{xspace}

\numberwithin{equation}{section}

\theoremstyle{plain}

\newtheorem*{theorem*}{Theorem}

\theoremstyle{definition}

\usepackage[margin=1in]{geometry}

\begin{document}

\title{Convex Clustering Redefined: Robust Learning with the Median of Means Estimator}

\author{
    Sourav De \IEEEauthorrefmark{1}, \IEEEauthorblockN{Koustav Chowdhury\IEEEauthorrefmark{1}\thanks{Sourav De, Koustav Chowdhury and Bibhabasu Mandal contributed equally to this work.}, Bibhabasu Mandal\IEEEauthorrefmark{1}, Sagar Ghosh\IEEEauthorrefmark{2}, Swagatam Das\IEEEauthorrefmark{3}, Debolina Paul\IEEEauthorrefmark{4}, and Saptarshi Chakraborty\IEEEauthorrefmark{5},} \\
    \IEEEauthorblockA{
    \IEEEauthorrefmark{1}Indian Statistical Institute, Kolkata\\
    \IEEEauthorrefmark{2}Department of Statistics and Data Science, University of Texas at Austin\\
    \IEEEauthorrefmark{3}Electronics and Communication Sciences, Indian Statistical Institute, Kolkata\\
    \IEEEauthorrefmark{4} Department of Statistics, University of Oxford\\
    \IEEEauthorrefmark{5} Department of Statistics, University of Michigan}\\
    
    {desourav02@gmail.com}, {koustavchowdhury2003@gmail.com}, {bibhabasumandal04@gmail.com}, {sagarghosh1729@utexas.edu},{swagatam.das@isical.ac.in}, {ddebolina.paul@stats.ox.ac.uk}, {saptarsc@umich.edu}
    \\
}



\maketitle




\begin{abstract}
Clustering approaches that utilize convex loss functions have recently attracted growing interest in the formation of compact data clusters. Although classical methods like $k-$means and its wide family of variants are still widely used, all of them require the number of clusters ($k$) to be supplied as input and many are notably sensitive to initialization. Convex clustering provides a more stable alternative by formulating the clustering task as a convex optimization problem, ensuring a unique global solution. However, it faces challenges in handling high-dimensional data, especially in the presence of noise and outliers. Additionally, strong fusion regularization, controlled by the tuning parameter, can hinder effective cluster formation within a convex clustering framework. To overcome these challenges, we introduce a robust approach that integrates convex clustering with the Median of Means (MoM) estimator, thus developing an outlier-resistant and efficient clustering framework that does not necessitate a prior knowledge of the number of clusters. By leveraging the robustness of MoM alongside the stability of convex clustering, our method enhances both performance and efficiency, especially on large-scale datasets. Theoretical analysis demonstrates weak consistency under specific conditions, while experiments on synthetic and real-world datasets validate the method’s superior performance compared to existing approaches.
\end{abstract}



\noindent Github Repository: \url{https://tinyurl.com/2v3dx75x}{}\\
Benchmark Dataset (CC BY-NC-ND 4.0): \url{https://tinyurl.com/2zatwkf3}{}\\
ASU Datasets (GPLv2): \url{https://tinyurl.com/49n36ume}{}\\
Micro-array Datasets: \url{https://tinyurl.com/2f2pjz7j}{}\\
Brain Dataset: \url{https://tinyurl.com/4ntav7b9}{}\\
Wisconsin Dataset: \url{https://tinyurl.com/58wxjha5}{}\\

\section{Introduction}

Clustering is a fundamental task in unsupervised learning, aiming to organize unlabeled data into coherent groups for better interpretation and downstream applications. It plays a critical role in diverse areas such as customer segmentation \cite{kansal2018customer}, image analysis \cite{munz2007traffic}, and anomaly detection \cite{coleman1979image}. Traditional algorithms, such as $k$-means, approach clustering as a non-convex optimization problem \cite{lu2016statistical}, typically solved using greedy heuristics. Although computationally efficient and widely used, these methods suffer from several well-known limitations \cite{jain2010data}: they require pre-specifying the number of clusters \cite{tibshirani_estimating_2001,learning_k}, are sensitive to initialization \cite{10.1145/2395116.2395117,xu_power_2019}, and degrade in performance in high-dimensional spaces or when the data contains noise and outliers \cite{witten_framework_2010,de_amorim_survey_2016,chakraborty_detecting_2022}.

To overcome these challenges, convex relaxations of non-convex clustering problems have gained significant attention \cite{tropp_just_2006}. A prominent example is \textit{convex clustering} (or sum-of-norms clustering), which enjoys strong theoretical guarantees such as global optimality and convergence, while remaining broadly applicable in practice \cite{pelckmans_convex_2005,hocking_clusterpath_2011,lindsten_clustering_2011}. Given a data matrix $\mathbf{X} \in \mathbb{R}^{n \times d}$, where each row represents a data point in $d$-dimensional Euclidean space, convex clustering solves the following objective:
\begin{equation} \label{eqn_1:convex}
    \min_{\mathbf{u}} \frac{1}{2}\left[\|\mathbf{x}_{i.}-\mathbf{u}_{i.}\|_2^2 + 
    \gamma \sum_{i,j} \theta_{ij}\|\mathbf{u}_{i.} - \mathbf{u}_{j.}\|_p^2 \right],
\end{equation}
where $\mathbf{u}_{i.}$ denotes the $i$-th row of $\mathbf{u}$, $\theta_{ij}$ are edge weights, and $\|\cdot\|_p$ is the $\ell^p$ norm. The first term encourages each point to remain close to its centroid, while the second term (controlled by tuning parameter $\gamma>0$) promotes fusion across centroids, effectively determining the number of clusters \cite{chi2019recovering}. Although convex clustering is effective even in large-sample settings \cite{radchenko2017convex}, strong regularization can lead to undesirable merging of outliers with genuine clusters, especially in high-dimensional data \cite{feng2023review}.

This paper focuses on addressing the \textbf{robustness challenges} of clustering in the presence of noise and outliers. Robust methods can mitigate these effects by either discarding outlier features \cite{wang_robust_2016} or directly controlling the influence of anomalous data points. One powerful approach is the \textit{Median-of-Means (MoM)} estimator, which provides strong robustness and concentration guarantees under mild assumptions \cite{refer1,refer2,refer3,refer4,refer5}. Related work by \cite{paul_uniform_2021} further unifies robust center-based clustering under general dissimilarity measures.

Consider $n$ data points $\mathbf{x}_1, \dots, \mathbf{x}_n \in \mathbb{R}^d$ to be grouped into $k$ clusters. Each cluster is represented by a centroid $\boldsymbol{\theta}_j \in \mathbb{R}^d$, and the set of centroids forms a matrix $\mathbf{\Theta} \in \mathbb{R}^{k \times d}$. Using a Bregman divergence $d_\phi(\cdot, \cdot)$ as the dissimilarity measure, where $\phi:\mathbb{R}^d \to \mathbb{R}$ is a differentiable convex function, clustering can be formulated as
\begin{equation}
    f_{\boldsymbol{\Theta}}(\mathbf{x}_1, \dots, \mathbf{x}_n) = 
    \frac{1}{n} \sum_{i=1}^n 
    \Psi_\alpha(d_\phi(\mathbf{x}_i, \boldsymbol{\theta}_1), \dots, d_\phi(\mathbf{x}_i, \boldsymbol{\theta}_k)),
    \label{gen_obj}
\end{equation}
where $\Psi_\alpha:\mathbb{R}^k_{\geq 0}\to\mathbb{R}_{\geq 0}$ is a non-decreasing function, $\Psi_\alpha(0)=0$, and $\alpha$ is a hyperparameter. Different choices of $\phi$ and $\Psi_\alpha$ recover well-known clustering algorithms such as $k$-means, power $k$-means, and $k$-harmonic-means.

Instead of directly minimizing \eqref{gen_obj}, MoM partitions the data into $L$ disjoint subsets $B_1,\dots,B_L$, each containing $b$ samples, and optimizes a robust median-based objective:
\begin{equation}
    \text{MoM}^n_L(\boldsymbol{\Theta}) = 
    \text{Median}\left( 
    \frac{1}{b} \sum_{i \in B_1} f_{\boldsymbol{\Theta}}(\mathbf{x}_i), 
    \dots,
    \frac{1}{b} \sum_{i \in B_L} f_{\boldsymbol{\Theta}}(\mathbf{x}_i)
    \right).
\end{equation}

Because outliers typically contaminate only a fraction of the partitions, the median effectively suppresses their influence, ensuring robustness even in adversarial settings. Formal breakdown-point analyses of MoM estimators support this intuition \cite{refer6,refer7}.

To demonstrate our method, Figure \ref{fig:sidebyside} illustrates results on a benchmark dataset with $k=5$, $N=1000$, $\gamma=5000$, $20\%$ noise, and varying $\mu$. Outliers near the boundary are successfully isolated. When any pair $\mu_i,\mu_j$ exceeds the separation threshold $\mu$, the connecting edge is dropped, preventing spurious merging. While most outliers are identified, some deeply embedded points remain undetected.

\begin{figure*}[h!]
    \centering
    \begin{subfigure}[b]{0.245\textwidth}
        \includegraphics[width=\textwidth]{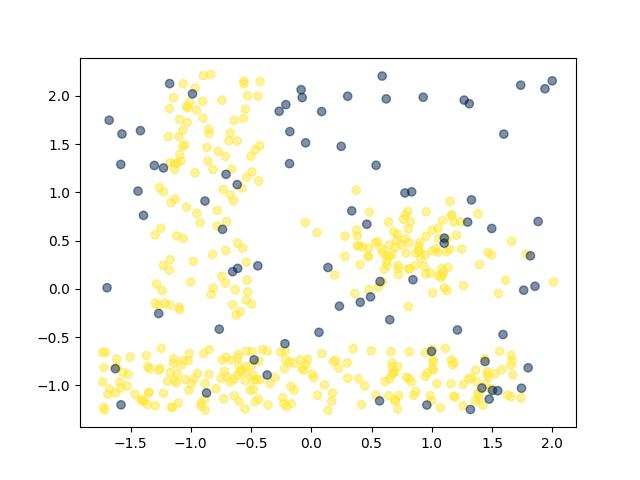}
        \caption{Dataset}
        \label{lsun_original}
    \end{subfigure}
    \begin{subfigure}[b]{0.245\textwidth}
        \includegraphics[width=\textwidth]{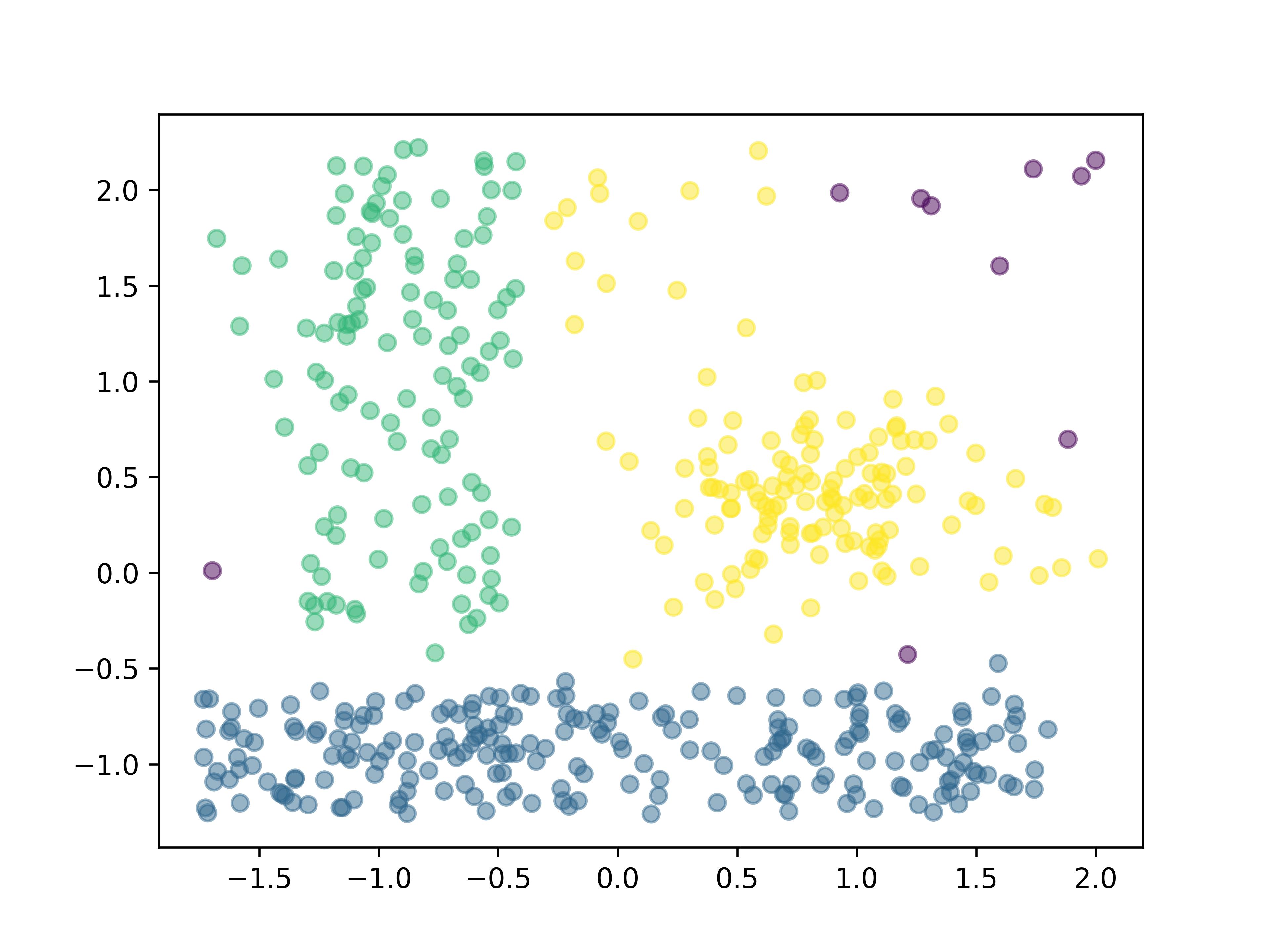}
        \caption{$\mu = 0.4$ }
        \label{lsun_1}
    \end{subfigure}
    \begin{subfigure}[b]{0.245\textwidth}
        \includegraphics[width=\textwidth]{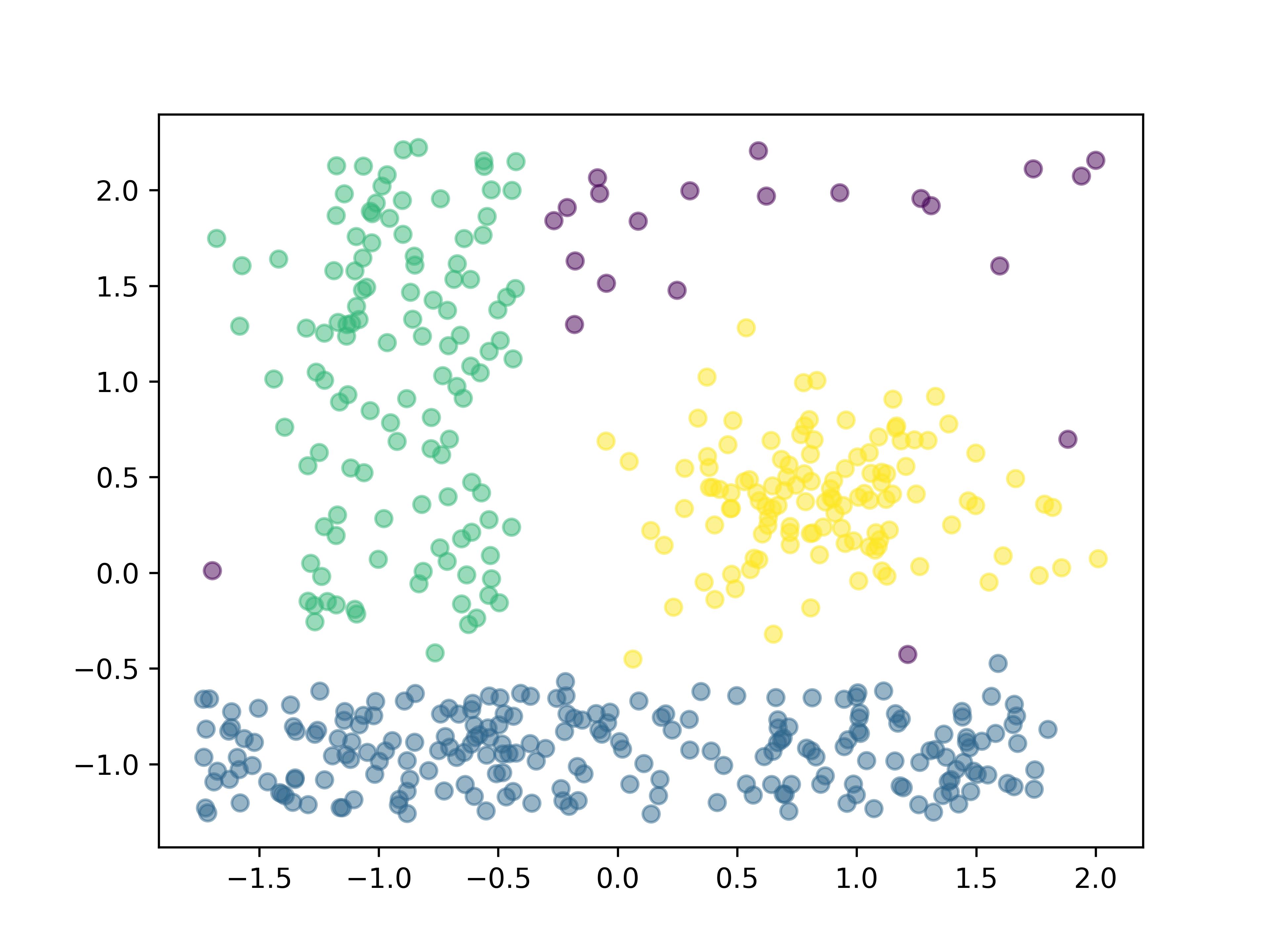}
        \caption{$\mu = 0.2$}
        \label{lsun_2}
    \end{subfigure}
    \begin{subfigure}[b]{0.245\textwidth}
        \includegraphics[width=\textwidth]{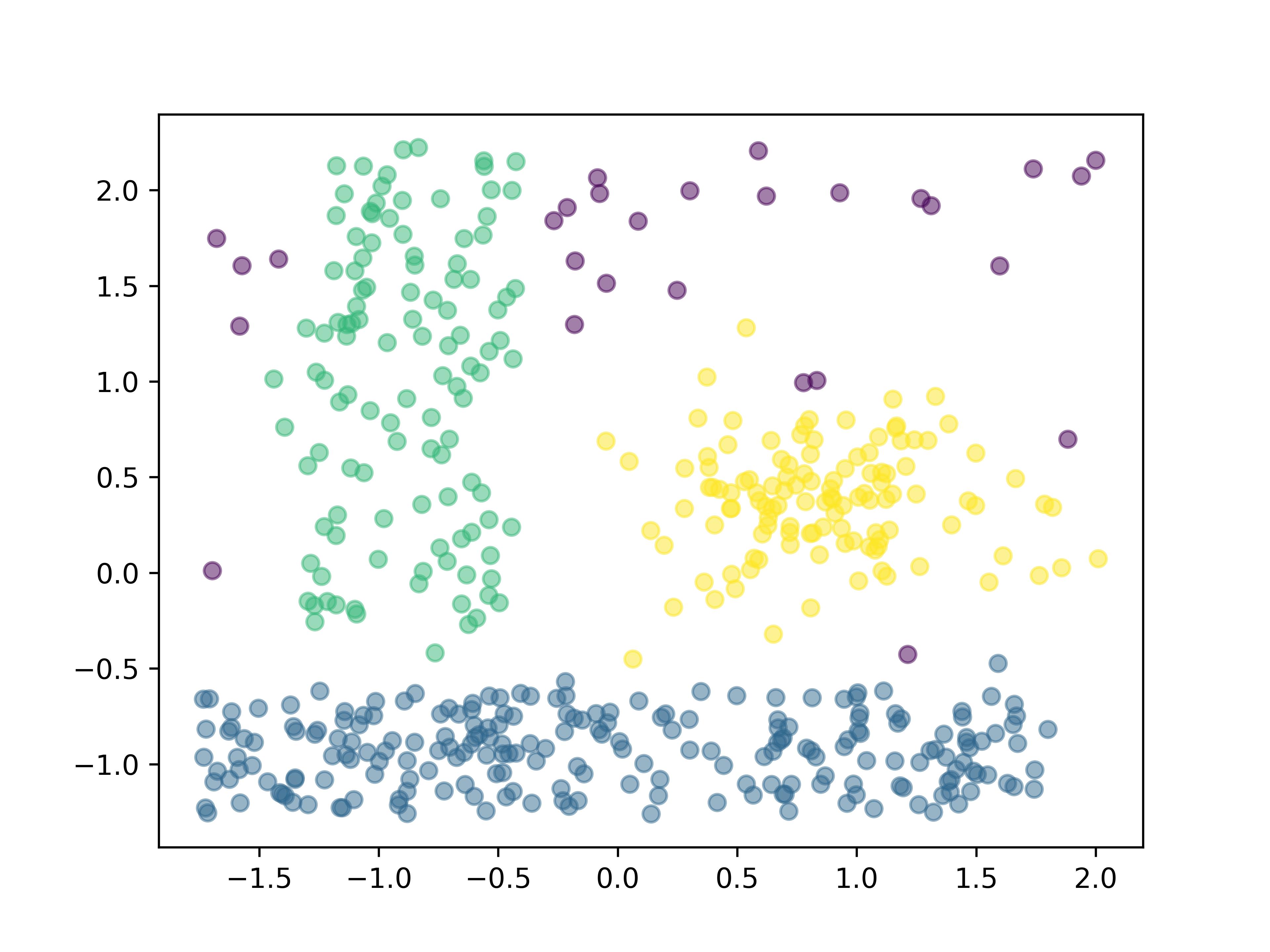}
        \caption{$\mu = 0.12$}
        \label{lsun_3}
    \end{subfigure}
    \caption{Figure \ref{lsun_original} shows the original dataset in yellow, with $20\%$ added noise represented by blue dots. As $\mu$ decreases, our method progressively identifies more noise points as outliers, which are marked by purple dots in Figures \ref{lsun_1}, \ref{lsun_2}, and \ref{lsun_3} respectively.}
    \label{fig:sidebyside}
\end{figure*}

\section{Contributions}

In this paper, we present a novel clustering framework that extends convex clustering with enhanced robustness using the Median-of-Means (MoM) estimator. Our main contributions are summarized below:
\begin{itemize}
    \item \textbf{Robust Convex Clustering Framework:} We propose a new clustering method that integrates the MoM estimator into the convex clustering paradigm, effectively mitigating the adverse impact of outliers and noisy data. We also develop a dedicated algorithm for the proposed framework, ensuring practical applicability and computational efficiency.
    
    \item \textbf{Theoretical Guarantees:} We establish uniform deviation bounds and concentration inequalities under standard regularity assumptions, providing strong theoretical reliability for our method.
    
    \item \textbf{Empirical Validation:} Extensive simulation studies demonstrate that our method consistently outperforms conventional clustering approaches, particularly in terms of robustness and efficiency under data contamination.
\end{itemize}

\section{Related Works}

\noindent\textbf{Convex Clustering and Semi-definite Programming:}
Since the introduction of Convex Clustering by \cite{pelckmans2005convex}, various extensions and perspectives have been explored \cite{lindsten2011clustering}, \cite{hocking2011clusterpath}, \cite{zhu2014convex}. Pelckmans and De Moor introduced a shrinkage term to induce sparsity between centroids, enabling hierarchical clustering by tuning the trade-off parameter. \cite{hocking_clusterpath_2011} propose a convex relaxation-based clustering algorithm that efficiently traces a regularization path, achieves state-of-the-art performance on non-convex clusters, and simultaneously infers a hierarchical tree structure from the data. In \cite{chen2011integrating}, two optimization approaches — ADMM and a variant of AMA were introduced to solve convex clustering problems for practical applications. Additionally, convex relaxations of the $k-$Means problem via Semi-Definite Programming (SDP) have been developed \cite{peng2007approximating, awasthi2015relax, mixon2017clustering}, replacing the $k-$means objective with a trace-based formulation. \cite{mixon2017clustering} further showed that this SDP relaxation achieves perfect recovery with high probability under the stochastic unit-ball model in $\mathbb{R}^d$, given mild regularity conditions.

\noindent\textbf{Robustness and Feature Selection:} Robustness to outliers is essential for ensuring learning algorithms remain stable under adversarial or noisy conditions. To address this, \cite{gong2012robust} proposed a Robust Multi-Task Feature Learning model that not only identifies shared features across tasks but also detects outlier tasks. Similarly, \cite{chen2011integrating} introduced a robust multi-task learning framework combining a low-rank structure for related tasks and a sparse group structure to isolate outlier tasks.

\noindent\textbf{Median of Means based Clustering:}
The Median of Means (MoM) estimator provides a robust and efficient framework for mean estimation with strong theoretical guarantees. \cite{brunet2022k} introduced a bootstrap-based MoM method, forming blocks with replacement, which improves the breakdown point over standard MoM when enough blocks are used. In the context of interpretable clustering, \cite{moshkovitz2020explainable} proposed a method using small decision trees to partition data, enabling clear cluster characterization. They further analyzed whether such tree-induced clusterings can match the cost of optimal unconstrained clustering and how to compute them efficiently.

\section{Proposed Method} \label{prop_meth}

In this section, we outline our proposed clustering technique based on the median of means estimate in sufficient detail. We also include an Adam-based gradient descent method to optimize our non-convex objective function effectively.

Let $\mathbf{X}_{n \times d} \in\mathbb{R}^{n\times d}$ be the data matrix, where each row $\{\mathbf{x}_i\}_{i = 1}^n$ is a data point, and $\mathbf{x}_i\in\mathbb{R}^{d}$ for each $i\in\{1,2,...,n\}$. Let $\mathbf{u}_i$ be the agent corresponding to point $\mathbf{x}_i$ $\forall i \in \{1, 2, \cdots, n\}$, and we define $\mathbf{U}_{n \times d}\in\mathbb{R}^{n\times d}$ as the agent matrix. 

One of the most challenging problems in convex clustering is assigning weights to every pair of neighbours based on certain similarity measures. This enforces a restriction on its performance in high dimensions, which depends heavily on the choice of the pairwise similarities, although most people follow a $k-$nearest neighbour-based approach \cite{chi2015splitting}, coupled with a Gaussian similarity measure:
\begin{equation}
    w_{ij}= \mathbb{1}_{ij,k} e^{-\phi ||\boldsymbol{x}_i - \boldsymbol{x}_j||_{2}^2},
\end{equation}
where $\mathbb{1}_{ij,k}=1$ if $x_i$ is one of the $k$-nearest neighbours of $x_j$ with respect to the $\|\cdot\|_2 $ and $0$ otherwise. Here, $\phi$ represents the bandwidth of the Gaussian kernel, and smaller values of $\phi$ indicate greater similarities between the two nodes. Arbitrary choices of $\phi$ can lead to poor cluster generation, formation of arbitrary clusters, or even collapse of all cluster centroids to a global centroid \cite{hocking_clusterpath_2011}, hindering the overall effectiveness of the $k-$nearest neighbour-based heuristics.

Next, we introduce a Random Binning strategy to partition the dataset before minimizing a non-convex objective function. This class of Random Binning (RB) techniques was originally proposed in \cite{rahimi2007random} and subsequently revisited in \cite{wu2016revisiting}, where it was shown to yield faster convergence than other Random Features methods when scaling large-scale kernel machines. Although these previous approaches typically employ a parametrized feature map \cite{wu2018scalable}, incorporating both bin widths and offsets, our method adopts a simplified variant of this strategy - designed specifically to randomly partition the dataset into $\mathcal{O}(n)$ number of bins, each containing a fixed number of samples drawn from the observables. Formally, we partition the index set ${1, 2, \cdots, n}$ into $l = \mathcal{O}(n)$ subsets, denoted by $B = \{B_i\}_{i = 1}^l$, where each $B_i$ contains exactly $b (= \lfloor\frac{n}{l}\rfloor)$ elements. If $n$ is not divisible by $l$, a small number of elements are discarded to maintain uniform bin sizes across all partitions.

Henceforth, we define the ``contribution'' of point $\mathbf{x}_r$ in a ``convex'' type cost function as
\begin{equation}\label{eqn:1_r}
    f_{U}(\mathbf{x}_r) = \frac{1}{2} ||\mathbf{x}_r - \mathbf{u}_r||^2_2 + \frac{\gamma}{2} \sum_{i,j} w_{ij} ||\mathbf{u}_i - \mathbf{u}_j||^2_2.
\end{equation}

By our aforementioned MoM framework, instead of directly minimizing $\displaystyle \frac{1}{n} \sum_{r = 1}^n f_{\mathbf{U}}(\mathbf{x}_r)$, we aim to minimize an objective function of the form
\begin{equation} \label{eqn:2_r}
    C(\mathbf{U}) =  Median \left( \left\{ \frac{1}{b}\sum_{r \in B_j} f_{\mathbf{U}}(\mathbf{x}_r)\right\}_{j = 1}^{l}\right). 
\end{equation}

Noting that the second term in (\ref{eqn:1_r}) is independent of $r$, we define $l_t \in \{1, 2, \cdots, l\}$ such that
\begin{align}
    MoM_B(\mathbf{U}) &:= Median\left( \left\{ \frac{1}{2b} \sum_{i \in B_j} ||\mathbf{x}_i - \mathbf{u}_i||_2^2\right\}_{j = 1}^l\right) \notag \\
    &= \frac{1}{2b} \sum_{i \in B_{l_t}} ||\mathbf{x}_i - \mathbf{u}_i||_2^2.
\end{align}

Next, we rewrite the cost function as 
\begin{equation} 
    C(\mathbf{U}) = MoM_B(\mathbf{U}) + \frac{\gamma}{2}\sum_{i,j}w_{ij}||\mathbf{u}_i - \mathbf{u}_j||^{2}_{2}.
    \label{eq:new_cost_function}
\end{equation}

Before initiating the optimization procedure of the cost function in \ref{eq:new_cost_function}, we will involve another robustness criterion to make our objective function more stable from outliers: we use 
 $\displaystyle \sum_{i,j} w_{ij} \min\left(\mu, ||\mathbf{u}_i - \mathbf{u}_j||^{2}_{2}\right)$ instead of $\displaystyle \sum_{i,j}w_{ij}||\mathbf{u}_i - \mathbf{u}_j||^{2}_{2}$. By clipping the maximum pairwise distances by another hyperparameter $\mu$, we can significantly remove the effect of such outliers or other distant clusters.

 Now, we are in a position to write down the final cost function, which is
\begin{equation}
    C(\mathbf{U}) = MoM_B(\mathbf{U}) + \frac{\gamma}{2}\sum_{i,j}w_{ij} \min \left\{\mu, ||\mathbf{u}_i - \mathbf{u}_j||^{2}_{2}\right\} \label{introduce mu}.
\end{equation}

Due to the non-convex nature of this objective function, we use the ADAM gradient descent algorithm \cite{kingma2014adam} to minimize it. The gradient of $C(\mathbf{U})$ with respect to $\mathbf{u}_i$ is
\begin{align}
g_i := \frac{\partial C(\mathbf{U})}{\partial \mathbf{u}_i} = &
\frac{1}{b} (\mathbf{u}_i - \mathbf{x}_i) \mathbb{1} (i \in B_{l_t}) + \gamma \sum_{j} w_{ij} (\mathbf{u}_i - \mathbf{u}_j) \mathbb{1} (||\mathbf{u}_i - \mathbf{u}_j||_2^2 < \mu).
\end{align}

After $N$ iterations, we construct a graph with $\{\mathbf{u}_i\}_{i = 1}^n$ as vertices and where $\mathbf{u}_i$ and $\mathbf{u}_j$ are adjacent if $||\mathbf{u}_i - \mathbf{u}_j|| < \eta_1$. We assign each connected component of this graph as a cluster and combine all clusters with less than half the average cluster size into a single cluster, marking this combined cluster as noise.

\begin{algorithm}[htb!]
    \caption{COMET : Convex Clustering with Median of Mean Estimator and Adam Optimization}
    \label{alg:comet}
    \textbf{Input}: Data $\{\mathbf{x}_i\}_{i = 1}^n$ where $\mathbf{x}_i \in \mathbb{R}^d$ \\
    \textbf{Hyperparameters}: $N, k, \phi, \gamma, \mu, \eta_1$  \\
    \textbf{Output}: Cluster assignment $\{Z_i\}_{i = 1}^n$ where $Z_i \in \mathbb{N}$ \\
    \begin{algorithmic}[1]
        \STATE Construct a $k-$NN graph on $\{\mathbf{x}_i\}_{i = 1}^n$ and assign $w_{ij} = e^{-\phi ||\mathbf{x}_i - \mathbf{x}_j||_{2}^2}$ if $\mathbf{x}_i$ and $\mathbf{x}_j$ are adjacent, $w_{ij} = 0$ otherwise
        \STATE Initialize $\mathbf{m}_i^{(0)} = 0$, $\mathbf{v}_i^{(0)} = 0$ and $\mathbf{u}_i^{(0)} = \mathbf{x}_i$
        \FOR{$t = 0$ to $N-1$}
        \STATE Construct a partition, $B = \{B_i\}_{i = 1}^l$, of $\{1, 2, \cdots, n\}$ into $l$ bins each of size $b$ 
        \STATE Find $B_{l_t} \in B$ such that $\displaystyle MoM_B(\mathbf{U}^{(t)}) = \frac{1}{2b} \sum_{i \in B_{l_t}} ||\mathbf{x}_i - \mathbf{u}_i^{(t)}||_2^2$
        \STATE $\displaystyle \mathbf{g}_i^{(t)} = \frac{1}{b} (\mathbf{u}_i^{(t)} - \mathbf{x}_i) \mathbb{1} (i \in B_{l_t}) + \gamma \sum_{j} w_{ij} (\mathbf{u}_i^{(t)} - \mathbf{u}_j^{(t)}) \mathbb{1} (||\mathbf{u}_i^{(t)} - \mathbf{u}_j^{(t)}||_2^2 < \mu)$
        \STATE $\mathbf{m}_i^{(t)} = \boldsymbol{\beta}_1 \mathbf{m}_i^{(t - 1)} + (1 - \boldsymbol{\beta}_1) \mathbf{g}_i^{(t)}$
        \STATE $\mathbf{v}_i^{(t)} = \boldsymbol{\beta}_2 \mathbf{v}_i^{(t - 1)} + (1 - \boldsymbol{\beta}_2) \left(\mathbf{g}_i^{(t)} \odot \mathbf{g}_i^{(t)}\right)$
        \STATE Calculate $\mathbf{u}_i^{(t + 1)}$ from $\mathbf{u}_i^{(t)}$ with the help of $\hat{\mathbf{m}}_i^{(t)}$ and $\hat{\mathbf{v}}_i^{(t)}$ using the ADAM update rule. Refer to the supplementary material (\ref{adam_rule}) for the exact update rule.
        \ENDFOR
        \STATE Construct a graph on $\left\{\mathbf{u}_i^{(N)}\right\}_{i = 1}^n$ where $\mathbf{u}_i$ and $\mathbf{u}_j$ are adjacent if $||\mathbf{u}_i - \mathbf{u}_j|| < \eta_1$
        \STATE Assign each connected component of this graph as a cluster
        \STATE Combine all clusters with less than half the average cluster size into a single cluster and mark this combined cluster as noise
    \end{algorithmic}
\end{algorithm}

\section{Theoretical Properties} \label{theo_prop}

This section establishes the theoretical properties of the (global) optimal solutions of the proposed objective function. We also analyse computational complexity and discuss the convergence properties of our method.

\subsection{Finite Sample Error Bounds and Weak Consistency}

We begin our statistical analysis of COMET by providing finite sample error bounds on the prediction error, using the widely used Hanson-Wright inequalities, especially the recent uniform versions by \cite{pmlr-v125-bousquet20b}. These bounds provide sufficient conditions for the consistency of the centroid and weight estimates.

Recall the objective function \eqref{eq:new_cost_function},
\[\text{min}_U\left\{\frac{1}{2b} \sum_{i \in B_{l_t}} ||\mathbf{x}_i - \mathbf{u}_i||_2^2 + \frac{\gamma}{2}\sum_{i,j}w_{ij}||\mathbf{u}_i-\mathbf{u}_j||^{2}_{2}\right\}, \label{eq:rev_cost_fn}\]

\noindent where $l_t$ is such that \(\frac{1}{2b} \sum_{i \in B_{l_t}} ||\mathbf{x}_i - \mathbf{u}_i||_2^2 = Median\left( \left\{ \frac{1}{2b} \sum_{i \in B_j} ||\mathbf{x}_i - \mathbf{u}_i||_2^2\right\}_{j = 1}^l\right) \), \(l_t \in \{1, \ldots, l\}\). 

\noindent Let $\boldsymbol{x} = vec (X)$ and $\boldsymbol{u} = vec (U)$, where $vec(\cdot)$ means to vectorize a matrix by appending its columns together. So, \(\boldsymbol{x}, \boldsymbol{u} \in \mathbb{R}^{nd}\) and \(\boldsymbol{x}_{d(i-1)+j} = X_{ij}\), \(\boldsymbol{u}_{d(i-1)+j} = U_{ij}\). 

\noindent Consider \(\boldsymbol{I}_{B_{l_t}}\) to be an \(nd \times nd\) diagonal matrix with $i$-th diagonal element  = 1 if \(bd \leq i < (b+1)d\) where \(b \in B_{l_t}\) and all other elements 0. So, we can write
\[\sum_{i \in B_{l_t}} ||\mathbf{x_i - u_i}||_2^2 = (\boldsymbol{x - u})^\top I_{B_{l_t}} (\boldsymbol{x - u}).\]

\noindent Also note that $w_{ij}$'s remain fixed in each iteration of the algorithm. Since $w_{ij}$'s are either 0 or $< 1$, we work with an upper bound of the cost function, where each $w_{ij}$ is replaced by $w'_{ij} = \mathbb{1}(w_{ij} > 0)$. Let $D^{n(n-1)d \times nd}$ be such that $\boldsymbol{D}_{\mathcal{C}(i,j)}\boldsymbol{u} = \mathbf{u}_{i} - \mathbf{u}_{j}$, where \(\mathcal{C}(i,j)\) is an index set: then the objective function can be written as 
\begin{equation}
    \text{min} \left\{\frac{1}{2b}(\boldsymbol{x - u})^\top I_{B_{l_t}} (\boldsymbol{x - u}) + \frac{\gamma}{2}\sum_{(i,j) \in \mathcal{E}}\|\boldsymbol{D}_{\mathcal{C}(i,j)}\boldsymbol{u}\|^2_2\right\},
    \label{eq:cost_fn_matrix}
\end{equation}
where $\mathcal{E} \subseteq \{(i, j) : i, j \in \{1, 2, \dots, n\}\}$ is an index set. We will assume the model \(\boldsymbol{x = u+\epsilon}\) , where $\epsilon \in \mathbb{R}^{nd}$ is a vector of independent noise variables and \(\mathbb{E}(\epsilon) = 0\). This model is fairly standard for analysing the large-sample behaviour of convex clustering methods \cite{tan_statistical_2015};\cite{Wang_2018}. For all practical purposes, one may assume that the error terms are almost surely bounded, that is, for some $M > 0$, \(|\epsilon_i| \leq M\) for all \(i = 1, \ldots, nd\). For notational simplicity, we write \(\|{\boldsymbol{y}}\|_{\boldsymbol{A}}^2 = {\boldsymbol{y}}^\top \boldsymbol{A y}\), for any positive semidefinite matrix $\boldsymbol{A}$. The goal of this analysis is to find probabilistic bounds on \(\|\boldsymbol{\hat{u} - u}\|_{\hat{\boldsymbol{I}}_{B_{l_t}}}^2\), where $\hat{\boldsymbol{u}}$ and ${\hat{\boldsymbol{I}}_{B_{l_t}}}$ are obtained by minimising the objective function in \eqref{eq:cost_fn_matrix}.

\begin{theorem}
\label{thm:important}
    Suppose the model behaves as \(\boldsymbol{x = u + \epsilon}\), where \(\boldsymbol{\epsilon} \in \mathbb{R}^{nd}\) is a vector of independent bounded random variables, with mean 0, covariance matrix $\sigma^2 \boldsymbol{I}_{nd\times nd}$ and \(|\epsilon_i| \leq M\), for all \(i = 1, \ldots, nd\). Further assume that \(\boldsymbol{\hat{u}}\) and \({\hat{\boldsymbol{I}}_{B_{l_t}}}\) are obtained from minimizing \eqref{eq:cost_fn_matrix}, then if $\gamma^{\prime} \geq \frac{M}{ndb\sqrt{n}}$ the following holds with probability at least 1 - $\delta$,
    \begin{equation}
        \begin{split}
            \frac{1}{2ndb}\|\boldsymbol{\hat{u} - u}\|_{\hat{\boldsymbol{I}}_{B_{l_t}}}^2 \leq M^2\left(\frac{\sqrt{db}+d\sqrt{n}}{n\sqrt{db}}\right)
            + M^2 \left(\frac{c}{b\sqrt{nd}}\sqrt{\log\left(\frac{1}{\delta}\right)}
            + \frac{c\log\left(\frac{1}{\delta}\right)}{ndb}\right) + \gamma^{\prime} \frac{\left|\mathcal{E}\right|}{4} \\+ \gamma^{\prime}\left[ \sum_{(i, j) \in \mathcal{E}} \left\|\boldsymbol{D}_{\mathcal{C}(i, j)}\boldsymbol{u}\right\|_2 + \sum_{(i, j) \in \mathcal{E}} \left\|\boldsymbol{D}_{\mathcal{C}(i, j)}\boldsymbol{u}\right\|_2^2\right].
        \end{split}
        \label{eq:bound}
    \end{equation}
\end{theorem}

\noindent The proof of this theorem is deferred to the supplemntary material (\ref{thm_proof}). From this theorem, we also arrive at the following two corollaries: Corollary \ref{cor:1} addresses the convergence of the centroid estimates under a minimum constraint on the hyperparameter, for the number of features being small enough with respect to the number of sample points, while Corollary \ref{cor:2} elaborates on the rate of convergence of those estimates under the constraint on the hyperparameter only. 

\begin{corollary}\label{cor:1}
    Suppose $\left\|\boldsymbol{D}_{\mathcal{C}(i, j)}\boldsymbol{u}\right\|_2 \leq C$, for all $1 \leq i, j \leq n$, for some constant $C$, $\left|\mathcal{E}\right| \leq kn$ and $\gamma^{\prime} \geq \frac{M}{ndb\sqrt{n}}$. If $d = o(n)$, then $\frac{1}{2ndb}\|\boldsymbol{\hat{u} - u}\|_{\hat{\boldsymbol{I}}_{B_{l_t}}}^2 \overset{p}{\rightarrow} 0$ as $n, d \rightarrow \infty$.
\end{corollary}

\begin{corollary}\label{cor:2}
    Suppose $\left\|\boldsymbol{D}_{\mathcal{C}(i, j)}\boldsymbol{u}\right\|_2 \leq C$, for all $1 \leq i, j \leq n$, for some constant $C$, $\left|\mathcal{E}\right| \leq kn$ and $\gamma^{\prime} \geq \frac{M}{ndb\sqrt{n}}$. Then $\frac{1}{2ndb}\|\boldsymbol{\hat{u} - u}\|_{\hat{\boldsymbol{I}}_{B_{l_t}}}^2 = O\left(\frac{1}{\sqrt{n}}\right)$.
\end{corollary}

 We lay out the complete proofs of Corollary \ref{cor:1} and Corollary \ref{cor:2} in the supplementary material (\ref{col1_proof}) and (\ref{col2_proof}), respectively.

\subsection{Computational Complexity and Other Competing Methods}
We compare the efficiency of our algorithm with other popular algorithms such as Convex Clustering ({\cite{chi_splitting_2015}}), Robust Continuous Clustering ({\cite{shah_robust_2017}}), and Robust Convex Clustering ({\cite{wang_robust_2016}}). For a detailed explanation, refer to supplementary material (\ref{time_complexity}).  

\begin{table}[htb!]
\centering
\label{tab:complexity_comparison_table}
\begin{tabular}{lc}
\textbf{Algorithm} & \textbf{Complexity} \\
\hline
COMET  & ${\mathcal{O}(Nnkd)}$  \\
Convex-Clustering             & ${\mathcal{O}(N(n^{2}d + d \epsilon))}$ \\
Robust Continuous Clustering  & ${\mathcal{O}(N(n^{2}d + nkd)}$  \\
Robust Convex Clustering      & ${\mathcal{O}(Nnkd)}$  \\
\hline
\end{tabular}
\caption{Comparison of Runtime Complexity with other SOTA methods}
\end{table}

\noindent From the above table, it is clear that COMET is better or at least on par in terms of computational cost with recent and most widely used robust clustering algorithms.

\section{Experiments and Results} \label{exp_result}

In this section, we demonstrate the superiority of our proposed algorithm, COMET, over different variants of existing clustering algorithms, including both real and simulated datasets. The description of the datasets is given in the supplementary material (\ref{data_description}). For simulated datasets, the generation procedure is described later.

\subsection{Algorithms under consideration}

We consider the following well-known clustering algorithms to assess the effectiveness of COMET: $k-$means (KM){\cite{hartigan_algorithm_1979}}, Convex Clustering (CC) {\cite{chi_splitting_2015}}, MoM $k-$means (MKM) {\cite{paul_uniform_2021}}, Robust Convex Clustering (RConv) {\cite{wang_robust_2016}}, Robust Continuous Clustering (RCC) {\cite{shah_robust_2017}} and Robust Bregman $k-$means (RBKM) {\cite{brecheteau_robust_2021}}.

\subsection{Performance Measures}

For evaluating our proposed COMET algorithm against competing methods, we adopt the following metrics and resources:
\begin{itemize}
    \item \textbf{Evaluation Metrics:} Since ground-truth cluster labels are available for all real and simulated datasets, we evaluate clustering performance using 
    \begin{itemize}
        \item \textbf{Adjusted Rand Index (ARI)}
        \item \textbf{Adjusted Mutual Information (AMI)}
    \end{itemize}
    Both metrics provide robust comparisons across algorithms.
    
    \item \textbf{Estimated Number of Clusters:} We also report the average number of clusters estimated by each algorithm to further assess performance.
\end{itemize}


\subsection{Experimental Set-up} \label{exp_setup}

We apply all the selected algorithms on the datasets listed in the supplementary material (\ref{data_description}). Our main goal is to make a proper comparison of robustness of these algorithms to the presence of noise and outliers in the data. We artificially add different levels of noise and outliers to the datasets under study and record the performances of the algorithms.\\
To add noise of level $p\%$ to a dataset, we first consider the smallest axis-parallel hypercube containing the whole original dataset. Then, we simulate $\lfloor \frac{np}{100} \rfloor$ points uniformly from the hypercube and add them to the original dataset, labeled as ``noise''. All the algorithms are run on this modified dataset, and the ARI/AMI is calculated based on the obtained cluster labels of the original data points only. We vary $p$ to observe the change in performance of the algorithms with the introduction of noise. $k-$means, MoM $k-$means and Robust Bregman $k-$means require the exact number of clusters to be given as input, but that gives these algorithms an unfair advantage considering Convex Clustering, Robust Convex Clustering, Robust Continuous Clustering as well as COMET determine the number clusters automatically. Hence, to ensure a fair comparison, we used \textit{Gapstat} ({\cite{tibshirani_estimating_2001}}) in those three algorithms to get an estimate of the number of clusters from the data itself and used that value for the clustering. We run all algorithms according to the recommended specification of hyperparameters or tune them to achieve maximum ARI. $k-$means, MoM $k-$means, and RBKM are run till there is no further update in the cluster assignment matrix. Convex Clustering and Robust Convex Clustering were run on each dataset after tuning its hyperparameters for 150 epochs. RCC is run according to the hyperparameter recommendations and termination condition specified in \cite{shah_robust_2017}. $k-$means, MoM $k-$means and Robust Bregman $k-$means are dependent on the choice of initial centroids, each of them were run 25 times for every noise level and the mean performance is reported with their standard deviation. The random noise was added to the data using numpy.random.default\_rng(0) in numpy library of python $3$ (ipykernel).

We perform the experiments for both generated and real-life datasets. In the next section we will focus on the results for real-life datasets. \textbf{For the detailed study on generated datasets refer to the supplementary material (\ref{app:synthetic_data_study})}

\subsection{Real-Life Datasets} \label{rl_data}

Here we show the clustering results of our algorithm COMET and other selected algorithms on some of the real-life datasets with $10\%$ noise. For the results on other datasets refer to the supplementary material (\ref{app:real_life_full_study}). Here, \text{$k^\ast$} refers to an estimated number of clusters. The actual number of clusters is indicated as $k$. Here the standard deviation is that of the performance measure, not of the mean statistic.

\begin{table*}[tb]
\centering

\begin{tabular}{l>{\centering\arraybackslash}p{.65cm}>{\centering\arraybackslash}p{1.5cm}>{\centering\arraybackslash}p{1.7cm}>{\centering\arraybackslash}p{1.7cm}>{\centering\arraybackslash}p{1.7cm}>{\centering\arraybackslash}p{1.7cm}>{\centering\arraybackslash}p{1.7cm}>{\centering\arraybackslash}p{1.7cm}}
\hline
\textbf{Dataset} & \textbf{Index} & \textbf{KM} & \textbf{MKM} & \textbf{CC} & \textbf{RCC} & \textbf{RConv} & \textbf{RBKM} & \textbf{COMET}\\
\hline

\multirow{3}{6em}{Newthyroid ($k$ = 3)}
&\text{$k^\ast$}
&3.08$\pm$1.28 
&2.94$\pm$1.38
&14.14$\pm$1.23 
&212.13$\pm$3.36 
&3.79$\pm$0.58
&2.00$\pm$0.00 
&4.14$\pm$0.36 \\
&\text{ARI}
&0.34$\pm$0.21$^\dagger$
&0.40$\pm$0.26$^\dagger$
&0.69$\pm$0.04$^\dagger$
&0.00$\pm$0.00$^\dagger$
&0.81$\pm$0.21$^\dagger$
&0.11$\pm$0.03$^\dagger$
&\textbf{0.97$\pm$0.01} \\
&\text{AMI}
&0.34$\pm$0.19$^\dagger$
&0.39$\pm$0.25$^\dagger$
&0.52$\pm$0.03$^\dagger$
&0.003$\pm$0.004$^\dagger$
&0.77$\pm$0.16$^\dagger$
&0.08$\pm$0.03$^\dagger$
&\textbf{0.90$\pm$0.02} \\

\hline

\multirow{3}{6em}{Wisconsin ($k$ = 2)}
&\text{$k^\ast$}
&2.25$\pm$0.63 
&2.00$\pm$0.82
&15$\pm$1.86 
& 477$\pm$9.06
& 2.00$\pm$1.04
&2.00$\pm$0.00 
&3.00$\pm$0.00 \\
&\text{ARI}
&0.52$\pm$0.35$^\dagger$ 
&0.47$\pm$0.39$^\dagger$ 
&0.81$\pm$0.01$^\dagger$ 
&0.01$\pm$0.00$^\dagger$
&0.85$\pm$0.03$^\sim$
&0.15$\pm$0.06$^\dagger$ 
&\textbf{0.87$\pm$0.01} \\
&\text{AMI}
&0.48$\pm$0.31$^\dagger$
&0.41$\pm$0.34$^\dagger$
&0.67$\pm$0.01$^\dagger$
&0.07$\pm$0.003$^\dagger$
&0.75$\pm$0.03$^\dagger$
&0.19$\pm$0.05$^\dagger$
&\textbf{0.76$\pm$0.01} \\

\hline

\multirow{3}{6em}{Wine ($k$ = 3)}
&\text{$k^\ast$}
&3.14$\pm$1.18
&3.17$\pm$1.35
&25.29$\pm$2.16 
&178$\pm$0.00 
&2.43$\pm$0.51
&2.00$\pm$0.00 
&4.64$\pm$0.84 \\
&\text{ARI}
&0.66$\pm$0.31$^\sim$ 
&0.59$\pm$0.29$^\dagger$ 
&0.59$\pm$0.15$^\dagger$ 
&0.0$\pm$0.0$^\dagger$
&0.22$\pm$0.28$^\dagger$
&0.01$\pm$0.02$^\dagger$ 
&\textbf{0.79$\pm$0.15} \\
&\text{AMI}
&0.67$\pm$0.29$^\sim$ 
&0.61$\pm$0.28$^\dagger$ 
&0.59$\pm$0.10$^\dagger$ 
&0.00$\pm$0.00$^\dagger$ 
&0.32$\pm$0.31$^\dagger$
&0.04$\pm$0.03$^\dagger$ 
&\textbf{0.80$\pm$0.09} \\

\hline

\multirow{3}{6em}{Dermatology ($k$ = 6)}
&\text{$k^\ast$}
&5.16$\pm$1.93 
&4.77$\pm$2.09
&4.00$\pm$0.00 
&358$\pm$0.00 
&5.00$\pm$0.00
& 2.00$\pm$0.00
&5.85$\pm$0.53 \\
&\text{ARI}
&0.61$\pm$0.17$^\dagger$
&0.56$\pm$0.17$^\dagger$
&0.21$\pm$0.00$^\dagger$
&0.00$\pm$0.00$^\dagger$
&0.66$\pm$0.01$^\dagger$
&0.004$\pm$0.02$^\dagger$
&\textbf{0.81$\pm$0.06} \\
&\text{AMI}
&0.78$\pm$0.10$^\dagger$ 
&0.73$\pm$0.13$^\dagger$ 
&0.44$\pm$0.00$^\dagger$ 
&0.00$\pm$0.00$^\dagger$ 
&0.79$\pm$0.01$^\dagger$
&0.04$\pm$0.04$^\dagger$
&\textbf{0.86$\pm$0.04} \\

\hline

\multirow{3}{6em}{Lung-Discrete ($k$ = 7)}
&\text{$k^\ast$}
&5.91$\pm$1.57 
&5.23$\pm$1.39
&2.36$\pm$0.63 
&14.9$\pm$16.8 
&9.79$\pm$0.43
&2$\pm$0.00
&9.21$\pm$0.80 \\
&\text{ARI}
&0.44$\pm$0.09$^\dagger$ 
&0.50$\pm$0.10$^\dagger$ 
&0.07$\pm$0.03$^\dagger$ 
&0.41$\pm$0.12$^\dagger$ 
&0.39$\pm$0.05$^\dagger$
&0.01$\pm$0.01$^\dagger$
&\textbf{0.71$\pm$0.02} \\
&\text{AMI}
&0.53$\pm$0.07$^\dagger$ 
&0.58$\pm$0.08$^\dagger$ 
&0.20$\pm$0.07$^\dagger$
&0.51$\pm$0.15$^\dagger$ 
&0.51$\pm$0.04$^\dagger$
&0.07$\pm$0.04$^\dagger$
&\textbf{0.69$\pm$0.01} \\

\hline

\multirow{3}{6em}{ORLRaws10p ($k$ = 10)}
&\text{$k^\ast$}
&4.85$\pm$1.83 
&4.89$\pm$1.72
&42$\pm$0.00 
&100$\pm$0.00 
&16$\pm$0.00
&2$\pm$0.00
&14$\pm$0.00 \\
&\text{ARI}
&0.33$\pm$0.11$^\dagger$ 
&0.33$\pm$0.10$^\dagger$ 
&0.53$\pm$0.00$^\dagger$ 
&0.00$\pm$0.00$^\dagger$ 
& 0.54$\pm$0.002$^\dagger$
& 0.02$\pm$0.01$^\dagger$
&\textbf{0.73$\pm$0.00} \\
&\text{AMI}
&0.58$\pm$0.11$^\dagger$ 
&0.61$\pm$0.09$^\dagger$ 
&0.61$\pm$0.00$^\dagger$ 
&0.00$\pm$0.00$^\dagger$ 
&0.69$\pm$0.001$^\dagger$
&0.11$\pm$0.03$^\dagger$
&\textbf{0.81$\pm$0.00} \\

\hline

\end{tabular}
\caption{Results for Real-life Datasets}

\label{tab:Real_data_partial}
{\raggedright \scriptsize $^\dagger$ : significantly different from the best performing algorithm,
$^\sim$ : statistically similar to the best performing algorithm .}
\end{table*}

\subsubsection{Discussion}

Table \ref{tab:Real_data_partial} shows that COMET outperforms other algorithms, achieving nearly accurate cluster numbers with low standard deviation across most datasets. However, in datasets like ORLRaws10P (Table \ref{tab:Real_data_partial}), Brain, and Wisconsin (present in the supplementary material (\ref{app:real_life_full_study})), the detected cluster size slightly deviates from the actual value, likely due to limitations in the $k-$NN graph structure. This issue is more pronounced in other algorithms. Despite this, COMET still provides better clustering patterns, with higher ARI and AMI than the others.

\subsection{Significance of Our Results: Wilcoxon-Rank Sum Test} \label{wilcox_test}

For various datasets, we want to test if the ARI and AMI produced by our algorithm are ``significantly higher'' than other selected clustering algorithms. We use Wilcoxon-Rank Sum test for this purpose. Refer to the supplementary material (\ref{app:p-value_full_table}) for detailed discussion related to this.

\subsection{Case Study on Brain dataset}

We evaluate our algorithm's performance on the Brain dataset, a Microarray dataset with 42 instances (brain tumor patients) and 5597 features. The dataset includes 5 categories: 10 medulloblastomas, 10 malignant gliomas, 10 AT/RT, 4 normal cerebellums, and 8 supratentorial PNETs, as described in \cite{pomeroy_prediction_2002}.

\begin{table*}[bh]
\centering
\small

\begin{tabular}{>{\centering\arraybackslash}p{1cm}>{\centering\arraybackslash}p{1cm}>{\centering\arraybackslash}p{1.7cm}>{\centering\arraybackslash}p{1.7cm}>{\centering\arraybackslash}p{1.7cm}>{\centering\arraybackslash}p{1.7cm}>{\centering\arraybackslash}p{1.7cm}>{\centering\arraybackslash}p{1.7cm}>{\centering\arraybackslash}p{1.7cm}}
\hline
\textbf{Index} & \textbf{Noise($\%$)} & \textbf{KM} & \textbf{MKM} & \textbf{CC} & \textbf{RCC} & \textbf{RConv} & \textbf{RBKM} & \textbf{COMET}\\

\hline
\multirow{5}{7em}{ARI}
& 0& 0.28$\pm$0.10$^\dagger$& 0.23$\pm$0.11$^\dagger$& 0.64$\pm$0.00$^\dagger$& 0.00$\pm$0.00$^\dagger$& 0.56$\pm$0.00$^\dagger$& 0.01$\pm$0.01$^\dagger$& \textbf{0.65$\pm$0.00}\\
& 5& 0.31$\pm$0.13$^\dagger$& 0.31$\pm$0.13$^\dagger$& 0.64$\pm$0.00$^\dagger$& 0.00$\pm$0.00$^\dagger$& 0.56$\pm$0.01$^\dagger$& 0.01$\pm$0.01$^\dagger$& \textbf{0.66$\pm$0.00}\\
& 10& 0.26$\pm$0.10$^\dagger$& 0.26$\pm$0.10$^\dagger$& 0.64$\pm$0.02$^\sim$& 0.00$\pm$0.00$^\dagger$& 0.56$\pm$0.06$^\dagger$& 0.016$\pm$0.02$^\dagger$& \textbf{0.66$\pm$0.03}\\
& 15&  0.22$\pm$0.09$^\dagger$&  0.10$\pm$0.08$^\dagger$& 0.63$\pm$0.02$^\sim$& 0.00$\pm$0.00$^\dagger$& 0.55$\pm$0.06$^\dagger$& 0.02$\pm$0.02$^\dagger$& \textbf{0.66$\pm$0.03}\\
& 20& 0.19$\pm$0.11$^\dagger$& 0.08$\pm$0.07$^\dagger$& 0.63$\pm$0.04$^\dagger$& 0.00$\pm$0.00$^\dagger$& 0.63$\pm$0.03$^\dagger$& 0.02$\pm$0.02$^\dagger$& \textbf{0.65$\pm$0.02}\\
\hline
\multirow{5}{7em}{AMI}
& 0& 0.35$\pm$0.10$^\dagger$& 0.32$\pm$0.10$^\dagger$& 0.62$\pm$0.00$^\dagger$& 0.00$\pm$0.00$^\dagger$& 0.62$\pm$0.00$^\dagger$& 0.017$\pm$0.01$^\dagger$& \textbf{0.67$\pm$0.00}\\
& 5& 0.38$\pm$0.13$^\dagger$& 0.28$\pm$0.14$^\dagger$& 0.62$\pm$0.00$^\dagger$& 0.00$\pm$0.00$^\dagger$& 0.62$\pm$0.01$^\dagger$& 0.02$\pm$0.02$^\dagger$& \textbf{0.72$\pm$0.00}\\
& 10& 0.33$\pm$0.10$^\dagger$& 0.27$\pm$0.11$^\dagger$& 0.62$\pm$0.03$^\dagger$& 0.00$\pm$0.00$^\dagger$& 0.62$\pm$0.05$^\dagger$& 0.03$\pm$0.04$^\dagger$& \textbf{0.72$\pm$0.03}\\
& 15& 0.29$\pm$0.11$^\dagger$& 0.18$\pm$0.12$^\dagger$& 0.62$\pm$0.01$^\dagger$& 0.00$\pm$0.00$^\dagger$& 0.62$\pm$0.04$^\dagger$& 0.03$\pm$0.04$^\dagger$& \textbf{0.72$\pm$0.02}\\
& 20& 0.27$\pm$0.13$^\dagger$& 0.15$\pm$0.12$^\dagger$& 0.62$\pm$0.01$^\dagger$& 0.00$\pm$0.00$^\dagger$& 0.62$\pm$0.05$^\dagger$& 0.03$\pm$0.05$^\dagger$& \textbf{0.72$\pm$0.03}\\
\hline
\multirow{5}{7em}{$k^\ast$}
& 0& 5$\pm$1.92& 5$\pm$1.50& 18$\pm$0.00& 42$\pm$0.00& 5$\pm$0.00& 2$\pm$0.00& \textbf{5$\pm$0.00}\\
& 5& 5$\pm$1.85& 5$\pm$1.46& 18$\pm$0.00& 42$\pm$0.00& 5$\pm$0.00& 2$\pm$0.00& \textbf{4$\pm$0.00}\\
& 10& 5$\pm$1.92&5$\pm$1.49& 18$\pm$0.00&42$\pm$0.00& 5$\pm$0.36& 2$\pm$0.50& \textbf{4$\pm$0.00}\\
& 15& 4$\pm$1.90& 4$\pm$1.41& 18$\pm$0.00& 42$\pm$0.00& 5.1$\pm$0.22& 2$\pm$0.50& \textbf{4$\pm$0.00}\\
& 20&3$\pm$1.45& 3$\pm$1.42& 18$\pm$0.00& 42$\pm$0.00& 18$\pm$0.52& 2$\pm$0.5& \textbf{4$\pm$0.00}\\
\hline
\end{tabular}\caption{Performance of Different Algorithms on \textbf{Brain} on Different Noise Levels}
\label{tab:table_brain}
{\raggedright \scriptsize $^\dagger$ : significantly different from the best performing algorithm,
$^\sim$ : statistically similar to the best performing algorithm .}
\end{table*}

We compare the algorithms under varying noise levels (0$\%$, 5$\%$, 10$\%$, 15$\%$, 20$\%$) using the procedure outlined earlier, without applying any feature reduction methods like PCA. The results, shown in Table \ref{tab:table_brain} and Figure \ref{fig:line_plot_brain}, show that COMET consistently outperforms all other algorithms, maintaining an ARI above 0.6 across all noise levels. Here, the standard deviation is that of the performance measure, not of the mean statistic. Convex clustering and Robust Convex clustering perform well but are outpaced by COMET. $k-$means and MoM $k-$means show poor results, worsening with increased noise. RCC and RBKM perform very poorly, as reflected in the results. The \textit{t}-SNE plots for the clustering results for various algorithms on this dataset are provided in the supplementary material (\ref{app:TSNE_plots}).

Also, refer to the supplementary material (\ref{app:case_study_wisconsin}) for a case study on the Wisconsin dataset.

\begin{figure}[htb!]
    \centering
    \includegraphics[width=0.64\linewidth]{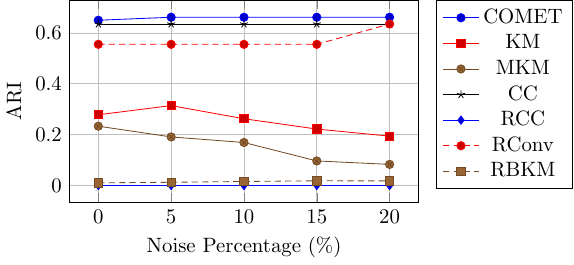}
    \caption{Line plot for performance of different algorithms on \textbf{Brain} dataset}
    \label{fig:line_plot_brain}
\end{figure}

\section{Ablation Study} \label{ablation_main}

A sensitivity analysis of the performance of our algorithm for the hyperparameter $\gamma$ is illustrated in the supplementary material (\ref{gamma_app}).
For a detailed illustration of the tuning process of our hyperparameters on the \textbf{Wisconsin Breast Cancer} dataset, refer to the supplementary material (\ref{abl_wisc_app}). Another ablation study on the \textbf{NewThyroid} Dataset is discussed in the supplementary material (\ref{abl_nwthy_app}).

\section{Limitations} \label{limitation}

While we tested our algorithm on specific distributions of noise for clustering synthetic data, a systematic method to choose an appropriate cost function needs to be developed. However, it is still obscure to us how the cost function can be modified in case the noise follows some definite pattern and is not randomly distributed. Also, in our proof of theoretical consistency \ref{thm:important}, we assumed $d = o(n)$. However, modifications to the clustering procedure need to be done for higher-dimensional datasets. Overall, given the flexibility of our clustering framework, other possibilities can be explored by incorporating different methods for different steps. One may further explore to relax assumptions on the errors \ref{thm:important} for consistency in a more general settings.

\section{Conclusion}
In this study, we introduce a robust and interpretable clustering framework designed for multivariate datasets affected by noise. Our method reformulates the underlying cost function to reduce the adverse effects of random noise that often undermine conventional convex clustering techniques. We provide rigorous theoretical guarantees by establishing both the consistency and the convergence rate of the proposed estimators. Furthermore, extensive empirical evaluations—including experiments on diverse real-world datasets, detailed case studies, and ablation studies—demonstrate the practical effectiveness and reliability of our approach.

\bibliographystyle{IEEEtran}

\bibliography{ref}

\newpage
\appendices

\section{Supplementary Material}

\subsection{ADAM update rule for optimization} \label{adam_rule}
\label{ADAM_update_rule}
For a chosen $\boldsymbol{\beta}_1$ and $\boldsymbol{\beta}_2$ the update rule to be followed is as follows: \\
$\displaystyle \hat{\mathbf{m}}_i^{(t)} = \frac{\mathbf{m}_i^{(t)}}{1 - \boldsymbol{\beta}_1^t}$ , \\
$\displaystyle \hat{\mathbf{v}}_i^{(t)} = \frac{\mathbf{v}_i^{(t)}}{1 - \boldsymbol{\beta}_2^t}$ , \\
$\displaystyle \mathbf{u}_i^{(t+1)} = \mathbf{u}_i^{(t)} - \frac{\alpha \hat{\mathbf{m}}_i^{(t)}}{\sqrt{\hat{\mathbf{v}}_i^{(t)} + \epsilon}}$.
\setcounter{theorem}{0}
\setcounter{corollary}{0}

\subsection{Proof of Theorem 1} \label{thm_proof}

\begin{theorem}
    Suppose \(\boldsymbol{x = u + \epsilon}\), where \(\boldsymbol{\epsilon} \in \mathbb{R}^{nd}\) is a vector of independent bounded random variables, with mean 0, covariance matrix $\sigma^2 \boldsymbol{I}$ and \(|\epsilon_i| \leq M\), for all \(i = 1, \ldots, nd\). Suppose that \(\boldsymbol{\hat{u}}\) and \({\hat{\boldsymbol{I}}_{B_{l_t}}}\) are obtained from minimizing equation ($10$), then if $\gamma^{\prime} \geq \frac{M}{ndb\sqrt{n}}$ the following holds with probability at least 1 - $\delta$:
    \[
        \begin{split}
            \frac{1}{2ndb}&\|\boldsymbol{\hat{u} -  u}\|_{\hat{\boldsymbol{I}}_{B_{l_t}}}^2 \leq  \quad M^2\left(\frac{\sqrt{db/n}+d}{\sqrt{ndb}} \right.
            \left. + \quad c\frac{1}{b\sqrt{nd}}\sqrt{\log\left(\frac{1}{\delta}\right)}  + c\frac{\log\left(\frac{1}{\delta}\right)}{ndb}\right) \notag\\
            + \quad \gamma^{\prime}& \frac{\left|\mathcal{E}\right|}{4} + \gamma^{\prime}\left[ \sum_{(i, j) \in \mathcal{E}} \left\|\boldsymbol{D}_{\mathcal{C}(i, j)}\boldsymbol{u}\right\|_2 + \sum_{(i, j) \in \mathcal{E}} \left\|\boldsymbol{D}_{\mathcal{C}(i, j)}\boldsymbol{u}\right\|_2^2\right]
        \end{split}
    \]
\end{theorem}

\begin{proof}
    Let $\boldsymbol{D = U \Lambda V_{\beta}^{\top}}$ be the singular value decomposition (SVD) of $\boldsymbol{D}$, where $\boldsymbol{V_{\beta}} \in \mathbb{R}^{nd \times (n-1)d}$. We construct $\boldsymbol{V_{\alpha}} \in \mathbb{R}^{nd \times d}$ such that $\boldsymbol{V} = [\boldsymbol{V_{\alpha}}, \boldsymbol{V_{\beta}}]$ is an $nd \times nd$ orthonormal matrix. \\

    Next, define $\boldsymbol{\beta = V_{\beta}^{\top}u}$, $\boldsymbol{\alpha = V_{\alpha}^{\top}u}$ and $\gamma^{\prime} = \frac{\gamma}{2nd}$. Also $\boldsymbol{Z = U \Lambda}$ and $\boldsymbol{Z^{-}}$ is the left inverse of $\boldsymbol{Z}$. Thus, the optimization problem becomes:

\begin{align}
    \min_{\boldsymbol{\alpha}, \boldsymbol{\beta}, \boldsymbol{I}_{B_{l_{t}}}} & \frac{1}{2ndb} \left(\boldsymbol{x} - \boldsymbol{V_{\alpha}\alpha} - \boldsymbol{V_{\beta}\beta}\right)^{\top} \boldsymbol{I}_{B_{l_t}} \left(\boldsymbol{x} - \boldsymbol{V_{\alpha}\alpha} - \boldsymbol{V_{\beta}\beta}\right) + \gamma^{\prime} \sum_{(i, j) \in \mathcal{E}} \left\|\boldsymbol{Z}_{\mathcal{C}(i, j)}\boldsymbol{\beta}\right\|_2^2
\end{align}

    Now, let $\hat{\boldsymbol{\alpha}}$, $\hat{\boldsymbol{\beta}}$, and $\hat{\boldsymbol{I}}_{B_{l_t}}$ be the minimiser of the above cost function. Then, by definition,

    \begin{align*}
        \begin{split}
            & \frac{1}{2ndb} \left\|\boldsymbol{x} - \boldsymbol{V_{\alpha}}\hat{\boldsymbol{\alpha}} - \boldsymbol{V_{\beta}}\hat{\boldsymbol{\beta}}\right\|_{\hat{\boldsymbol{I}}_{B_{l_t}}}^2 + \gamma^{\prime} \sum_{(i, j) \in \mathcal{E}} \left\|\boldsymbol{Z}_{\mathcal{C}(i, j)}\hat{\boldsymbol{\beta}}\right\|_2^2 \\
            & \leq \frac{1}{2ndb} \left\|\boldsymbol{x} - \boldsymbol{V_{\alpha}\alpha} - \boldsymbol{V_{\beta}\beta}\right\|_{{\boldsymbol{I}}_{B_{l_t}}}^2 + \gamma^{\prime} \sum_{(i, j) \in \mathcal{E}} \left\|\boldsymbol{Z}_{\mathcal{C}(i, j)}\boldsymbol{\beta}\right\|_2^2\\
            & \leq \frac{1}{2ndb} \left\|\boldsymbol{x} - \boldsymbol{V_{\alpha}\alpha} - \boldsymbol{V_{\beta}\beta}\right\|_{\hat{{\boldsymbol{I}}}_{B_{l_t}}}^2 + \frac{1}{2ndb} \boldsymbol{\epsilon}^{\top}({\boldsymbol{I}}_{B_{l_t}} - \hat{{\boldsymbol{I}}}_{B_{l_t}})\boldsymbol{\epsilon} 
            + \gamma^{\prime} \sum_{(i, j) \in \mathcal{E}} \left\|\boldsymbol{Z}_{\mathcal{C}(i, j)}\boldsymbol{\beta}\right\|_2^2 \\
            & \leq \frac{1}{2ndb} \left\|\boldsymbol{x} - \boldsymbol{V_{\alpha}\alpha} - \boldsymbol{V_{\beta}\beta}\right\|_{\hat{{\boldsymbol{I}}}_{B_{l_t}}}^2 + \frac{1}{n} M^2  + \gamma^{\prime} \sum_{(i, j) \in \mathcal{E}} \left\|\boldsymbol{Z}_{\mathcal{C}(i, j)}\boldsymbol{\beta}\right\|_2^2
        \end{split}
    \end{align*}
    
    On further simplification, we get the following,

    \begin{align*}
        & \frac{1}{2ndb} \left\|\boldsymbol{V_{\alpha}}\left(\hat{\boldsymbol{\alpha}} - \boldsymbol{\alpha}\right) + \boldsymbol{V_{\beta}}\left(\hat{\boldsymbol{\beta}} - \boldsymbol{\beta}\right)\right\|_{\hat{\boldsymbol{I}}_{B_{l_t}}}^2  + \gamma^{\prime} \sum_{(i, j) \in \mathcal{E}} \left\|\boldsymbol{Z}_{\mathcal{C}(i, j)}\hat{\boldsymbol{\beta}}\right\|_2^2 \leq \frac{1}{ndb} G\left(\hat{\boldsymbol{\alpha}}, \hat{\boldsymbol{\beta}}, \hat{\boldsymbol{I}}_{B_{l_t}}\right) + \frac{1}{n} M^2  + \gamma^{\prime} \sum_{(i, j) \in \mathcal{E}} \left\|\boldsymbol{Z}_{\mathcal{C}(i, j)}\boldsymbol{\beta}\right\|_2^2\\
    \end{align*}
    
    where $G(\hat{\boldsymbol{\alpha}}, \hat{\boldsymbol{\beta}}, \hat{\boldsymbol{I}}_{B_{l_t}}) = \boldsymbol{\epsilon}^{\top} \hat{\boldsymbol{I}}_{B_{l_t}} (\boldsymbol{V_{\alpha}}(\hat{\boldsymbol{\alpha}} - \boldsymbol{\alpha}) + \boldsymbol{V_{\beta}}(\hat{\boldsymbol{\beta}} - \boldsymbol{\beta}))$. Since $\hat{\boldsymbol{\alpha}}$ is the minimiser, we can choose $\hat{\boldsymbol{\alpha}}$ such that $\boldsymbol{x} - \boldsymbol{V_{\alpha}}\hat{\boldsymbol{\alpha}} - \boldsymbol{V_{\beta}}\hat{\boldsymbol{\beta}} = 0$. Therefore, $\hat{\boldsymbol{\alpha}} = \boldsymbol{\alpha} + \boldsymbol{V_{\alpha}}^{\top}\boldsymbol{\epsilon}$. Now, we can bound,

    \begin{align*}
            & \frac{1}{ndb}\left|G(\hat{\boldsymbol{\alpha}}, \hat{\boldsymbol{\beta}}, \hat{\boldsymbol{I}}_{B_{l_t}})\right| \\ 
            & = \frac{1}{ndb} \left|\boldsymbol{\epsilon}^{\top} \hat{\boldsymbol{I}}_{B_{l_t}} \left( \boldsymbol{V_{\alpha}}\left(\hat{\boldsymbol{\alpha}} - \boldsymbol{\alpha}\right) +  \boldsymbol{V_{\beta}}\left(\hat{\boldsymbol{\beta}} - \boldsymbol{\beta}\right)\right)\right| \\
            & = \frac{1}{ndb} \left|\boldsymbol{\epsilon}^{\top} \hat{\boldsymbol{I}}_{B_{l_t}} \left( \boldsymbol{V_{\alpha}} \boldsymbol{V_{\alpha}}^{\top}\boldsymbol{\epsilon} + \boldsymbol{V_{\beta}}\left(\hat{\boldsymbol{\beta}} - \boldsymbol{\beta}\right)\right)\right| \\
            & \leq \frac{1}{ndb} \boldsymbol{\epsilon}^{\top} \hat{\boldsymbol{I}}_{B_{l_t}}  \boldsymbol{V_{\alpha}} \boldsymbol{V_{\alpha}}^{\top}\boldsymbol{\epsilon} + \frac{1}{ndb}\left|\boldsymbol{\epsilon}^{\top} \hat{\boldsymbol{I}}_{B_{l_t}}  \boldsymbol{V_{\beta}}     \left(\hat{\boldsymbol{\beta}} - \boldsymbol{\beta}\right)\right| \\   
            & = \frac{1}{ndb} \boldsymbol{\epsilon}^{\top} \hat{\boldsymbol{I}}_{B_{l_t}}  \boldsymbol{V_{\alpha}} \boldsymbol{V_{\alpha}}^{\top}\boldsymbol{\epsilon} + \frac{1}{ndb}\left|\boldsymbol{\epsilon}^{\top} \hat{\boldsymbol{I}}_{B_{l_t}}  \boldsymbol{V_{\beta}}  \boldsymbol{Z}^{-} \boldsymbol{Z} \left(\hat{\boldsymbol{\beta}} - \boldsymbol{\beta}\right)\right| \\
            & = \frac{1}{ndb} \boldsymbol{\epsilon}^{\top} \hat{\boldsymbol{I}}_{B_{l_t}}  \boldsymbol{V_{\alpha}} \boldsymbol{V_{\alpha}}      ^{\top}\boldsymbol{\epsilon}  + \frac{1}{ndb}\left|\sum_{(i, j) \in \mathcal{E}} \boldsymbol{\epsilon}^{\top} \hat{\boldsymbol{I}}_{B_{l_t}}  \boldsymbol{V_{\beta}}  \boldsymbol{Z}_{\mathcal{C}(i, j)}^{-} \boldsymbol{Z}_{\mathcal{C}(i, j)} \left(\hat{\boldsymbol{\beta}} - \boldsymbol{\beta}\right)\right| \\
            & \leq \frac{1}{ndb} \boldsymbol{\epsilon}^{\top} \hat{\boldsymbol{I}}_{B_{l_t}}  \boldsymbol{V_{\alpha}} \boldsymbol{V_{\alpha}}^{\top}\boldsymbol{\epsilon} + \frac{1}{ndb} \sum_{(i, j) \in \mathcal{E}} \left|\boldsymbol{\epsilon}^{\top} \hat{\boldsymbol{I}}_{B_{l_t}}  \boldsymbol{V_{\beta}}  \boldsymbol{Z}_{\mathcal{C}(i, j)}^{-} \boldsymbol{Z}_{\mathcal{C}(i, j)} \left(\hat{\boldsymbol{\beta}} - \boldsymbol{\beta}\right)\right| \\
            & \leq \frac{1}{ndb} \boldsymbol{\epsilon}^{\top} \hat{\boldsymbol{I}}_{B_{l_t}}  \boldsymbol{V_{\alpha}} \boldsymbol{V_{\alpha}}^{\top}\boldsymbol{\epsilon} + \frac{1}{ndb} \sum_{(i, j) \in \mathcal{E}} \left(\left\|( \boldsymbol{Z}_{\mathcal{C}(i, j)}^{-})^{\top} \boldsymbol{V_{\beta}}^{\top}\hat{\boldsymbol{I}}_{B_{l_t}}\boldsymbol{\epsilon}\right\|_2 \right. \left. \cdot \left\| \boldsymbol{Z}_{\mathcal{C}(i, j)} \left(\hat{\boldsymbol{\beta}} -      \boldsymbol{\beta}\right)\right\|_2 \right)\\
            & \leq \frac{1}{ndb} \boldsymbol{\epsilon}^{\top} \hat{\boldsymbol{I}}_{B_{l_t}}  \boldsymbol{V_{\alpha}} \boldsymbol{V_{\alpha}}^{\top}\boldsymbol{\epsilon} + \frac{1}{ndb} \max_{(i, j) \in \mathcal{E}} \left\|( \boldsymbol{Z}_{\mathcal{C}(i, j)}^{-})^{\top} \boldsymbol{V_{\beta}}^{\top}\hat{\boldsymbol{I}}_{B_{l_t}}\boldsymbol{\epsilon}\right\|_2 \cdot \sum_{(i, j) \in \mathcal{E}} \left\| \boldsymbol{Z}_{\mathcal{C}(i, j)} \left(\hat{\boldsymbol{\beta}} - \boldsymbol{\beta}\right)\right\|_2  
    \end{align*}
    
    Next, we derive high-probability bounds for the terms $\boldsymbol{\epsilon}^{\top} \hat{\boldsymbol{I}}_{B_{l_t}}  \boldsymbol{V_{\alpha}} \boldsymbol{V_{\alpha}}^{\top}\boldsymbol{\epsilon}$ and $\max_{(i, j) \in \mathcal{E}} \left\|( \boldsymbol{Z}_{\mathcal{C}(i, j)}^{-})^{\top} \boldsymbol{V_{\beta}}^{\top}\hat{\boldsymbol{I}}_{B_{l_t}}\boldsymbol{\epsilon}\right\|_2$. Now, using the Hanson Wright Inequality,

    \begin{align*}
            & \sup_{\boldsymbol{I}_{B_{l_t}} \in \mathcal{I}} \boldsymbol{\epsilon}^{\top} \hat{\boldsymbol{I}}_{B_{l_t}}  \boldsymbol{V_{\alpha}} \boldsymbol{V_{\alpha}}^{\top}\boldsymbol{\epsilon} \\
            & \leq \mathbb{E}\left[\sup_{\boldsymbol{I}_{B_{l_t}} \in \mathcal{I}} \boldsymbol{\epsilon}^{\top} \hat{\boldsymbol{I}}_{B_{l_t}}  \boldsymbol{V_{\alpha}} \boldsymbol{V_{\alpha}}^{\top}\boldsymbol{\epsilon}\right] + c\left(M\sqrt{r} \ \mathbb{E}\left[\sup_{\boldsymbol{I}_{B_{l_t}} \in \mathcal{I}} \left\|\hat{\boldsymbol{I}}_{B_{l_t}}  \boldsymbol{V_{\alpha}} \boldsymbol{V_{\alpha}}^{\top}\boldsymbol{\epsilon}\right\|_{sp}\right]\right) + c\left(rM^2 \sup_{\boldsymbol{I}_{B_{l_t}} \in \mathcal{I}} \left\|I_{B_{l_t}}\right\|\right)\\
            & \leq \mathbb{E}\left[\sup_{\boldsymbol{I}_{B_{l_t}} \in \mathcal{I}} \text{tr}\left(\boldsymbol{\epsilon}^{\top} \hat{\boldsymbol{I}}_{B_{l_t}}  \boldsymbol{V_{\alpha}} \boldsymbol{V_{\alpha}}^{\top}\boldsymbol{\epsilon}\right)\right] + c\left(M\sqrt{r} \ \mathbb{E}\left[\sup_{\boldsymbol{I}_{B_{l_t}} \in \mathcal{I}} \left\|\hat{\boldsymbol{I}}_{B_{l_t}}  \boldsymbol{V_{\alpha}} \boldsymbol{V_{\alpha}}^{\top}\right\|_{sp}\left\|\boldsymbol{\epsilon}\right\|_2\right]\right) + c\left(rM^2 \sup_{\boldsymbol{I}_{B_{l_t}} \in \mathcal{I}} 1\right)\\
            & = \mathbb{E}\left[\sup_{\boldsymbol{I}_{B_{l_t}} \in \mathcal{I}} \text{tr}\left(\hat{\boldsymbol{I}}_{B_{l_t}}  \boldsymbol{V_{\alpha}} \boldsymbol{V_{\alpha}}^{\top}\boldsymbol{\epsilon}\boldsymbol{\epsilon}^{\top}\right)\right] + c\left(M\sqrt{r} \sup_{\boldsymbol{I}_{B_{l_t}} \in \mathcal{I}} \left\|\hat{\boldsymbol{I}}_{B_{l_t}}  \boldsymbol{V_{\alpha}} \boldsymbol{V_{\alpha}}^{\top}\right\|_{sp} \ \mathbb{E}\left[\left\|\boldsymbol{\epsilon}\right\|_2\right] + rM^2\right) \\  
            &\leq \mathbb{E} \left[ 
            \sup_{\boldsymbol{I}_{B_{l_t}} \in \mathcal{I}} 
            \sqrt{ \text{tr}\left( \hat{\boldsymbol{I}}_{B_{l_t}}^2 \right) } \cdot \right.\left. 
            \sqrt{ \text{tr}\left( 
            \left( \boldsymbol{V}_{\alpha} \boldsymbol{V}_{\alpha}^{\top} 
            \boldsymbol{\epsilon} \boldsymbol{\epsilon}^{\top} \right)^{\top}
            \left( \boldsymbol{V}_{\alpha} \boldsymbol{V}_{\alpha}^{\top}
            \boldsymbol{\epsilon} \boldsymbol{\epsilon}^{\top} \right)
            \right) }
            \right] + c\left(M\sqrt{r} \sup_{\boldsymbol{I}_{B_{l_t}} \in \mathcal{I}} \left\|\hat{\boldsymbol{I}}_{B_{l_t}}\right\|_{sp} \left\| \boldsymbol{V_{\alpha}} \boldsymbol{V_{\alpha}}^{\top}\right\|_{sp} \ \mathbb{E}\left[M\sqrt{nd}\right]\right) \\
            & \qquad+ crM^2 \\
            & = \mathbb{E}\left[\sup_{\boldsymbol{I}_{B_{l_t}} \in \mathcal{I}} \sqrt{db} \sqrt{\text{tr}\left(\boldsymbol{\epsilon}\boldsymbol{\epsilon}^{\top} \boldsymbol{V_{\alpha}} \boldsymbol{V_{\alpha}}^{\top} \boldsymbol{V_{\alpha}} \boldsymbol{V_{\alpha}}^{\top}\boldsymbol{\epsilon}\boldsymbol{\epsilon}^{\top}\right)}\right] + c\left(M^2\sqrt{ndr} + rM^2\right) \\    
            & = \sqrt{db} \ \mathbb{E}\left[\sqrt{\text{tr}\left(\boldsymbol{\epsilon}\boldsymbol{\epsilon}^{\top} \boldsymbol{V_{\alpha}} \boldsymbol{V_{\alpha}}^{\top}\boldsymbol{\epsilon}\boldsymbol{\epsilon}^{\top}\right)}\right] + c\left(M^2\sqrt{ndr} + rM^2\right)\\
            & = \sqrt{db} \ \mathbb{E}\left[\left\|\boldsymbol{\epsilon}\right\|_2 \sqrt{\text{tr}\left( \boldsymbol{V_{\alpha}} \boldsymbol{V_{\alpha}}^{\top}\boldsymbol{\epsilon}\boldsymbol{\epsilon}^{\top}\right)}\right] + c\left(M^2\sqrt{ndr} + rM^2\right) \\
        \end{align*}       
        \begin{align*}
            & \leq Md\sqrt{nb} \ \mathbb{E}\left[\sqrt{\text{tr}\left( \boldsymbol{V_{\alpha}} \boldsymbol{V_{\alpha}}^{\top}\boldsymbol{\epsilon}\boldsymbol{\epsilon}^{\top}\right)}\right] + c\left(M^2\sqrt{ndr} + rM^2\right) \\
            & \leq Md\sqrt{nb} \ \sqrt{\mathbb{E}\left[\text{tr}\left( \boldsymbol{V_{\alpha}} \boldsymbol{V_{\alpha}}^{\top}\boldsymbol{\epsilon}\boldsymbol{\epsilon}^{\top}\right)\right]} + c\left(M^2\sqrt{ndr} + rM^2\right) \\
            & = Md\sqrt{nb} \ \sqrt{\text{tr}\left(\mathbb{E}\left[ \boldsymbol{V_{\alpha}} \boldsymbol{V_{\alpha}}^{\top}\boldsymbol{\epsilon}\boldsymbol{\epsilon}^{\top}\right]\right)} + c\left(M^2\sqrt{ndr} + rM^2\right) \\
            & = Md\sqrt{nb} \ \sqrt{\text{tr}\left( \boldsymbol{V_{\alpha}} \boldsymbol{V_{\alpha}}^{\top}\mathbb{E}\left[\boldsymbol{\epsilon}\boldsymbol{\epsilon}^{\top}\right]\right)} + c\left(M^2\sqrt{ndr} + rM^2\right) \\
            & = Md\sqrt{nb} \ \sqrt{\text{tr}\left( \boldsymbol{V_{\alpha}} \boldsymbol{V_{\alpha}}^{\top}\left(\sigma^2 I\right)\right)} + c\left(M^2\sqrt{ndr} + rM^2\right)\\
            & = Md\sigma\sqrt{nb} \ \sqrt{\text{tr}\left( \boldsymbol{V_{\alpha}} \boldsymbol{V_{\alpha}}^{\top}\right)} + c\left(M^2\sqrt{ndr} + rM^2\right) \\
            & = M^2d\sqrt{nb} \ \sqrt{\text{tr}\left( \boldsymbol{V_{\alpha}}^{\top} \boldsymbol{V_{\alpha}}\right)} + c\left(M^2\sqrt{ndr} + rM^2\right) \\
            & = M^2d\sqrt{ndb} + c\left(M^2\sqrt{ndr} + rM^2\right) \\
            & = M^2\left(d\sqrt{ndb} + c\sqrt{ndr} + cr\right)
    \end{align*}

    Thus, from the above analysis, we get

    \begin{align*}
        &\mathbb{P}\left(\frac{1}{ndb} \sup_{\boldsymbol{I}_{B_{l_t}} \in \mathcal{I}} \boldsymbol{\epsilon}^{\top} \hat{\boldsymbol{I}}_{B_{l_t}}  \boldsymbol{V_{\alpha}} \boldsymbol{V_{\alpha}}^{\top}\boldsymbol{\epsilon} \right.\left.\geq M^2\left(\frac{d}{\sqrt{ndb}} + c\frac{1}{b\sqrt{nd}}\sqrt{r} + \frac{cr}{ndb}\right)\right) \leq e^{-r}
    \end{align*}

    Taking $r = \log\left(\frac{1}{\delta}\right)$, we get,

    \begin{align*}
        &\mathbb{P}\left(\frac{1}{ndb} \sup_{\boldsymbol{I}_{B_{l_t}} \in \mathcal{I}} \boldsymbol{\epsilon}^{\top} \hat{\boldsymbol{I}}_{B_{l_t}}  \boldsymbol{V_{\alpha}} \boldsymbol{V_{\alpha}}^{\top}\boldsymbol{\epsilon} \right. \left.\geq M^2\left(\frac{d}{\sqrt{ndb}} + c\frac{1}{b\sqrt{nd}}\sqrt{\log\left(\frac{1}{\delta}\right)} + c\frac{\log\left(\frac{1}{\delta}\right)}{ndb}\right)\right) \leq \delta\\
    \end{align*}

    Thus, with probability atleast $1 - \delta$,

    \begin{align*}
        &\frac{1}{ndb} \boldsymbol{\epsilon}^{\top} \hat{\boldsymbol{I}}_{B_{l_t}}  \boldsymbol{V_{\alpha}} \boldsymbol{V_{\alpha}}^{\top}\boldsymbol{\epsilon} \leq \frac{1}{ndb} \sup_{\boldsymbol{I}_{B_{l_t}} \in \mathcal{I}} \boldsymbol{\epsilon}^{\top} \hat{\boldsymbol{I}}_{B_{l_t}}  \boldsymbol{V_{\alpha}} \boldsymbol{V_{\alpha}}^{\top}\boldsymbol{\epsilon} \leq M^2\left(\frac{d}{\sqrt{ndb}} + c\frac{1}{b\sqrt{nd}}\sqrt{\log\left(\frac{1}{\delta}\right)} + c\frac{\log\left(\frac{1}{\delta}\right)}{ndb}\right)\\
    \end{align*}

    Let $y_j = \boldsymbol{e}_j^{\top}( \boldsymbol{Z}_{\mathcal{C}(i, j)}^{-})^{\top} \boldsymbol{V_{\beta}}^{\top}\hat{\boldsymbol{I}}_{B_{l_t}}\boldsymbol{\epsilon}$. Now, $y_j$ is a univariate, bounded random variable with $\left|y_j\right| \leq \frac{M}{\sqrt{n}}$. Thus,

    \begin{align*}
        &\quad \max_{(i, j) \in \mathcal{E}} \left\|( \boldsymbol{Z}_{\mathcal{C}(i, j)}^{-})^{\top} \boldsymbol{V_{\beta}}^{\top}\hat{\boldsymbol{I}}_{B_{l_t}}\boldsymbol{\epsilon}\right\|_{\infty} = \max_j \left|y_j\right| \leq \frac{M}{\sqrt{n}}\\
    & \Rightarrow \frac{1}{ndb}\max_{(i, j) \in \mathcal{E}} \left\|( \boldsymbol{Z}_{\mathcal{C}(i, j)}^{-})^{\top} \boldsymbol{V_{\beta}}^{\top}\hat{\boldsymbol{I}}_{B_{l_t}}\boldsymbol{\epsilon}\right\|_2 \leq \frac{1}{ndb}\max_{(i, j) \in \mathcal{E}} \left\|( \boldsymbol{Z}_{\mathcal{C}(i, j)}^{-})^{\top} \boldsymbol{V_{\beta}}^{\top}\hat{\boldsymbol{I}}_{B_{l_t}}\boldsymbol{\epsilon}\right\|_{\infty} \leq \frac{M}{ndb\sqrt{n}}\\
    \end{align*}

    Note $\gamma^{\prime} \geq \frac{M}{ndb\sqrt{n}}$ implies that, \[\gamma^{\prime} \geq \frac{1}{ndb}\max_{(i, j) \in \mathcal{E}} \left\|( \boldsymbol{Z}_{\mathcal{C}(i, j)}^{-})^{\top} \boldsymbol{V_{\beta}}^{\top}\hat{\boldsymbol{I}}_{B_{l_t}}\boldsymbol{\epsilon}\right\|_2.\] 
    
    So we get,

    \begin{align*}
        \frac{1}{ndb}&\left|G(\hat{\boldsymbol{\alpha}}, \hat{\boldsymbol{\beta}}, \hat{\boldsymbol{I}}_{B_{l_t}})\right| \leq M^2\left(\frac{d}{\sqrt{ndb}} + c\frac{1}{b\sqrt{nd}}\sqrt{\log\left(\frac{1}{\delta}\right)} \right. \left. +   c\frac{\log\left(\frac{1}{\delta}\right)}{ndb}\right) + \gamma^{\prime} \sum_{(i, j) \in \mathcal{E}} \left\| \boldsymbol{Z}_{\mathcal{C}(i, j)} \left(\hat{\boldsymbol{\beta}} - \boldsymbol{\beta}\right)\right\|_2
    \end{align*}

    holds with probability atleast $1 - \delta$. Now combining all the results we get, with probability atleast $1 - \delta$

    \begin{align*}
            & \frac{1}{2ndb} \left\| \boldsymbol{V_{\alpha}}\left(\hat{\boldsymbol{\alpha}} - \boldsymbol{\alpha}\right) +  \boldsymbol{V_{\beta}}\left(\hat{\boldsymbol{\beta}} - \boldsymbol{\beta}\right)\right\|_{\hat{\boldsymbol{I}}_{B_{l_t}}}^2 + \gamma^{\prime} \sum_{(i, j) \in \mathcal{E}} \left\| \boldsymbol{Z}_{\mathcal{C}(i, j)}\hat{\boldsymbol{\beta}}\right\|_2^2 \\
            & \leq M^2\left(\frac{d}{\sqrt{ndb}} + c\frac{1}{b\sqrt{nd}}\sqrt{\log\left(\frac{1}{\delta}\right)} + c\frac{\log\left(\frac{1}{\delta}\right)}{ndb}\right) + \frac{1}{n} M^2 + \gamma^{\prime} \sum_{(i, j) \in \mathcal{E}} \left\| \boldsymbol{Z}_{\mathcal{C}(i, j)} \left(\hat{\boldsymbol{\beta}} - \boldsymbol{\beta}\right)\right\|_2 + \gamma^{\prime} \sum_{(i, j) \in \mathcal{E}} \left\| \boldsymbol{Z}_{\mathcal{C}(i, j)}\boldsymbol{\beta}\right\|_2^2\\
    \end{align*}
    
    Upon rearranging the terms, we obtain that with probability atleast $1 - \delta$,

    \begin{align*}
            & \frac{1}{2ndb} \left\| \boldsymbol{V_{\alpha}}\left(\hat{\boldsymbol{\alpha}} - \boldsymbol{\alpha}\right) +  \boldsymbol{V_{\beta}}\left(\hat{\boldsymbol{\beta}} - \boldsymbol{\beta}\right)\right\|_{\hat{\boldsymbol{I}}_{B_{l_t}}}^2 \\
            & \leq M^2\left(\frac{1}{n} + \frac{d}{\sqrt{ndb}} + c\frac{1}{b\sqrt{nd}}\sqrt{\log\left(\frac{1}{\delta}\right)} 
            + c\frac{\log\left(\frac{1}{\delta}\right)}{ndb}\right) \\
            & \quad \quad  + \gamma^{\prime} \left[\sum_{(i, j) \in \mathcal{E}} \bigg(\left\| \boldsymbol{Z}_{\mathcal{C}(i, j)}\hat{\boldsymbol{\beta}}\right\|_2 
            - \left\| \boldsymbol{Z}_{\mathcal{C}(i, j)}\hat{\boldsymbol{\beta}}\right\|_2^2\bigg)  + \sum_{(i, j) \in \mathcal{E}} \left\| \boldsymbol{Z}_{\mathcal{C}(i, j)}\boldsymbol{\beta}\right\|_2 
            + \sum_{(i, j) \in \mathcal{E}} \left\| \boldsymbol{Z}_{\mathcal{C}(i, j)}\boldsymbol{\beta}\right\|_2^2 \right]\\
            & \leq M^2\left(\frac{1}{n} + \frac{d}{\sqrt{ndb}} + c\frac{1}{b\sqrt{nd}}\sqrt{\log\left(\frac{1}{\delta}\right)} + c\frac{\log\left(\frac{1}{\delta}\right)}{ndb}\right) + \gamma^{\prime} \left[\frac{\left|\mathcal{E}\right|}{4} +  \sum_{(i, j) \in \mathcal{E}} \left\| \boldsymbol{Z}_{\mathcal{C}(i, j)}\boldsymbol{\beta}\right\|_2 + \sum_{(i, j) \in \mathcal{E}} \left\| \boldsymbol{Z}_{\mathcal{C}(i, j)}\boldsymbol{\beta}\right\|_2^2 \right] \\
            & \leq M^2\left(\frac{1}{n} + \frac{d}{\sqrt{ndb}} + c\frac{1}{b\sqrt{nd}}\sqrt{\log\left(\frac{1}{\delta}\right)} + c\frac{\log\left(\frac{1}{\delta}\right)}{ndb}\right) + \gamma^{\prime} \left[\frac{\left|\mathcal{E}\right|}{4} +\sum_{(i, j) \in \mathcal{E}} \left\|\boldsymbol{D}_{\mathcal{C}(i, j)}\boldsymbol{u}\right\|_2 +  \sum_{(i, j) \in \mathcal{E}} \left\|\boldsymbol{D}_{\mathcal{C}(i, j)}\boldsymbol{u}\right\|_2^2 \right] \\
    \end{align*}

    \label{app:proof_theorem}
\end{proof}

\subsection{Proof of Corollary 1} \label{col1_proof}

\begin{corollary}
    Suppose $\left\|\boldsymbol{D}_{\mathcal{C}(i, j)}\boldsymbol{u}\right\|_2 \leq C$, for all $1 \leq i, j \leq n$, for some constant $C$, $\left|\mathcal{E}\right| \leq kn$ and $\gamma^{\prime} \geq \frac{M}{ndb\sqrt{n}}$. If $d = o(n)$, then $\frac{1}{2ndb}\|\boldsymbol{\hat{u} - u}\|_{\hat{\boldsymbol{I}}_{B_{l_t}}}^2 \overset{p}{\rightarrow} 0$ as $n, d \rightarrow \infty$.
\end{corollary}

\begin{proof}
    For any fixed $\delta$, we know from Theorem 1 that with probability atleast $1 - \delta$.

    \begin{align*}
         \frac{1}{2ndb}\|\boldsymbol{\hat{u} - u}\|_{\hat{\boldsymbol{I}}_{B_{l_t}}}^2 &\leq M^2 \left(\frac{1}{n} + \frac{d}{\sqrt{ndb}} + c\frac{1}{b\sqrt{nd}}\sqrt{\log\left(\frac{1}{\delta}\right)} + c\frac{\log\left(\frac{1}{\delta}\right)}{ndb}\right) + \gamma^{\prime} \frac{\left|\mathcal{E}\right|}{4}\\
        & \quad \quad + \gamma^{\prime}\left[ \sum_{(i, j) \in \mathcal{E}} \left\|\boldsymbol{D}_{\mathcal{C}(i, j)}\boldsymbol{u}\right\|_2 + \sum_{(i, j) \in \mathcal{E}} \left\|\boldsymbol{D}_{\mathcal{C}(i, j)}\boldsymbol{u}\right\|_2^2\right] \\
        & \leq M^2 \left(\frac{1}{n} + \frac{d}{\sqrt{ndb}} + c\frac{1}{b\sqrt{nd}}\sqrt{\log\left(\frac{1}{\delta}\right)} + c\frac{\log\left(\frac{1}{\delta}\right)}{ndb}\right) + \gamma^{\prime} \frac{\left|\mathcal{E}\right|}{4} + \left(C + C^2\right)\gamma^{\prime}\left|\mathcal{E}\right| \\
        & \leq M^2 \left(\frac{1}{n} + \frac{d}{\sqrt{ndb}} + c\frac{1}{b\sqrt{nd}}\sqrt{\log\left(\frac{1}{\delta}\right)} + c\frac{\log\left(\frac{1}{\delta}\right)}{ndb}\right) + \gamma^{\prime} \frac{kn}{4} + \left(C + C^2\right)\gamma^{\prime}kn \\
        & \rightarrow 0 \quad \text{as} \quad n, d \rightarrow \infty.
    \end{align*}

    Thus, for any fixed $\epsilon > 0$, $P\left(\frac{1}{2ndb}\|\boldsymbol{\hat{u} - u}\|_{\hat{\boldsymbol{I}}_{B_{l_t}}}^2 > \epsilon\right) \leq \delta$ as $n, p \rightarrow \infty$. Hence, $\frac{1}{2ndb}\|\boldsymbol{\hat{u} - u}\|_{\hat{\boldsymbol{I}}_{B_{l_t}}}^2 \overset{p}{\rightarrow} 0$.

    \label{app:proof_cor1}
\end{proof}

\subsection{Proof of Corollary 2} \label{col2_proof}

\begin{corollary}
    Suppose $\left\|\boldsymbol{D}_{\mathcal{C}(i, j)}\boldsymbol{u}\right\|_2 \leq C$, for all $1 \leq i, j \leq n$, for some constant $C$, $\left|\mathcal{E}\right| \leq kn$ and $\gamma^{\prime} \geq \frac{M}{ndb\sqrt{n}}$. Then $\frac{1}{2ndb}\|\boldsymbol{\hat{u} - u}\|_{\hat{\boldsymbol{I}}_{B_{l_t}}}^2 = O\left(\frac{1}{\sqrt{n}}\right)$.
\end{corollary}

\begin{proof}
    For any fixed $\delta$, we know from Theorem $1$ that with probability atleast $1 - \delta$.

    \begin{align*}
            & \frac{1}{2ndb}\|\boldsymbol{\hat{u} - u}\|_{\hat{\boldsymbol{I}}_{B_{l_t}}}^2 \\
            & \leq M^2 \left(\frac{1}{n} + \frac{d}{\sqrt{ndb}} + c\frac{1}{b\sqrt{nd}}\sqrt{\log\left(\frac{1}{\delta}\right)} + c\frac{\log\left(\frac{1}{\delta}\right)}{ndb}\right) + \gamma^{\prime} \frac{\left|\mathcal{E}\right|}{4} + \gamma^{\prime}\left[ \sum_{(i, j) \in \mathcal{E}} \left\|\boldsymbol{D}_{\mathcal{C}(i, j)}\boldsymbol{u}\right\|_2 + \sum_{(i, j) \in \mathcal{E}} \left\|\boldsymbol{D}_{\mathcal{C}(i, j)}\boldsymbol{u}\right\|_2^2\right] \\        
            & \leq M^2 \left(\frac{1}{n} + \frac{d}{\sqrt{ndb}} + c\frac{1}{b\sqrt{nd}}\sqrt{\log\left(\frac{1}{\delta}\right)} + c\frac{\log\left(\frac{1}{\delta}\right)}{ndb}\right) + \gamma^{\prime} \frac{\left|\mathcal{E}\right|}{4} + \left(C + C^2\right)\gamma^{\prime}\left|\mathcal{E}\right| \\
            & \leq M^2 \left(\frac{1}{n} + \frac{d}{\sqrt{ndb}} + c\frac{1}{b\sqrt{nd}}\sqrt{\log\left(\frac{1}{\delta}\right)} + c\frac{\log\left(\frac{1}{\delta}\right)}{ndb}\right) + \gamma^{\prime} \frac{kn}{4} + \left(C + C^2\right)\gamma^{\prime}kn \\
            & = O\left(\frac{1}{\sqrt{n}}\right) \\
    \end{align*}

    Thus, $\sqrt{n} \frac{1}{2ndb}\|\boldsymbol{\hat{u} - u}\|_{\hat{\boldsymbol{I}}_{B_{l_t}}}^2 = O(1)$. This implies that there exists a constant $\epsilon$ such that, \\$P\left(\frac{1}{2ndb}\|\boldsymbol{\hat{u} - u}\|_{\hat{\boldsymbol{I}}_{B_{l_t}}}^2 \leq \epsilon\right) \geq 1 - \delta$, for all $n \in \mathbb{N}$. Hence, $\sqrt{n} \frac{1}{2ndb}\|\boldsymbol{\hat{u} - u}\|_{\hat{\boldsymbol{I}}_{B_{l_t}}}^2$ is tight.

    \label{app:proof_cor2}
\end{proof}

\subsection{Derivation of Time Complexity \label{time_complexity}}
The main loop in Algorithm 1 
performs $N$ iterations. It is enough to calculate the complexity of each step inside this loop. Constructing a random partition of $l$ buckets of $b$ points takes $O(lb) = O(n)$ steps. Calculating $MoM_B(\mathbf{U})$ for each bucket takes $O(bd)$ steps. Therefore, finding the median bucket $B_{l_t}$, takes $O(lbd) = O(nd)$ steps. Calculating each $g_i$ requires checking whether the index $i$ is in $B_{l_t}$, which takes at most $O(b)$ checks, evaluating $(\mathbf{u}_i^{(t)} - \mathbf{x}_i)$, which takes $O(d)$ constant time evaluations, and evaluating $\sum_{j} w_{ij} (\mathbf{u}_i^{(t)} - \mathbf{u}_j^{(t)}) \mathbb{1} (||\mathbf{u}_i^{(t)} - \mathbf{u}_j^{(t)}||_2^2 < \mu)$, which takes $O(kd)$ constant-time evaluations. Therefore, calculating all $g_i$'s require at most $O(n(b + d + kd)) = O(nkd)$ steps. Calculating each $m_i$, $v_i$ and $u_i$ takes $O(d)$ constant-time evaluations. Therefore, calculating all $m_i$'s, $v_i$'s, and $u_i$'s require $O(nd)$ constant-time evaluations. Hence, the per-iteration complexity of the main loop is $O(nkd)$. Therefore, the complexity of Algorithm 1 
is $O(Nnkd)$. 

\subsection{Description of datasets} \label{data_description}
Table \ref{tab:dataset_list} contains the name and description of all 18 datasets that were used in the experiments.

\begin{table*}[!hbt]
    \centering
    \small
    \begin{tabular}{clcccc}
    \hline
    & \textbf{Datasets} & \textbf{Type} & \textbf{No. of Data-points} & \textbf{Input Dimension} & \textbf{No. of Clusters}\\
    \hline
    1.& Iris&Real &150 &4 &3 \\
    2.& Newthyroid           &Real &215 &5 &3 \\
    3.& Ecoli&Real &336 &7 &8 \\
    4.& Wisconsin            &Real &683 &9 &2 \\
    5.& Wine&Real &178 &13 &3 \\
    6.& Zoo &Real &101 &16 &7 \\
    7.& Dermatology          &Real &358 &34 &6 \\
    8.& Brain&Real &42 &5597 &5 \\
    9.& Lung&Real &203	&3312 &5 \\
    10.& Lymphoma (bio)       &Real &96	&4026 &9 \\
    11.& Coil 20              &Real &1440 &1024 &20 \\
    12.& Wdbc&Real &569 &30 &2 \\
    13.& Lung-discrete        &Real &73	&325 &7 \\
    14.& ORLRaws10p           &Real &100 &10304 &10 \\
    15.& Lymphoma (microarray)&Real &62 &4026 &3 \\
    16.& Blobs&Simulated &1500 &2 &3 \\
    17.& Circles              &Simulated &700 &2 &2 \\
    18.& Moons&Simulated &1500 &2 &2 \\
    \hline
    \end{tabular}
    \caption{List of Datasets}
    \label{tab:dataset_list}
\end{table*}

\subsection{Synthetic Data \label{app:synthetic_data_study}}
In this section, we get empirical results on 3 simulated datasets. We contaminate every dataset by adding noise points following a specific uniform distribution. The contamination levels are: 0$\%$, 5$\%$, 10$\%$, 15$\%$ and 20$\%$.

\begin{itemize}
    \item \textbf{Blobs}: 3 blobs of data points containing 500 points each simulated from bi-variate Gaussian distribution with different means and covariance matrix as the $\mathbf{I}_{2\times2}$. Noise is simulated uniformly from the smallest enclosing axis-parallel rectangle.

    \item \textbf{Circles}: This dataset contains a large circle of radius 1 (consisting of 350 points) containing a smaller circle of radius 0.25 (consisting of 350 points) in two-dimensional space, forming a non-linearly separable pattern. Each point in the dataset is assigned to one of two classes, representing the respective circle to which it belongs. We artificially contaminate the dataset by introducing random points in the middle-annulus to understand which algorithm best separates the two clusters.

    \item \textbf{Moons}: The two moons dataset is a synthetic dataset consisting of two interleaving crescent-shaped clusters resembling two half-moons. The data contains 1500 two-dimensional points, and the two clusters are non-linearly separable. Noise is simulated uniformly from the smallest enclosing axis-parallel rectangle.
\end{itemize}

 \begin{figure}[!hbt]
    \centering
    \begin{subfigure}{0.45\textwidth}
        \centering
        \includegraphics[width = 1\linewidth]{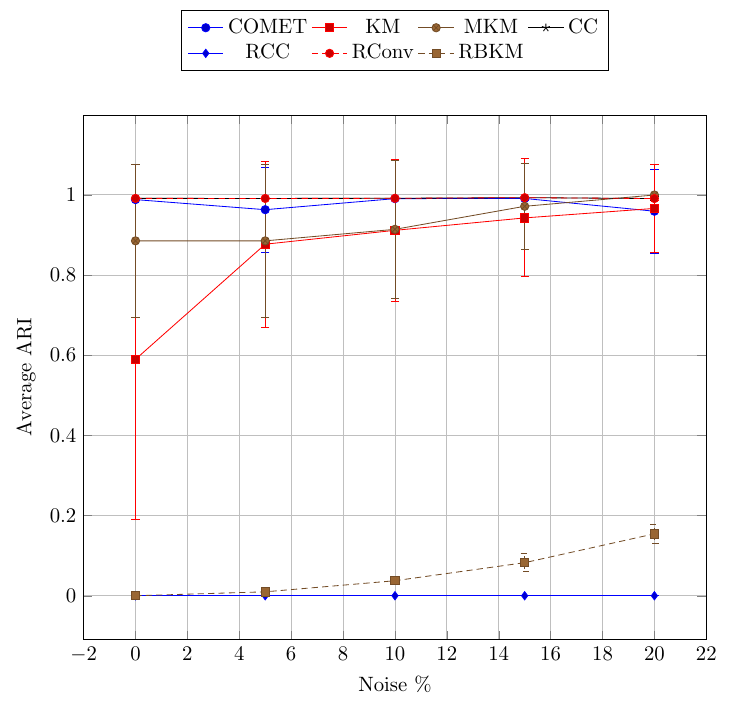}
    \end{subfigure}
    \begin{subfigure}{0.45\textwidth}
        \centering
        \includegraphics[width = 1\linewidth]{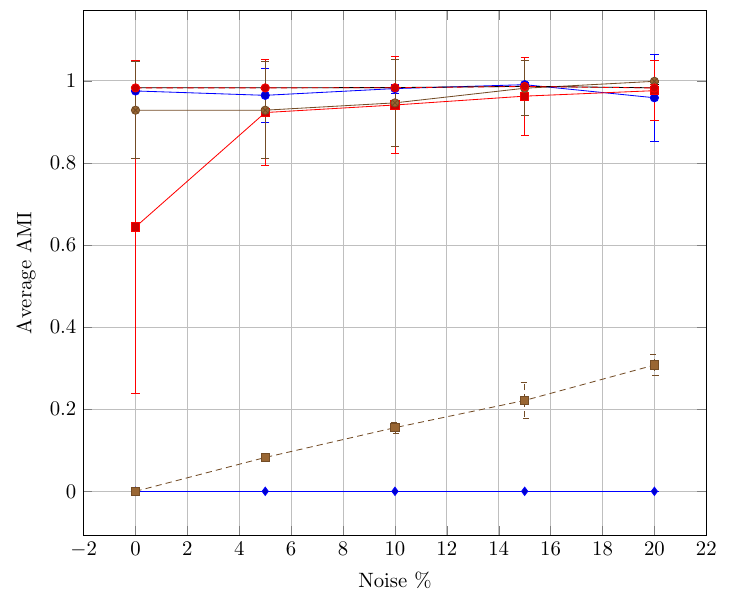}
    \end{subfigure}
    \caption{Performance of Algorithms on Blobs Dataset (ARI and AMI Values)}
    \label{fig:blob_dataset}
\end{figure}

Figures \ref{fig:blob_dataset}, \ref{fig:circles_datatset} and \ref{fig:moons_dataset} represent a comparison of COMET with other SOTA algorithms for different noise levels on different datasets. Figure \ref{fig:cluster_grid} gives a visual representation of the outputs of different algorithms at 10$\%$ noise. It is clearly visible that COMET captures the clustering pattern well and also properly classifies noise.

\begin{figure}[hbt!]
    \centering
    \begin{subfigure}{0.45\textwidth}
        \centering
        \includegraphics[width = 1\linewidth]{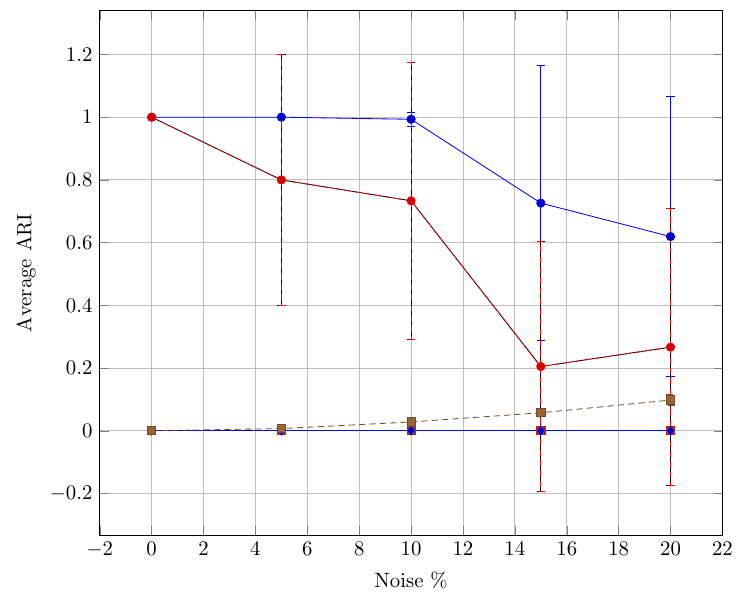}
    \end{subfigure}
    \begin{subfigure}{0.45\textwidth}
        \centering
        \includegraphics[width = 1\linewidth]{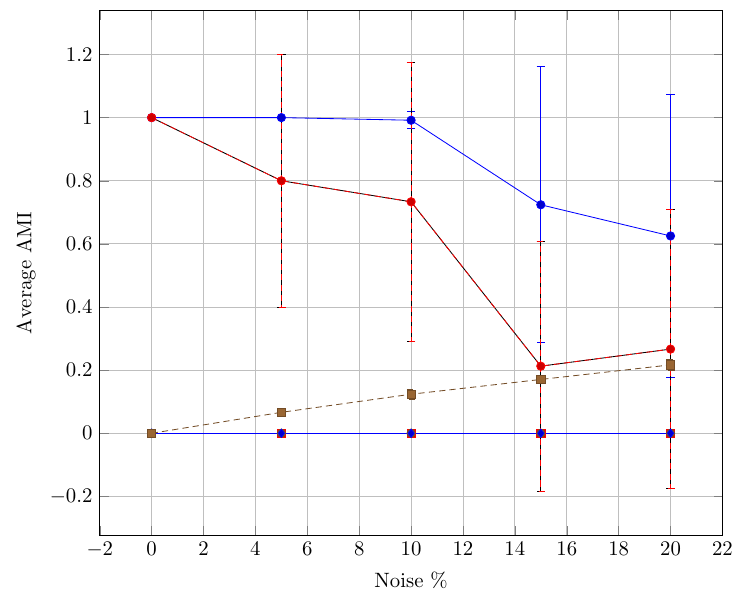}    
    \end{subfigure}
    \caption{Performance of Algorithms on Circles Datasets (ARI and AMI Values)}
    \label{fig:circles_datatset}
\end{figure}
\begin{figure}[htb!]
    \centering
    \begin{subfigure}{0.45\textwidth}
        \centering
        \includegraphics[width = 1\linewidth]{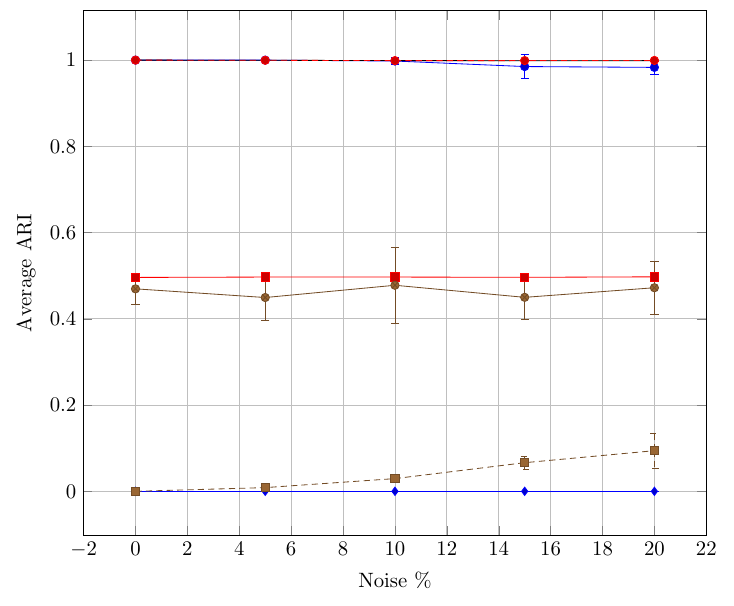}
    \end{subfigure}
    \begin{subfigure}{0.45\textwidth}
        \centering
        \includegraphics[width = 1\linewidth]{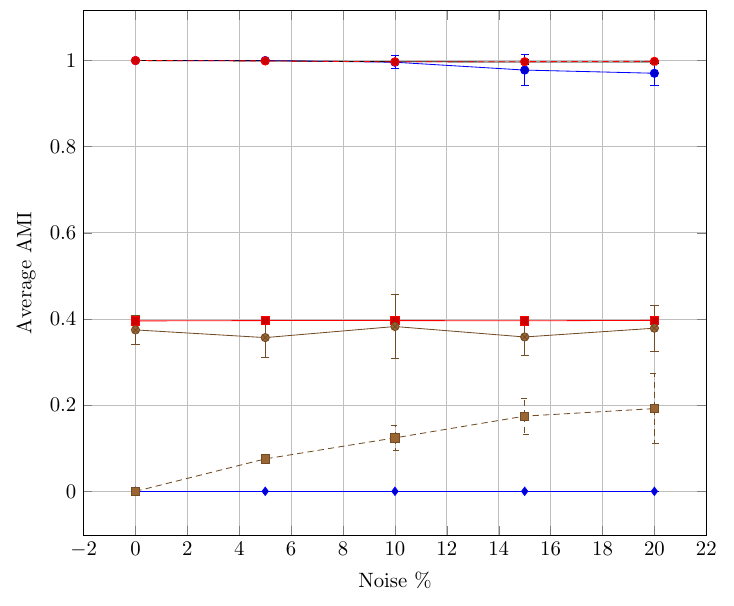}
    \end{subfigure}
    \caption{Performance of Algorithms on Moons Dataset (ARI and AMI Values)}
    \label{fig:moons_dataset}
\end{figure}

\textbf{Discussion: }
As shown in Figure \ref{fig:blob_dataset}, \ref{fig:circles_datatset}, \ref{fig:moons_dataset} and \ref{fig:cluster_grid}, clustering results vary significantly across algorithms. Our proposed algorithm, \textbf{COMET}, excels in identifying true clusters, maintaining high ARI values even as noise increases, demonstrating robustness. In contrast, $k-$means and MoM $k-$means struggle to detect the underlying structure which is a limitation of the $k-$means framework. Convex Clustering performs similarly to COMET in identifying cluster structure and count, but struggles with noise, whereas COMET effectively isolates noise, showcasing its superior performance in noisy environments.

\subsection{Real-life Data \label{app:real_life_full_study}}
Here we have provided the performance of the selected algorithms on the real-life datasets mentioned earlier in Table \ref{tab:dataset_list}, which are not included in Table 2 in the main paper. Here, \text{$k^\ast$} refers to an estimated number of clusters. The actual number of clusters is indicated as $k$. Here the standard deviation is that of the performance measure, not of the mean statistic.

Observe in Table \ref{Real_data_full}, that even though in some datasets like Zoo, Coil20, Lymphoma (Microarray), etc., COMET is not the best performing algorithm, still it gives performance close to the respective best performing algorithms.

\begin{table*}[!htb]
\centering
\small

\begin{tabular}{l>{\centering\arraybackslash}p{.65cm}>{\centering\arraybackslash}p{1.5cm}>{\centering\arraybackslash}p{1.7cm}>{\centering\arraybackslash}p{1.7cm}>{\centering\arraybackslash}p{1.7cm}>{\centering\arraybackslash}p{1.7cm}>{\centering\arraybackslash}p{1.7cm}>{\centering\arraybackslash}p{1.7cm}}
\hline
\textbf{Dataset} & \textbf{Index} & \textbf{KM} & \textbf{MKM} & \textbf{CC} & \textbf{RCC} & \textbf{RConv} & \textbf{RBKM} & \textbf{COMET}\\
\hline

\multirow{3}{6em}{Iris ($k$ = 3)}
&\text{$k^\ast$}
&3.21$\pm$1.01 
&2.91$\pm$0.79 
&2.57$\pm$0.51 
&147.90$\pm$0.70 
&3.71$\pm$0.47
&2.90$\pm$0.80 
&3.21$\pm$0.43 \\
&\text{ARI}
&0.56$\pm$0.06$^\sim$  
&\textbf{0.59$\pm$0.06}
&0.55$\pm$0.01$^\dagger$ 
&0.001$\pm$0.00$^\dagger$ 
& 0.46$\pm$0.06$^\dagger$
&\textbf{0.59$\pm$0.06 }
&0.55$\pm$0.038$^\sim$ \\
&\text{AMI}
&0.66$\pm$0.05$^\sim$ 
&0.67$\pm$0.04$^\sim$ 
&\textbf{0.71$\pm$0.01} 
&0.006$\pm$0.002$^\dagger$ 
&0.61$\pm$0.03$^\dagger$
&0.67$\pm$0.04$^\sim$ 
&0.65$\pm$0.03$^\dagger$ \\

\hline



\multirow{3}{6em}{Ecoli ($k$ = 8)}
&\text{$k^\ast$}
&3.20$\pm$2.24
&1.92$\pm$1.82
&34.21$\pm$1.25 
&334.80$\pm$2.56 
& 16.86$\pm$2.66
&2.00$\pm$0.00 
&10.6$\pm$1.95 \\
&\text{ARI}
&0.27$\pm$0.23$^\dagger$ 
&0.10$\pm$0.17$^\dagger$
&0.47$\pm$0.06$^\sim$ 
&0.00$\pm$0.00$^\dagger$ 
& \textbf{0.51$\pm$0.04}
&0.04$\pm$0.01$^\dagger$ 
&0.46$\pm$0.04$\sim$ \\
&\text{AMI}
&0.25$\pm$0.21$^\dagger$ 
&0.09$\pm$0.16$^\dagger$ 
&\textbf{0.46$\pm$0.02} 
&0.002$\pm$0.004$^\dagger$ 
& 0.46$\pm$0.01$^\sim$
&0.09$\pm$0.01$^\dagger$ 
&0.43$\pm$0.02$^\sim$ \\

\hline




\multirow{3}{6em}{Zoo ($k$ = 7)}
&\text{$k^\ast$}
&4.43$\pm$2.39
&4.37$\pm$1.87
&16.78$\pm$2.26 
&66.93$\pm$3.95 
&6.00$\pm$0.88
&2.00$\pm$0.00 
&6.57$\pm$0.85 \\
&\text{ARI}
&0.53$\pm$0.22$^\dagger$ 
&0.64$\pm$0.23$^\dagger$ 
&0.71$\pm$0.10$^\dagger$ 
&0.11$\pm$0.02$^\dagger$ 
&\textbf{0.89$\pm$0.09}
&0.03$\pm$0.04$^\dagger$ 
&0.85$\pm$0.02$^\sim$ \\
&\text{AMI}
&0.61$\pm$0.21$^\dagger$ 
&0.69$\pm$0.21$^\dagger$ 
&0.75$\pm$0.05$^\dagger$ 
&0.27$\pm$0.03$^\dagger$ 
&0.84$\pm$0.06$^\sim$
&0.05$\pm$0.02$^\dagger$ 
&\textbf{0.85$\pm$0.02} \\

\hline



\multirow{3}{6em}{Brain ($k$ = 5)}
& \text{$k^*$}
& 5$\pm$1.92
&5$\pm$1.49
& 18$\pm$0.00
&42$\pm$0.00
& 5$\pm$0.36
& 2$\pm$0.50
& 4$\pm$0.00\\
&\text{ARI}
& 0.26$\pm$0.10$^\dagger$
& 0.26$\pm$0.10$^\dagger$
& 0.64$\pm$0.02$^\sim$
& 0.00$\pm$0.00$^\dagger$
& 0.56$\pm$0.06$^\dagger$
& 0.016$\pm$0.02$^\dagger$
& \textbf{0.66$\pm$0.03}\\
& \text{AMI}
& 0.33$\pm$0.10$^\dagger$
& 0.27$\pm$0.11$^\dagger$
& 0.62$\pm$0.03$^\dagger$
& 0.00$\pm$0.00$^\dagger$
& 0.62$\pm$0.05$^\dagger$
& 0.03$\pm$0.04$^\dagger$
& \textbf{0.72$\pm$0.03}\\

\hline

\multirow{3}{6em}{Lung ($k$ = 5)}
&\text{$k^\ast$}
&4.21$\pm$1.63 
&4.41$\pm$1.64
&23.86$\pm$0.36 
&7.40$\pm$2.16 
&3.00$\pm$0.00
&1.04$\pm$0.20
&5.00$\pm$0.00 \\
&\text{ARI}
&0.39$\pm$0.20$^\dagger$ 
&\textbf{0.53$\pm$0.19} 
&0.35$\pm$0.004$^\dagger$ 
&0.31$\pm$0.13$^\dagger$ 
& 0.35$\pm$0.003$^\dagger$
&0.00$\pm$0.00$^\dagger$
&0.36$\pm$0.003$^\dagger$ \\
&\text{AMI}
&0.49$\pm$0.20$^\dagger$ 
&\textbf{0.55$\pm$0.17} 
&0.35$\pm$0.001$^\dagger$ 
&0.22$\pm$0.09$^\dagger$ 
&0.37$\pm$0.001$^\dagger$
&0.00$\pm$0.00$^\dagger$
&0.54$\pm$0.01$^\sim$ \\

\hline

\multirow{3}{6em}{Lymphoma(bio) ($k$ = 9)}
&\text{$k^\ast$}
&4.28$\pm$1.61 
&6.32$\pm$1.41
&96$\pm$0.00 
&96$\pm$0.00 
&6.57$\pm$0.51
&2$\pm$0.00
&6.93$\pm$0.82 \\
&\text{ARI}
&0.37$\pm$0.12$^\dagger$ 
&\textbf{0.49$\pm$0.09} 
&0.00$\pm$0.00$^\dagger$ 
&0.00$\pm$0.00$^\dagger$ 
&0.26$\pm$0.09$^\dagger$
& 0.06$\pm$0.06$^\dagger$
&0.41$\pm$0.02$^\dagger$ \\
&\text{AMI}
&0.47$\pm$0.08$^\dagger$ 
&\textbf{0.56$\pm$0.07} 
&0.00$\pm$0.00$^\dagger$ 
&0.00$\pm$0.00$^\dagger$ 
&0.45$\pm$0.08$^\dagger$
&0.09$\pm$0.06$^\dagger$
&0.51$\pm$0.01$^\sim$ \\

\hline

\multirow{3}{6em}{Coil20 ($k$ = 20)}
&\text{$k^\ast$}
&8.66$\pm$3.72 
&5.62$\pm$1.69
& 90.07$\pm$5.01
&1440$\pm$0.00 
& 20$\pm$0.00
& 1.31$\pm$0.47
&19$\pm$0.00 \\
&\text{ARI}
&0.30$\pm$0.11$^\dagger$ 
&0.21$\pm$0.05$^\dagger$ 
& 0.69$\pm$0.001$^\dagger$
&0.00$\pm$0.00$^\dagger$ 
& \textbf{0.82$\pm$0.00}
& 0.00$\pm$0.00$^\dagger$
&0.80$\pm$0.00$^\dagger$ \\
&\text{AMI}
&0.58$\pm$0.09$^\dagger$ 
&0.50$\pm$0.07$^\dagger$ 
&0.86$\pm$0.00$^\dagger$
&0.00$\pm$0.00$^\dagger$ 
&\textbf{0.93$\pm$0.00}
&0.00$\pm$0.00$^\dagger$
&0.92$\pm$0.00$^\dagger$ \\

\hline

\multirow{3}{6em}{Wdbc ($k$ = 2)}
&\text{$k^\ast$}
&1$\pm$0.00 
&1$\pm$0.00
&136.43$\pm$5.65 
&569$\pm$0.00 
&6.78$\pm$0.58
&2$\pm$0.00
&2.21$\pm$0.42 \\
&\text{ARI}
&0.00$\pm$0.00$^\dagger$ 
&0.00$\pm$0.00$^\dagger$ 
&\textbf{0.38$\pm$0.03} 
&0.00$\pm$0.00$^\dagger$ 
&0.09$\pm$0.01$^\dagger$
&0.001$\pm$0.002$^\dagger$
&0.17$\pm$0.31$^\dagger$ \\
&\text{AMI}
&0.00$\pm$0.00$^\dagger$ 
&0.00$\pm$0.00$^\dagger$ 
&\textbf{0.24$\pm$0.01} 
&0.00$\pm$0.00$^\dagger$ 
&0.08$\pm$0.01$^\dagger$
&0.01$\pm$0.003$^\dagger$
&0.14$\pm$0.26$^\dagger$ \\

\hline





\multirow{3}{7em}{Lymphoma (Microarray) ($k$ = 3)}
&\text{$k^\ast$}
&2.44$\pm$1.63 
&2.23$\pm$1.48
&3$\pm$0.00 
&3$\pm$0.85 
&2.43$\pm$0.51
&1.06$\pm$0.25
&3$\pm$0.00 \\
&\text{ARI}
&0.22$\pm$0.27$^\dagger$ 
&0.22$\pm$0.28$^\dagger$
&0.79$\pm$0.00$^\dagger$
&\textbf{0.86$\pm$0.03 }
&0.33$\pm$0.42$^\dagger$
&0.00$\pm$0.011$^\dagger$ 
&0.79$\pm$0.00$^\dagger$ \\
&\text{AMI}
&0.26$\pm$0.29$^\dagger$
&0.27$\pm$0.31$^\dagger$
&0.71$\pm$0.00$^\dagger$
&\textbf{0.81$\pm$0.03} 
&0.29$\pm$0.37$^\dagger$
&0.00$\pm$0.001$^\dagger$
&0.71$\pm$0.00$^\dagger$ \\

\hline
\end{tabular}
\caption{Results for Real Life Datasets }
\label{Real_data_full}
{\raggedright \scriptsize $^\dagger$ : significantly different from the best performing algorithm ,
$^\sim$ : statistically same as the best performing algorithm .}
\end{table*}

\subsection{Case study on Wisconsin Dataset \label{app:case_study_wisconsin}}

We evaluate our algorithm's performance on the Wisconsin dataset, available, and compare it with other algorithms. The dataset contains 699 cases from a study on breast cancer patients, with 9 categorical features and two classes: Benign and Malignant. After refining the dataset to exclude missing values, we work with 683 instances.

\begin{table*}[!htb]
    \centering
    \small
    \begin{tabular}{>{\centering\arraybackslash}p{0.7cm}>{\centering\arraybackslash}p{0.7cm}>{\centering\arraybackslash}p{1.5cm}>{\centering\arraybackslash}p{1.5cm}>{\centering\arraybackslash}p{1.5cm}>{\centering\arraybackslash}p{1.5cm}>{\centering\arraybackslash}p{1.5cm}>{\centering\arraybackslash}p{1.5cm}>{\centering\arraybackslash}p{1.5cm}}
    \hline
    \textbf{Index} & \textbf{Noise} & \textbf{KM} & \textbf{MKM} & \textbf{CC} & \textbf{RCC} & \textbf{RConv} & \textbf{RBKM} & \textbf{COMET}\\

    \hline
    \multirow{5}{7em}{ARI}
    & 0& 0.52$\pm$0.36$^\dagger$& 0.55$\pm$0.39$^\dagger$& 0.79$\pm0.00$$^\dagger$& 0.01$\pm$0.00$^\dagger$& 0.85$\pm$0.00$^\dagger$& 0.002$\pm$0.001$^\dagger$& \textbf{0.88$\pm$0.00}\\
    & 5& 0.58$\pm$0.36$^\dagger$& 0.65$\pm$0.33$^\dagger$& 0.82$\pm$0.01$^\dagger$& 0.01$\pm$0.01$^\dagger$& 0.84$\pm$0.00$^\dagger$& 0.08$\pm$0.03$^\dagger$& \textbf{0.87$\pm$0.00}\\
    & 10& 0.52$\pm$0.35$^\dagger$& 0.47$\pm$0.39$^\dagger$& 0.81$\pm$0.01$^\dagger$& 0.01$\pm$0.00$^\dagger$& 0.85$\pm$0.03$^\dagger$& 0.15$\pm$0.06$^\dagger$& \textbf{0.87$\pm$0.01}\\
    & 15& 0.66$\pm$0.29$^\dagger$& 0.54$\pm$0.38$^\dagger$& 0.80$\pm$0.02$^\dagger$& 0.01$\pm$0.01$^\dagger$& 0.84$\pm$0.02$^\dagger$& 0.24$\pm$0.05$^\dagger$& \textbf{0.86$\pm$0.01}\\
    & 20& 0.40$\pm$0.39$^\dagger$& 0.53$\pm$0.38$^\dagger$& 0.82$\pm$0.03$^\dagger$& 0.01$\pm$0.01$^\dagger$& 0.86$\pm$0.03$^\dagger$& 0.27$\pm$0.11$^\dagger$& \textbf{0.88$\pm$0.02}\\
    \hline
    \multirow{5}{7em}{AMI}
    & 0& 0.47$\pm$0.32$^\dagger$& 0.47$\pm$0.34$^\dagger$& 0.66$\pm$0.00$^\dagger$& 0.08$\pm$0.00$^\dagger$& 0.73$\pm$0.00$^\dagger$& 0.002$\pm$0.001$^\dagger$& \textbf{0.80$\pm$0.00}\\
    & 5& 0.52$\pm$0.32$^\dagger$& 0.57$\pm$0.29$^\dagger$& 0.69$\pm$0.01$^\dagger$& 0.06$\pm$0.001$^\dagger$& 0.73$\pm$0.01$^\dagger$& 0.12$\pm$0.02$^\dagger$& \textbf{0.76$\pm$0.02}\\
    & 10& 0.48$\pm$0.31$^\dagger$& 0.41$\pm$0.34$^\dagger$& 0.67$\pm$0.01$^\dagger$& 0.07$\pm$0.003$^\dagger$& 0.75$\pm$0.03$^\dagger$& 0.19$\pm$0.05$^\dagger$& \textbf{0.76$\pm$0.01}\\
    & 15& 0.60$\pm$0.25$^\dagger$& 0.47$\pm$0.33$^\dagger$& 0.67$\pm$0.02$^\dagger$& 0.08$\pm$0.005$^\dagger$& 0.72$\pm$0.02$^\sim$& 0.27$\pm$0.03$^\dagger$& \textbf{0.75$\pm$0.03}\\
    & 20& 0.37$\pm$0.35$^\dagger$& 0.46$\pm$0.33$^\dagger$& 0.68$\pm$0.02$^\dagger$& 0.07$\pm$0.01$^\dagger$& 0.76$\pm$0.03$^\sim$& 0.31$\pm$0.06$^\dagger$& \textbf{0.77$\pm$0.02}\\
    \hline
    \multirow{5}{7em}{$k^\ast$}
    &  0& 2.09$\pm$0.85& 1.92$\pm$0.83& 15$\pm$0.00& 462$\pm$0.00& 3$\pm$0.00& 2$\pm$0.00& \textbf{3$\pm$0.00}\\
    &  5& 2.13$\pm$0.45&  2.2$\pm$0.81& 15$\pm$1.86& 508$\pm$7.12& 3$\pm$0.56& 2$\pm$0.00& \textbf{2$\pm$0.00}\\
    & 10& 2.25$\pm$0.63&    2.00$\pm$0.82& 15$\pm$1.86& 477$\pm$9.06& 2$\pm$1.04& 2$\pm$0.00& \textbf{3$\pm$0.00}\\
    & 15& 2.47$\pm$0.62& 2.04$\pm$0.92& 15$\pm$1.42& 457$\pm$$\pm$8.57& 3$\pm$0.97& 2$\pm$0.00& \textbf{3$\pm$0.00}\\
    & 20& 1.93$\pm$0.69& 2.12$\pm$0.83& 15$\pm$1.42& 481$\pm$9.21& 3$\pm$0.97& 2$\pm$0.00& \textbf{2$\pm$0.00}\\
    \hline
    \end{tabular}
    \caption{Performance of Different Algorithms on \textbf{Wisconsin} on different Noise levels}
    \label{tab:table_wisconsin}
    {\raggedright \scriptsize $^\dagger$ : significantly different from the best performing algorithm,
$^\sim$ : statistically same as the best performing algorithm .}
\end{table*}

\begin{figure}[!htb]
    \centering
    \includegraphics[width = 0.6\linewidth]{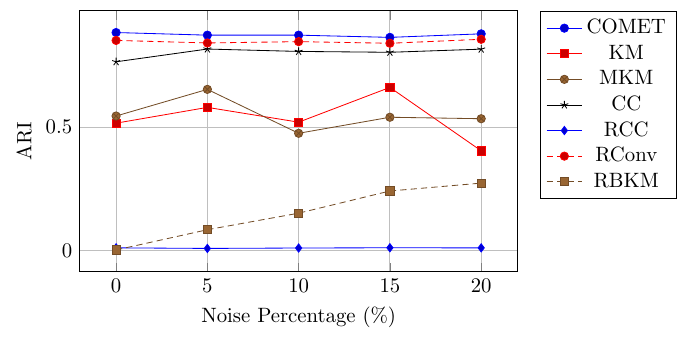}
    \caption{Line plot for performance of different algorithms on \textbf{Wisconsin}}
    \label{fig:line_plot_wisconsin}
\end{figure}

The case study focuses on how our algorithm performs in the presence of noise, which is common in real-life clustering tasks. We introduce varying levels of uniform noise and test the algorithms at 0$\%$, 5$\%$, 10$\%$, 15$\%$, and 20$\%$ noise. For each noise level, we add $\lfloor np \rfloor$ new "noise" datapoints uniformly within the minimum hypercube containing the original data. The results, shown in Table \ref{tab:table_wisconsin} and Figure \ref{fig:line_plot_wisconsin}, reveal that COMET outperforms all algorithms, achieving the highest ARI and AMI across noise levels. Here the standard deviation is that of the performance measure, not of the mean statistic. Convex clustering and Robust Convex clustering also perform well, but slightly worse than COMET. $k-$means and MoM $k-$means are similar but much less effective, while RBKM and RCC perform poorly, with RBKM improving as noise increases. The \textit{t}-SNE plots showing the clustering results of various algorithms on this datsets are provided in the appendix (see \ref{app:TSNE_plots}).

\subsection{\textit{t}-SNE Plots for Wisconsin and Brain dataset \label{app:TSNE_plots}}
\textit{t}-SNE plots for the clustering results on Wisconsin and Brain dataset using various clustering algorithms are given in Figure \ref{fig:TSNE Wisconsin}, \ref{fig:TSNE Brain}. The \textit{t}-SNE plots help in understanding the clustering pattern picked by various algorithms better for high dimensional datasets.

\begin{figure*}[!htb]
    \centering
    
    \subfloat[Dataset]{\includegraphics[width=0.24\textwidth]{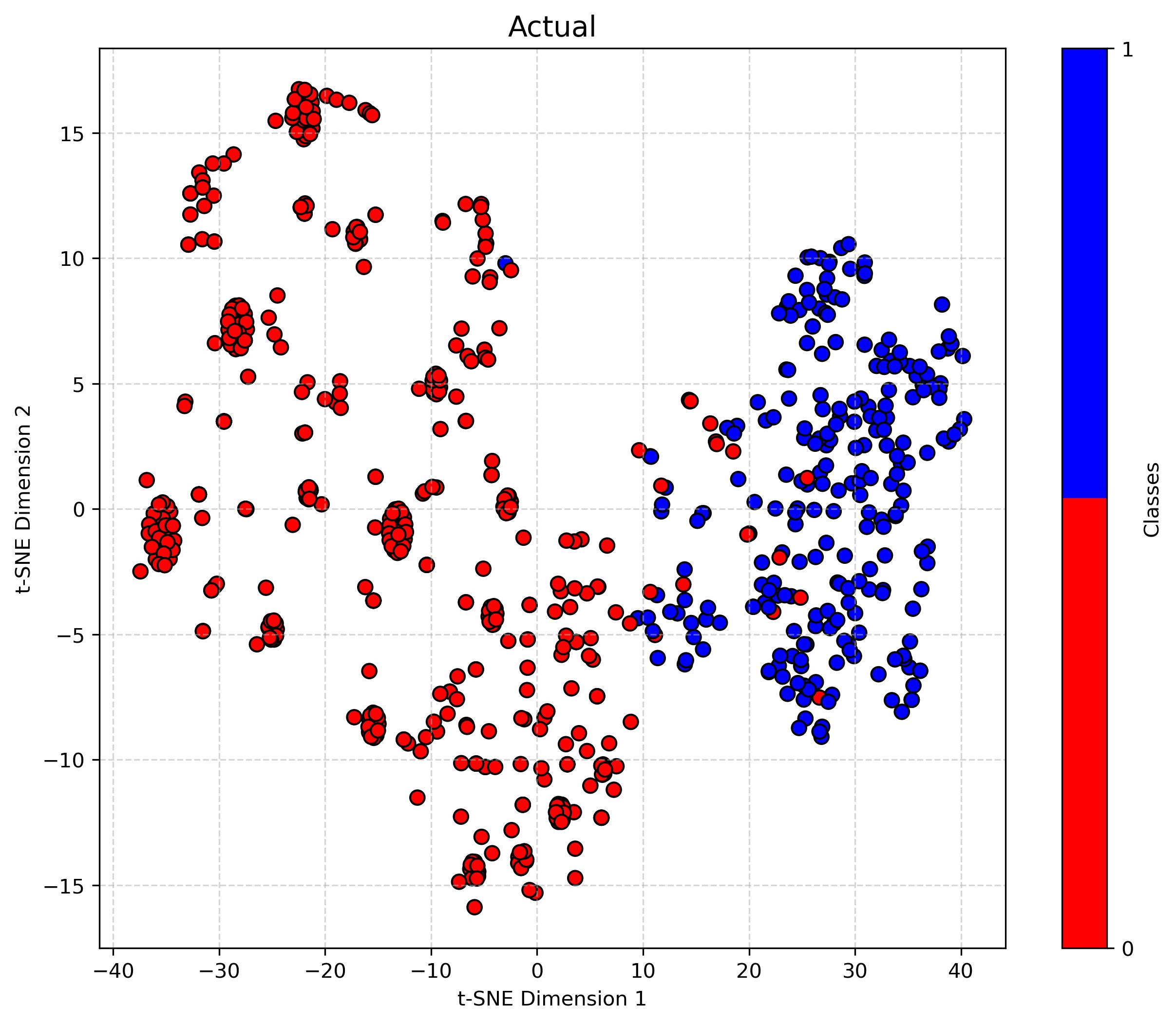}} 
    \subfloat[KM]{\includegraphics[width=0.24\textwidth]{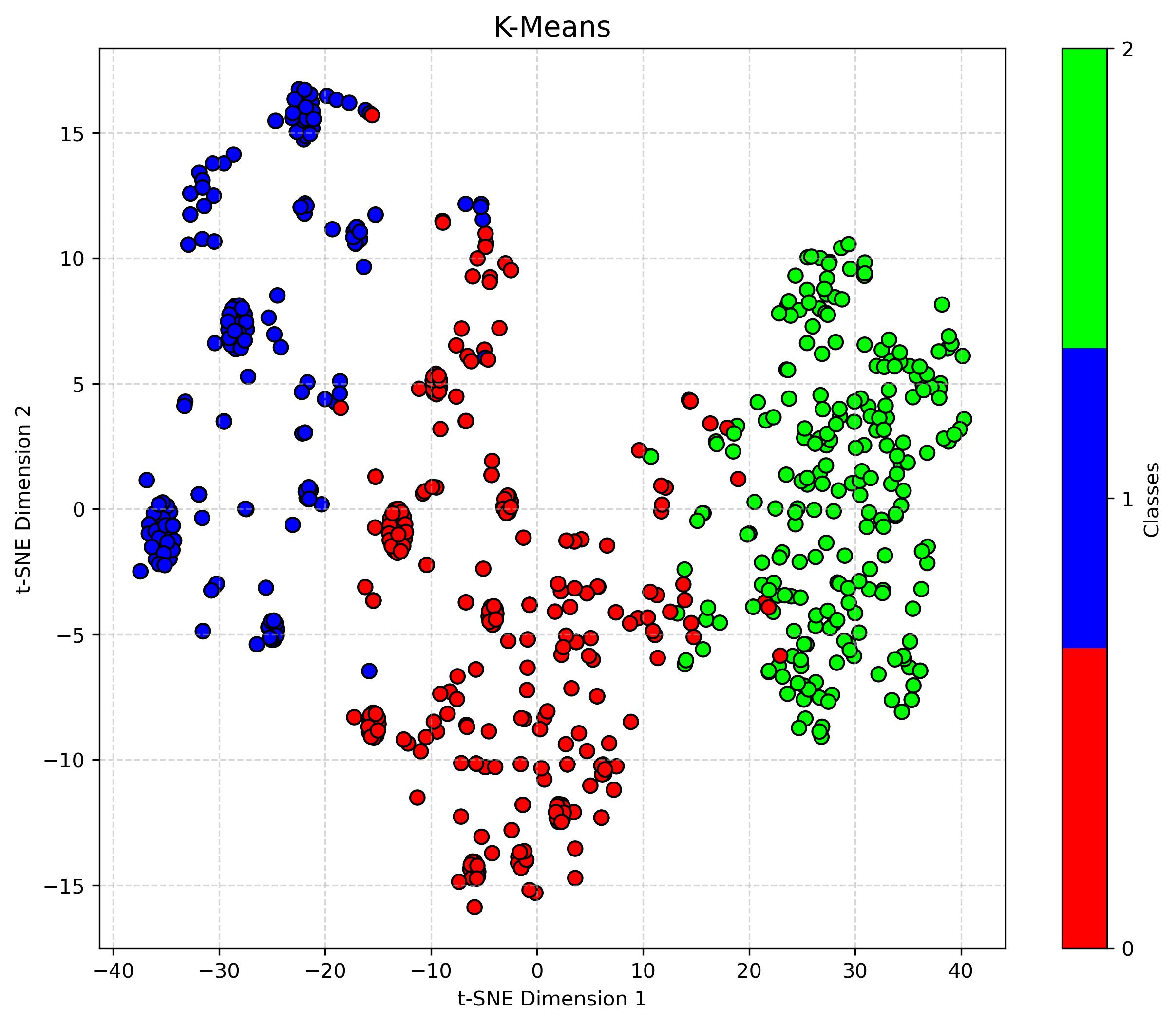}} 
    \subfloat[MKM]{\includegraphics[width=0.24\textwidth]{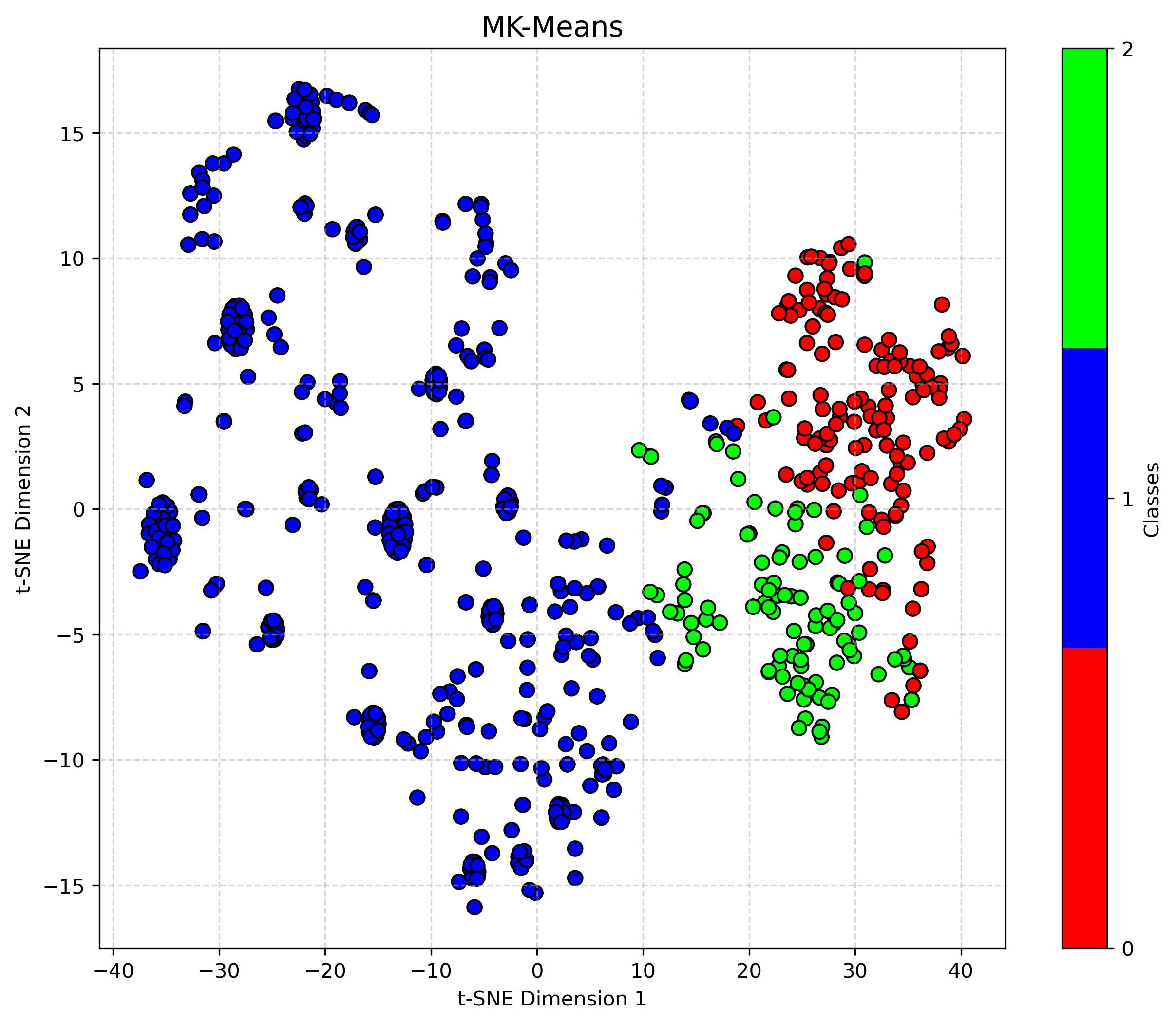}} 
    \subfloat[CC]{\includegraphics[width=0.24\textwidth]{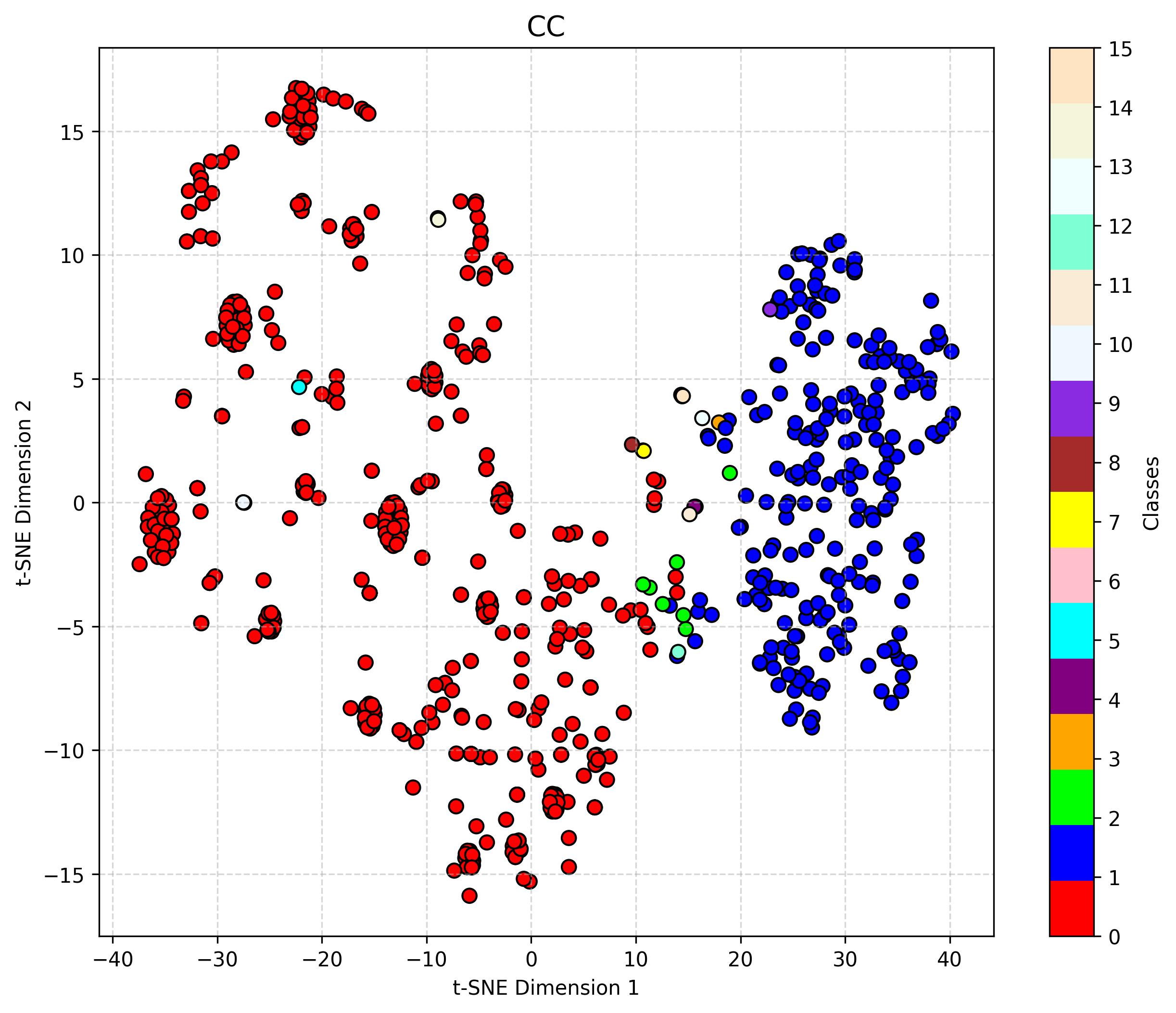}} \\[2ex]
    \subfloat[RCC]{\includegraphics[width=0.24\textwidth]{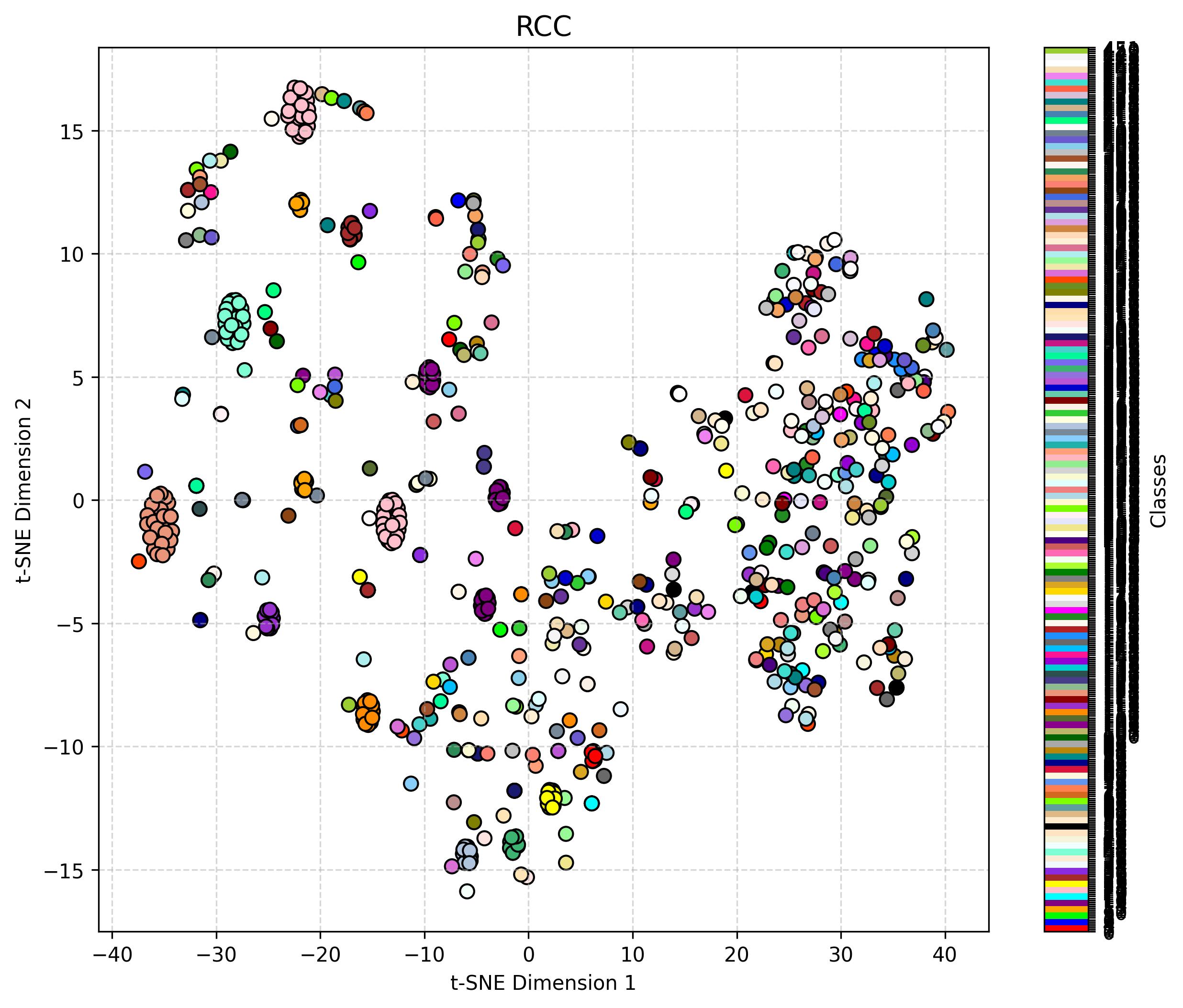}} 
    \subfloat[RConv]{\includegraphics[width=0.24\textwidth]{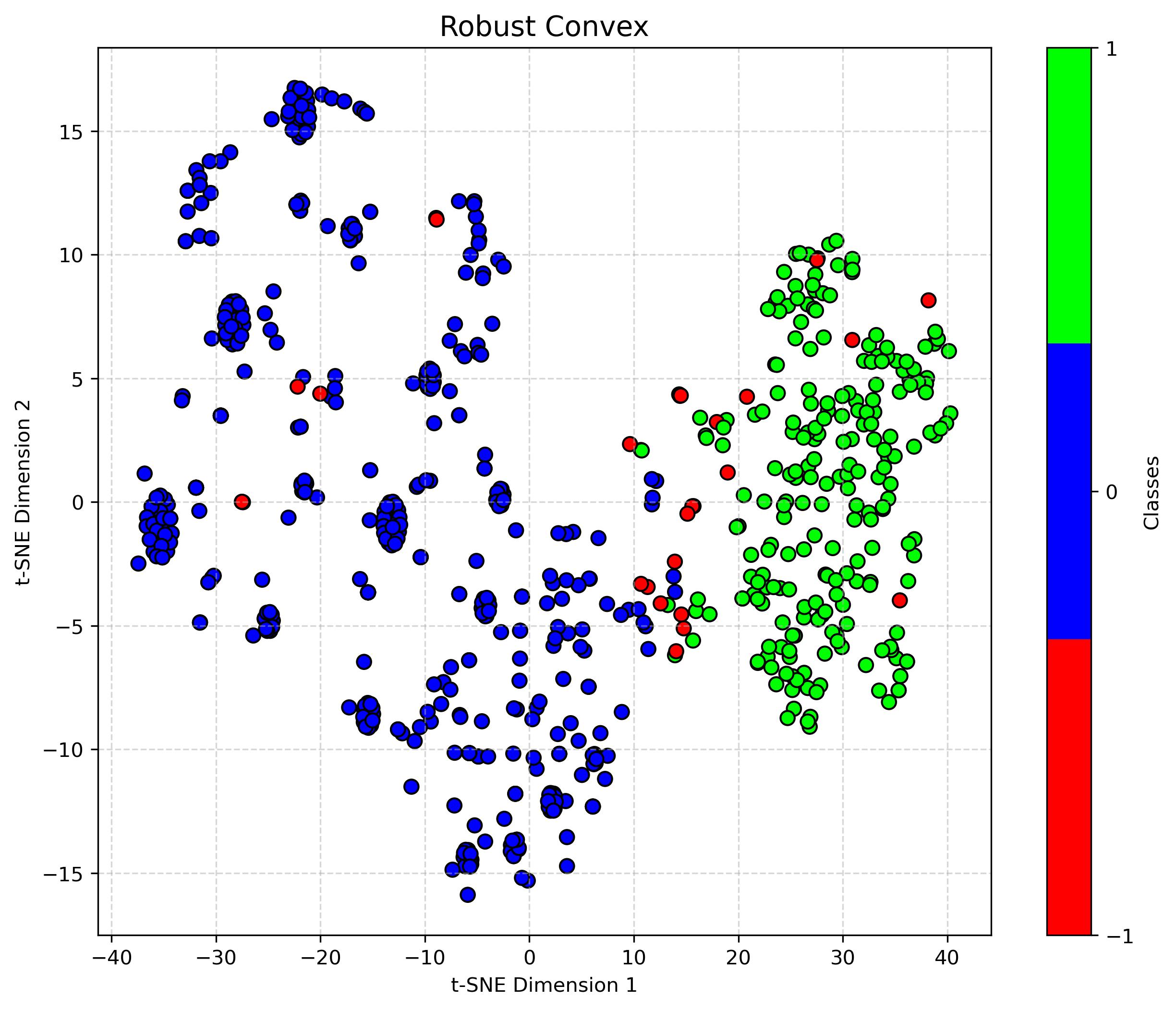}}
    \subfloat[RBKM]{\includegraphics[width=0.24\textwidth]{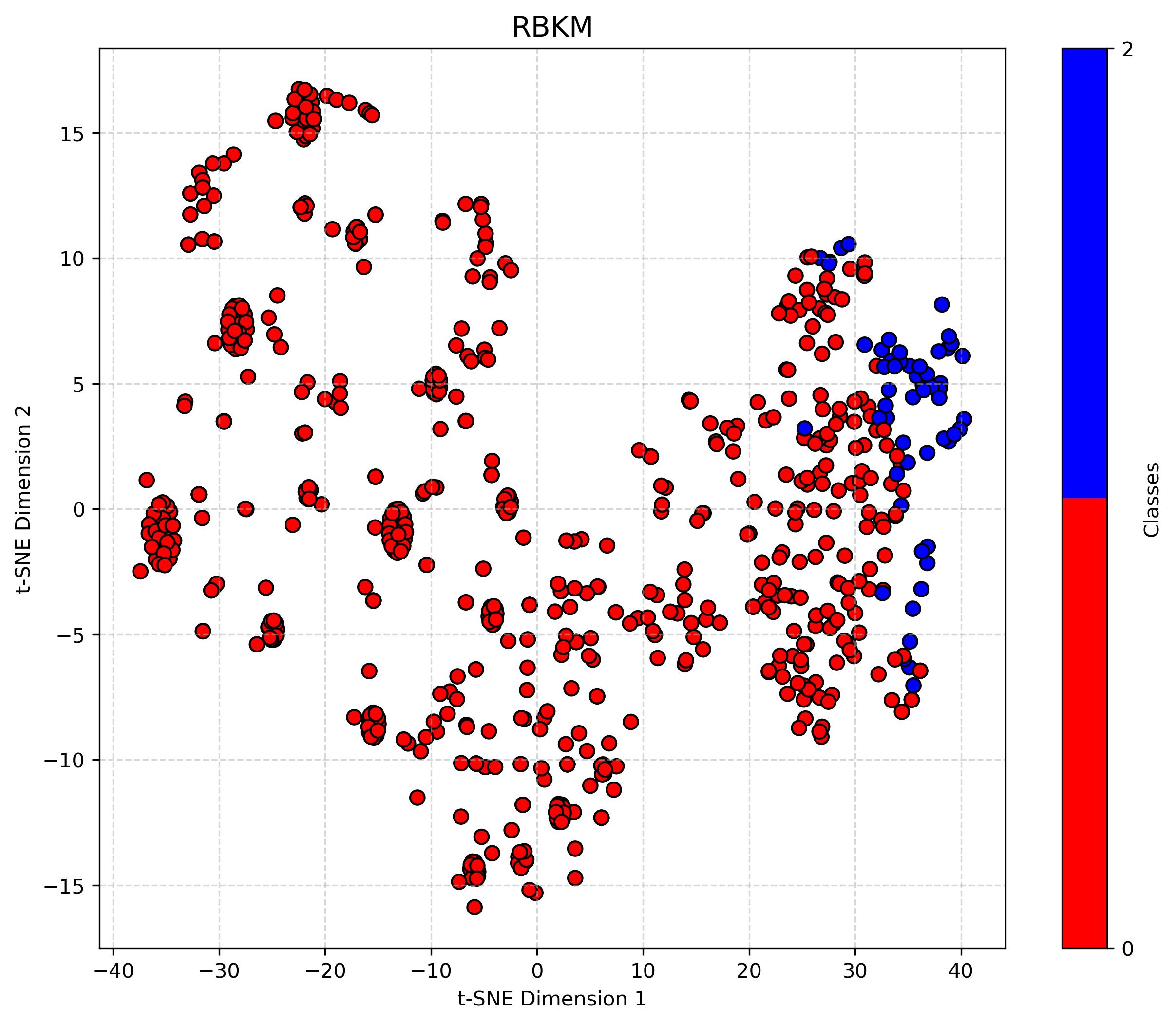}}
    \subfloat[COMET]{\includegraphics[width=0.24\textwidth]{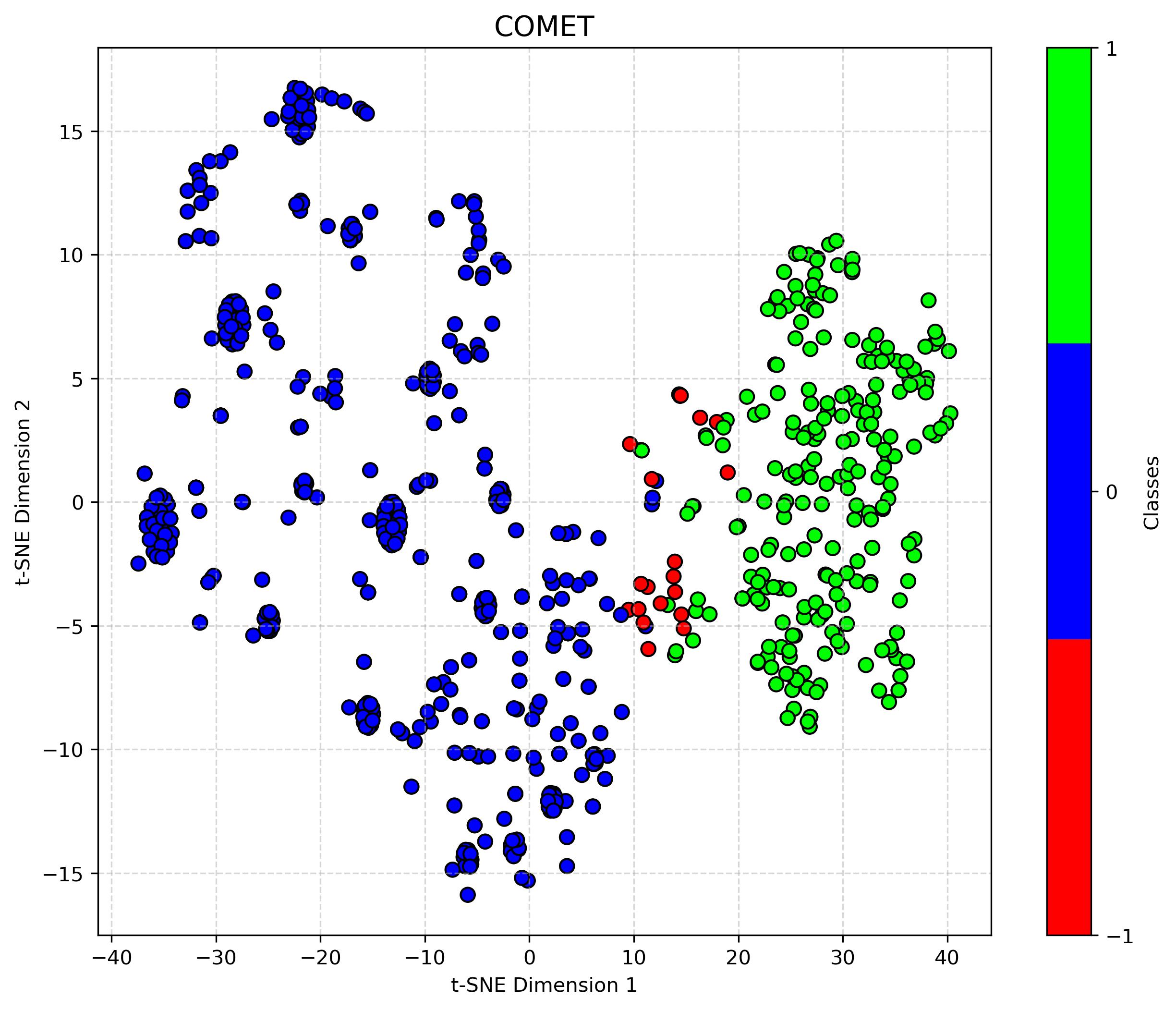}}
    
    \caption{\textit{t}-SNE plot of the Wisconsin dataset after clustering under various algorithms at 10$\%$ noise.}
    \label{fig:TSNE Wisconsin}
\end{figure*}

\begin{figure*}[!htb]
    \centering
    
    \subfloat[Dataset]{\includegraphics[width=0.24\textwidth]{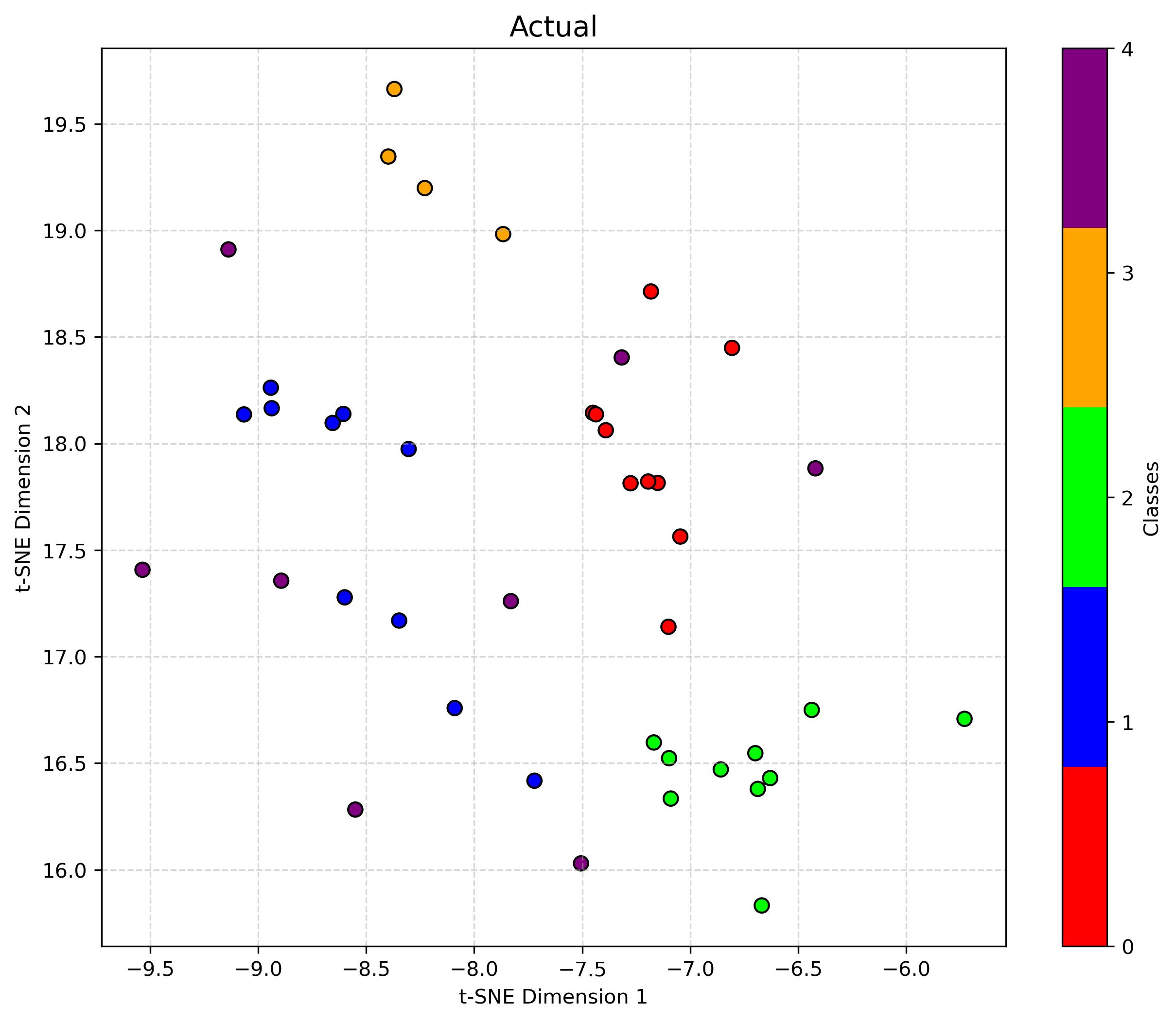}}
    \subfloat[KM]{\includegraphics[width=0.24\textwidth]{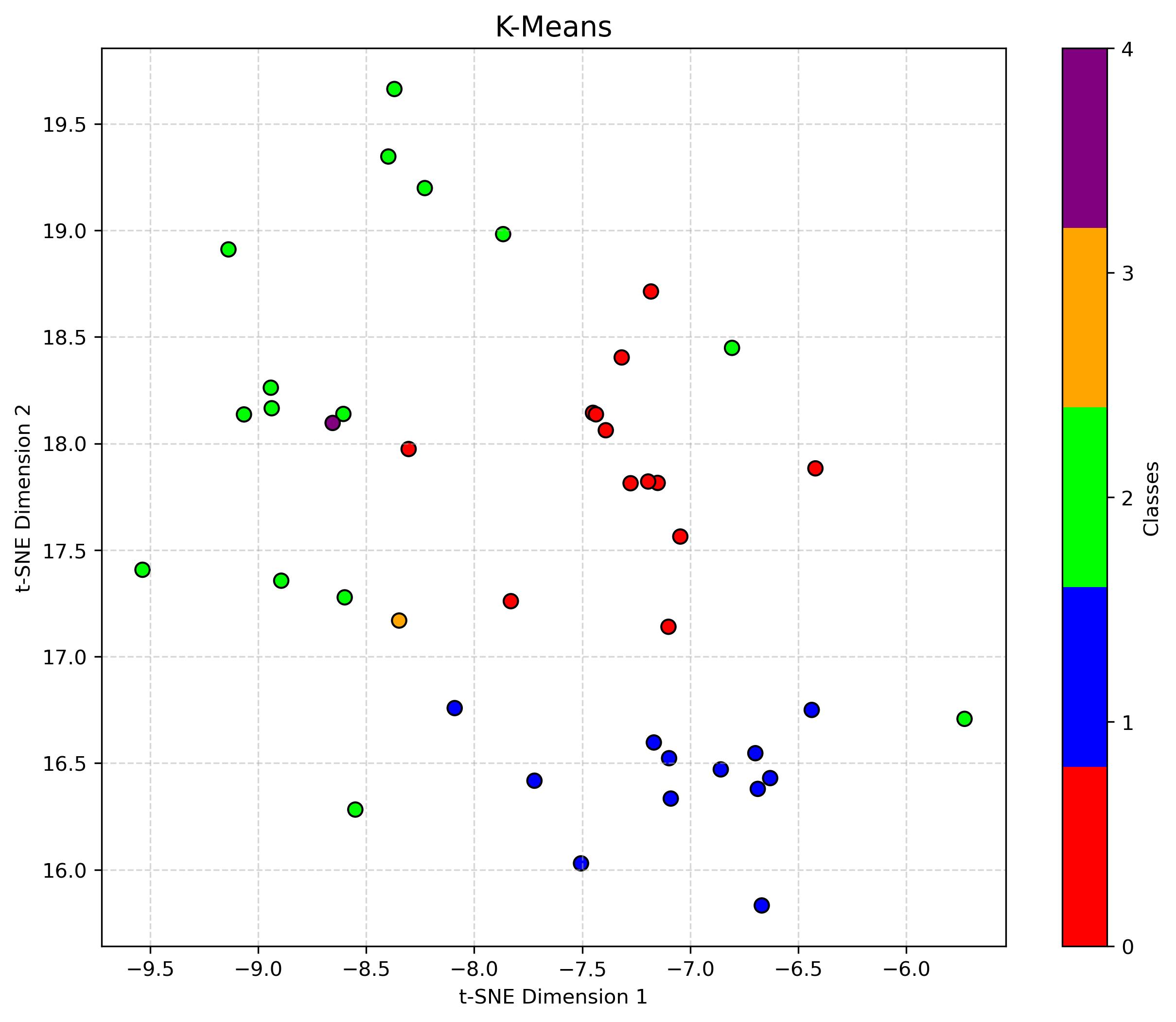}} 
    \subfloat[MKM]{\includegraphics[width=0.24\textwidth]{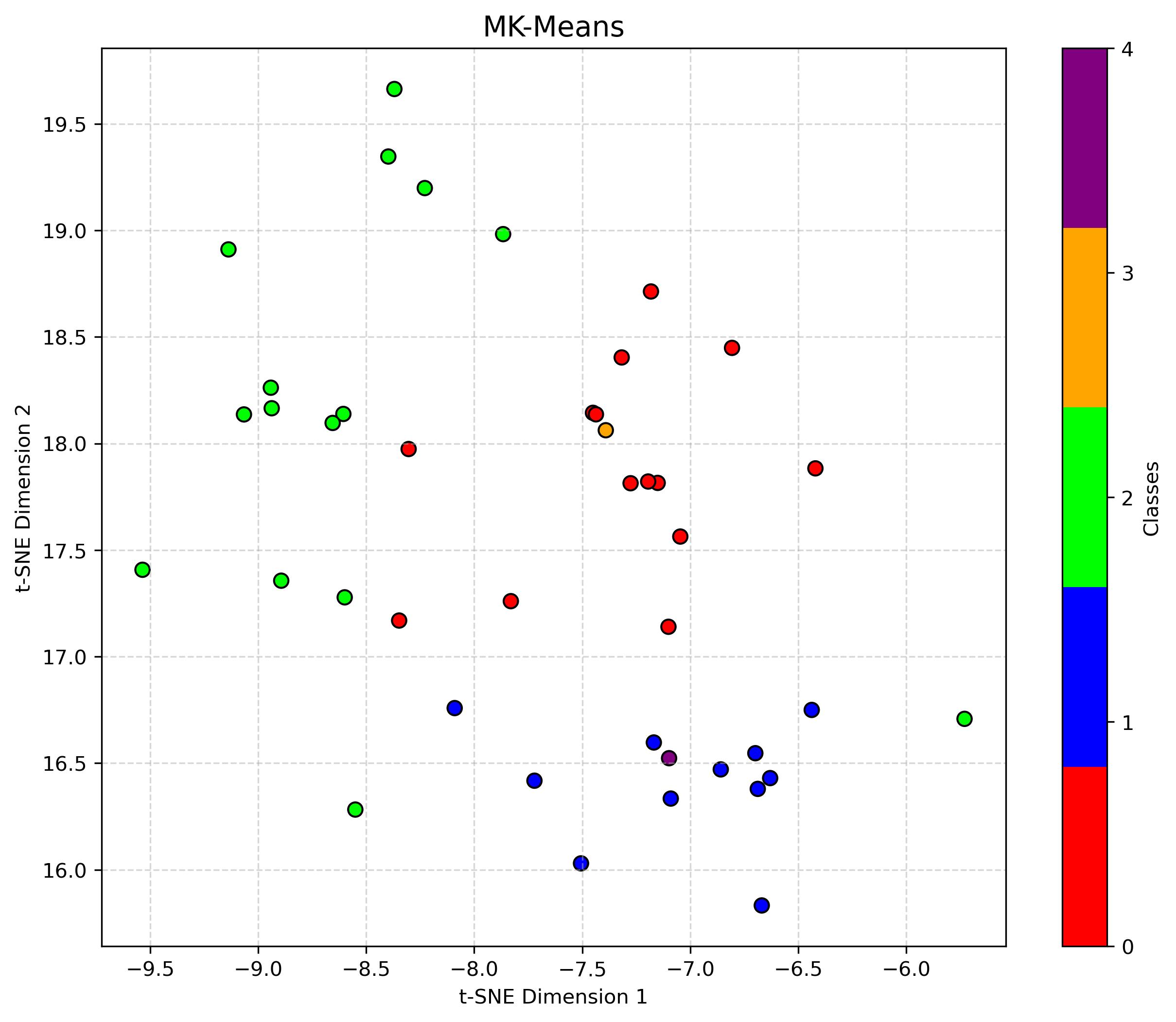}} 
    \subfloat[CC]{\includegraphics[width=0.24\textwidth]{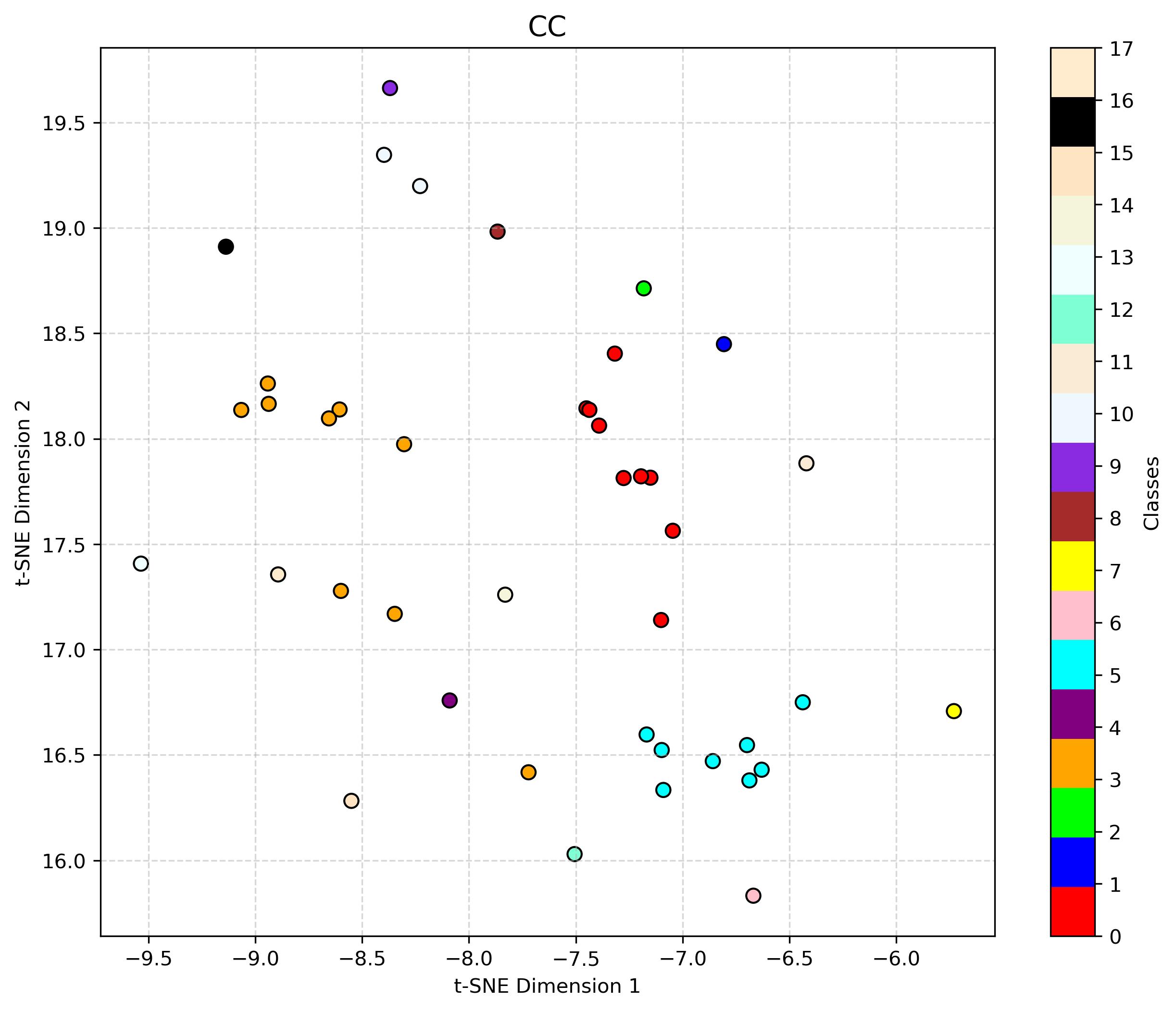}} \\[2ex]
    \subfloat[RCC]{\includegraphics[width=0.24\textwidth]{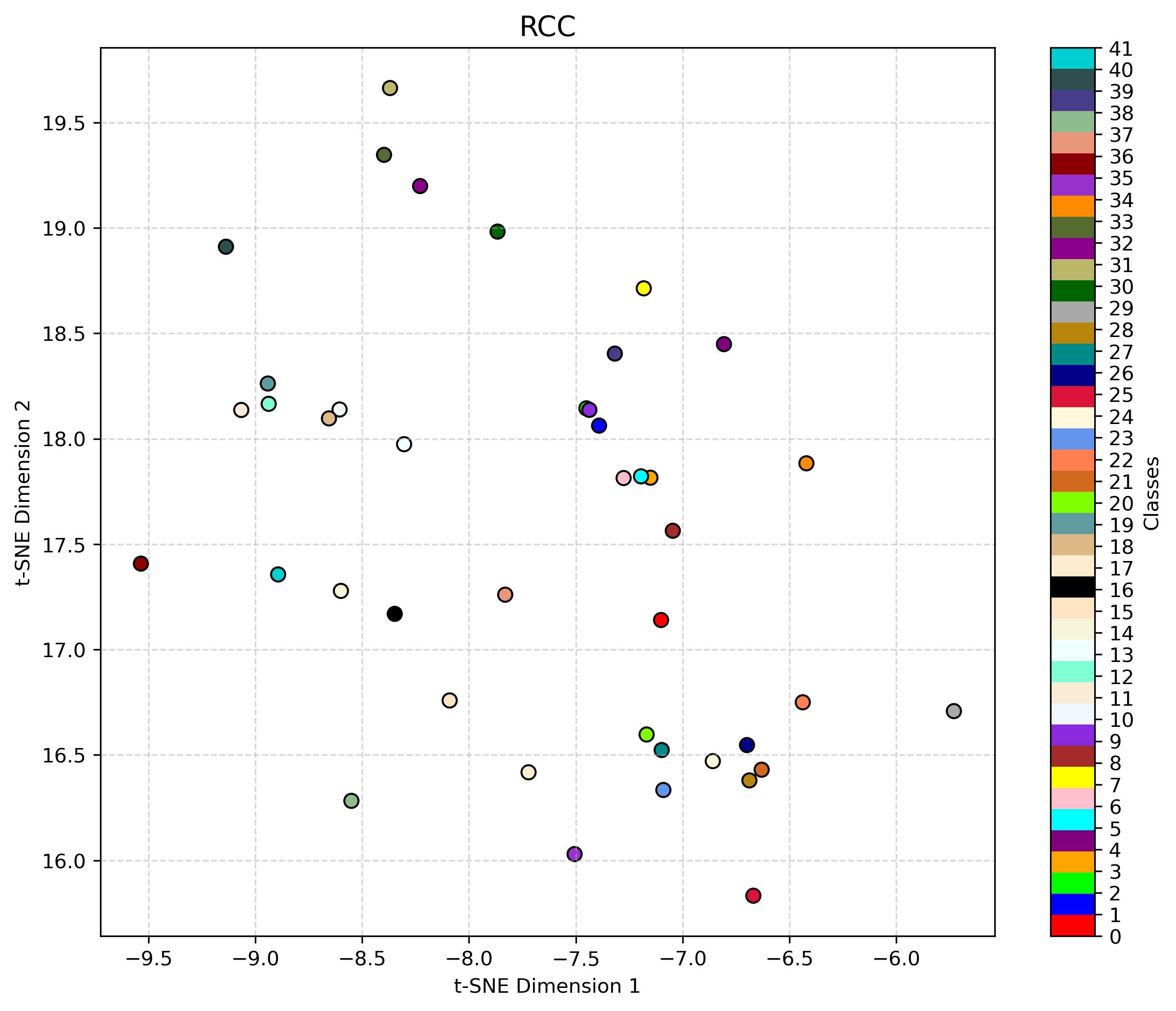}} 
    \subfloat[RConv]{\includegraphics[width=0.24\textwidth]{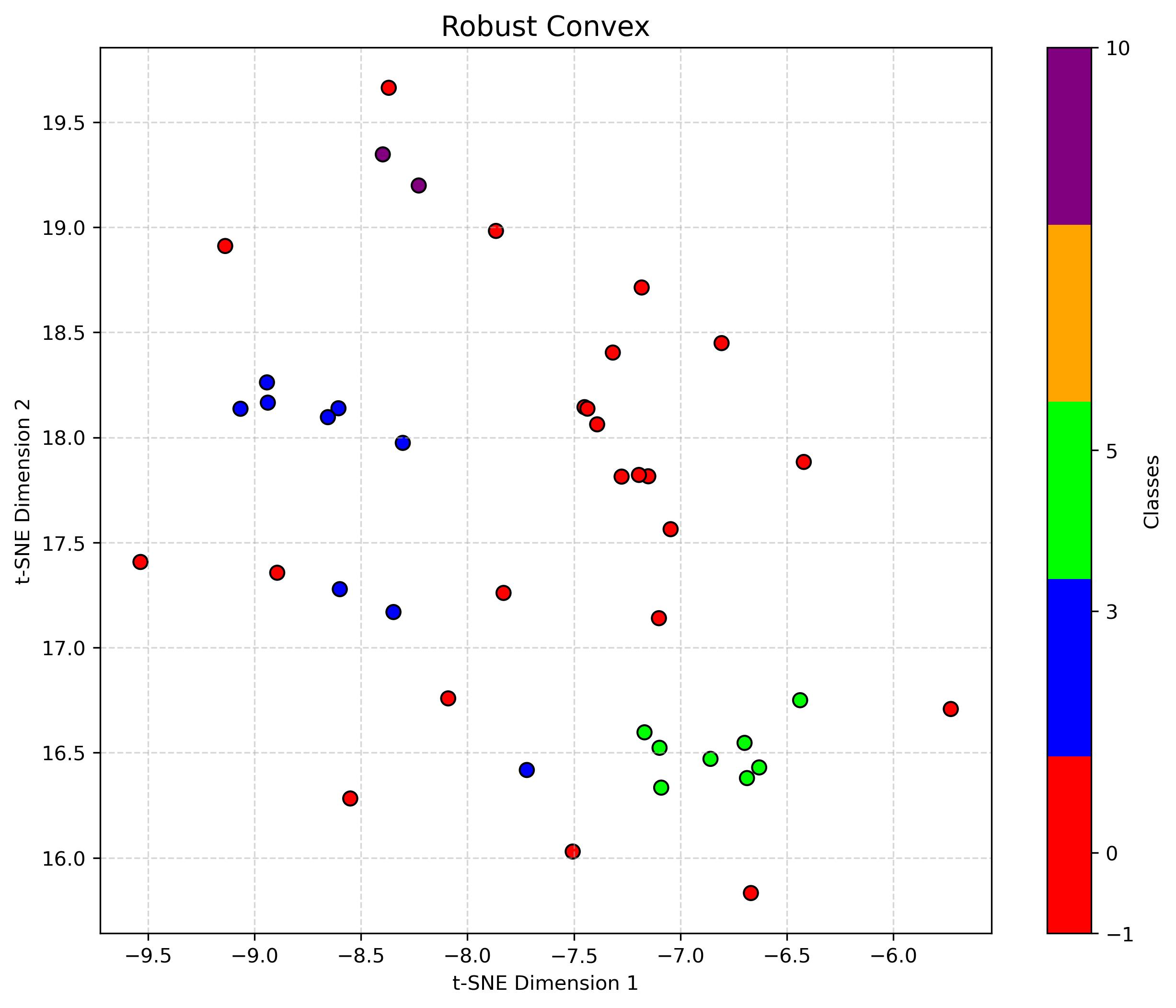}} 
    \subfloat[RBKM]{\includegraphics[width=0.24\textwidth]{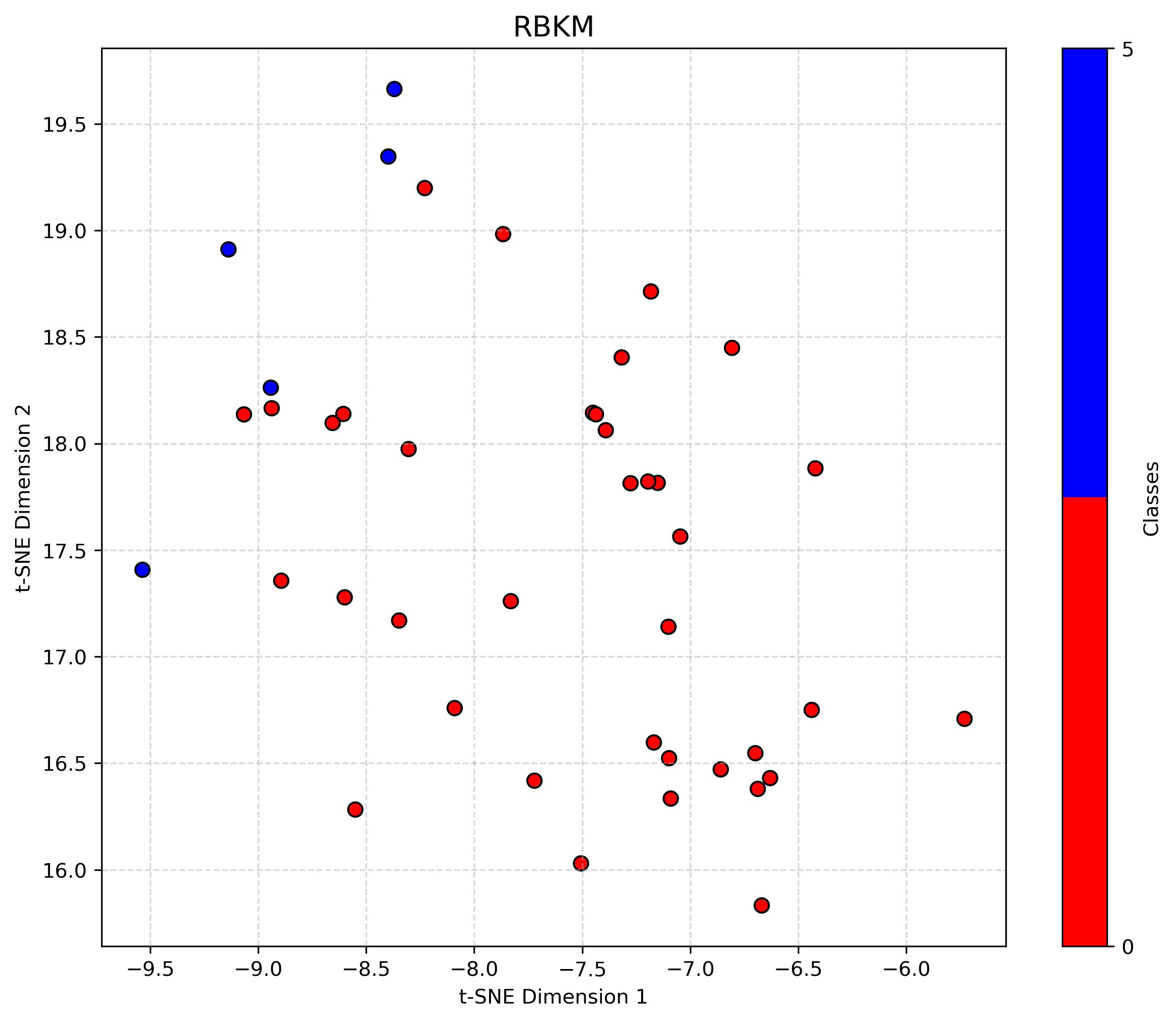}}
    \subfloat[COMET]{\includegraphics[width=0.24\textwidth]{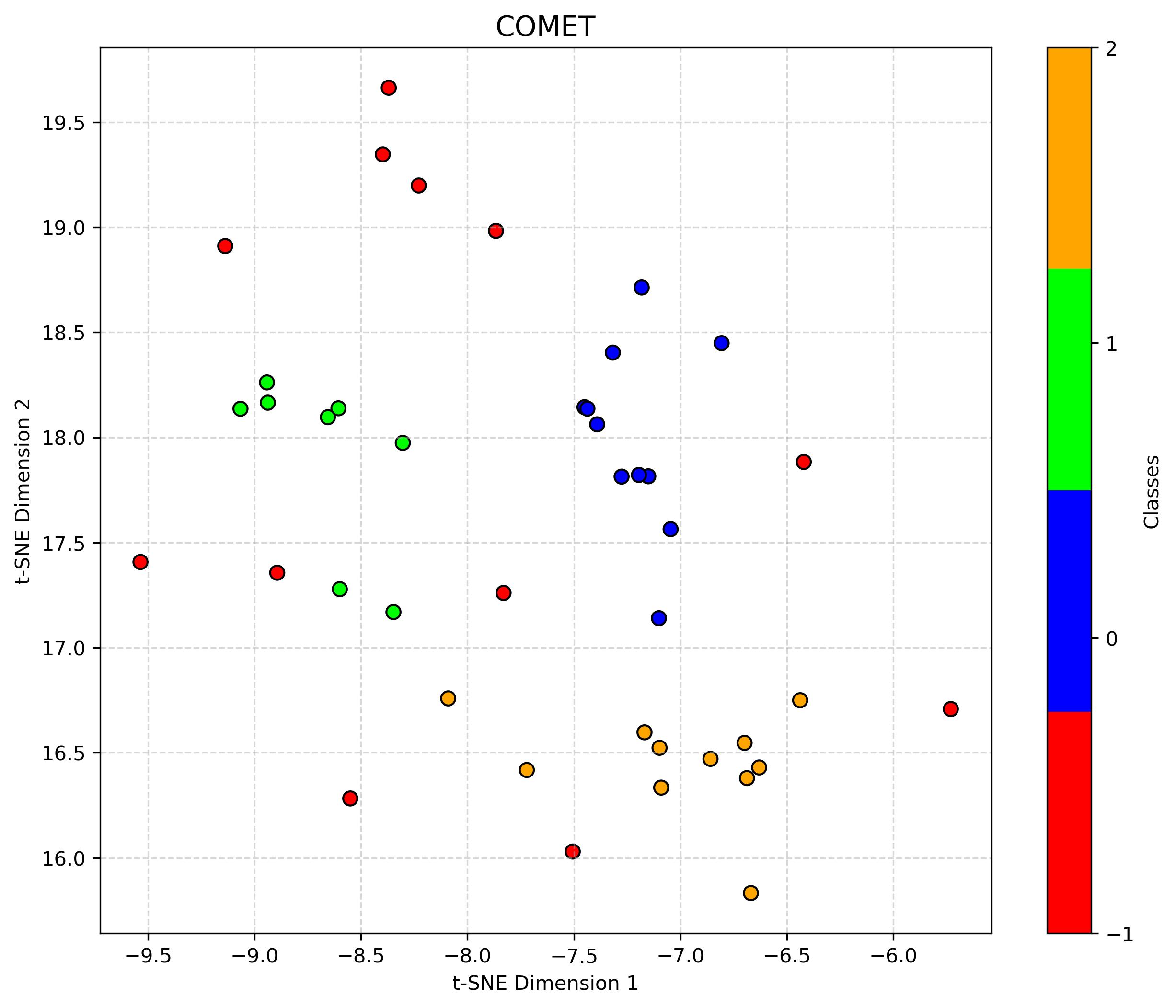}}
    
    \caption{\textit{t}-SNE plot of the Brain dataset after clustering under various algorithms.}
    \label{fig:TSNE Brain}
\end{figure*}

\subsection{Wilcoxon-Rank Sum Test \label{app:p-value_full_table}}
To assess whether the ARI and AMI scores produced by our algorithm are \emph{significantly higher} than those of selected baseline clustering methods, we employ the Wilcoxon Rank-Sum test. The ARI and AMI scores are computed for various algorithms on different $10\%$ contaminated datasets, and the corresponding $p$-values of the Wilcoxon Rank-Sum test are estimated using \emph{Monte Carlo simulation}. The results are presented in Table~\ref{ari_table} and \ref{ami_table}. For any value in the table with $p \leq 0.05$, we consider the difference to be \emph{statistically significant}, indicating that our algorithm produces higher ARI and AMI scores under the tested conditions.

Based on the values in the Table \ref{ari_table} and \ref{ami_table}, we observe that the performance measures for our algorithm are significantly higher than those for other algorithms. Our closest competitor is RConv, which performs better than us on 4 datasets out of the 15 we have tested on.

\begin{table*}[!htb]
\centering
\small

\begin{tabular}{lcccccc}
\hline
\scriptsize
\textbf{Dataset} & \textbf{CC} & \textbf{RCC} & \textbf{RConv} & \textbf{KM} & \textbf{MKM} & \textbf{RBKM} \\ \hline

Iris & 0.9409 & 2.32E-06 & 0.000385 & 0.97549 & 0.99922 & 0.998 \\  
Newthyroid & 3.35E-06 & 1.96E-06 & 1.13E-06 & 1.61E-09 & 1.42E-09 & 1.48E-09 \\  
Ecoli & 0.7399 & 1.37E-06 & 0.9967 & 2.52E-09 & 2.52E-09 & 1.31E-09 \\  
Wisconsin & 3.35E-06 & 3.72E-06 & 1.78E-06 & 2.79E-05 & 1.41E-07 & 1.27E-09 \\  
Wine & 0.0001 & 4.49E-07 & 1.64E-05 & 0.18801 & 0.001763 & 1.63E-09 \\  
Zoo & 3.88E-05 & 2.18E-06 & 0.9984 & 5.96E-09 & 8.90E-07 & 1.49E-09 \\  
Dermatology & 3.35E-06 & 3.62E-07 & 2.6E-05 & 1.05E-05 & 2.07E-07 & 1.61E-09 \\  
Brain & 0.9993 & 2.41E-07 & 4.96E-05 & 1.98E09 & 1.72E-09 & 1.55E-10 \\  
Lung & 3.35E-06 & 0.0159 & 5.28E-07 & 0.913832 & 0.9999 & 5.29E-17 \\  
Lymphoma & 3.35E-06 & 4.44E-07 & 1.31E-06 & 0.03995 & 0.9997 & 1.56E-09 \\  
Coil-20 & 3.55E-06 & 1.17E-07 & 0.999 & 1.74E-09 & 1.73E-09 & 5.35E-15 \\  
WDBC & 0.9949 & 4.46E-07 & 0.9895 & 9.09E-21 & 9.09E-21 & 0.00126 \\  
Lung-discrete & 3.35E-06 & 3.50E-06 & 9.7E-07 & 2.04E-09 & 6.03E-09 & 1.27E-09 \\  
ORLRaw10P & 3.35E-06 & 1.17E-07 & 3.55E-07 & 1.65E-09 & 1.63E-09 & 1.21E-09 \\  
Lymphoma (micro) & 0.5 & 0.9999 & 0.000241 & 1.67E-08 & 9.58E-08 & 5.7E-16 \\ \hline

\end{tabular}
\caption{COMET vs Other Algorithms (ARI values only)}
\label{ari_table}
\end{table*}

\begin{table*}[htb!]
\centering
\small

\begin{tabular}{lcccccc}
\hline
\scriptsize
\textbf{Dataset} & \textbf{CC} & \textbf{RCC} & \textbf{RConv} & \textbf{KM} & \textbf{MKM} & \textbf{RBKM} \\ \hline

Iris & 0.9999 & 2.32E-06 & 0.00029 & 0.58855 & 0.9992 & 0.9669 \\  
Newthyroid & 3.35E-06 & 1.96E-06 & 7.47E-05 & 1.61E-09 & 1.42E-09 & 1.5E-09 \\  
Ecoli & 0.9999 & 1.37E-06 & 0.9997 & 2.52E-09 & 2.52E-09 & 1.31E-09 \\  
Wisconsin & 3.35E-06 & 3.72E-06 & 1.78E-06 & 0.004352 & 9.95E-07 & 1.27E-09 \\  
Wine & 6.85E-06 & 4.49E-07 & 1.64E-05 & 0.17594 & 0.000765 & 1.63E-09 \\  
Zoo & 5.15E-06 & 2.18E-06 & 0.4583 & 3.24E-09 & 1.44E-07 & 1.49E-09 \\  
Dermatology & 3.35E-06 & 3.62E-07 & 2.6E-05 & 0.000254 & 1.65E-06 & 1.61E-09 \\  
Brain & 3.35E-06 & 2.41E-07 & 4.96E-05 & 1.73E-09 & 1.72E-09 & 1.55E-10 \\  
Lung & 3.35E-06 & 1.87E-06 & 5.28E-07 & 0.83273 & 0.9847 & 5.3E-17 \\  
Lymphoma & 3.35E-06 & 4.44E-07 & 0.0921 & 0.0508 & 0.9996 & 1.6E-09 \\  
Coil-20 & 3.54E-06 & 1.17E-07 & 0.999 & 1.74E-09 & 1.73E-09 & 5.6E-15 \\  
WDBC & 0.9949 & 4.46E-07 & 0.994 & 1.79E-11 & 9.09E-21 & 0.8603 \\  
Lung-discrete & 3.35E-06 & 3.50E-06 & 9.7E-07 & 2.34E-09 & 2.51E-07 & 1.27E-09 \\  
ORLRaw10P & 3.35E-06 & 1.17E-07 & 3.55E-07 & 1.65E-09 & 1.63E-09 & 1.21E-09 \\  
Lymphoma (micro) & 0.5 & 0.9999 & 0.000241 & 1.06E-07 & 1.28E-06 & 5.7E-16 \\ \hline

\end{tabular}
\caption{COMET vs Other Algorithms (AMI values only)}
\label{ami_table}
\end{table*}

\subsection{Ablation Study} \label{ablation_study}

\subsubsection{Gamma}
\label{gamma_app}

\begin{figure}[!htb]
    \centering
    \includegraphics[width=0.8\linewidth]{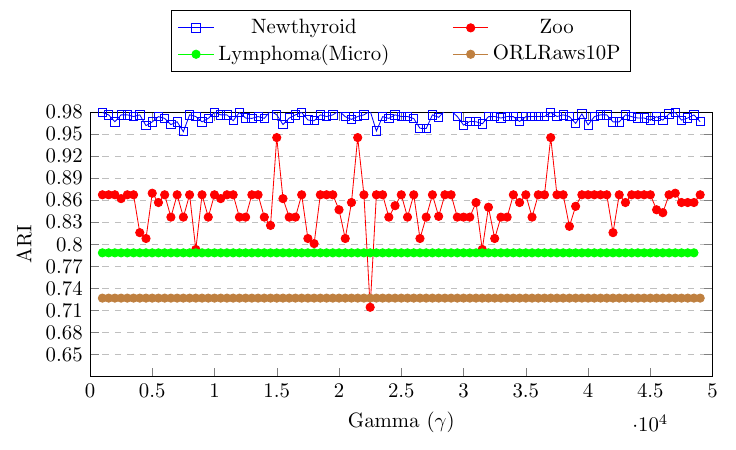}
    \caption{ARI for different values of $\gamma$ after adding 10$\%$ noise}
    \label{fig:gamma}
\end{figure}

We will do a sensitivity analysis of the hyperparameter $\gamma$ on the performance of our algorithm
For the convex clustering cost function, once a graph is created using $k-$NN for some k, the $u_i$'s converge to the mean of the connected components if the value of $\gamma$ is high. During our test runs, we also took a large value of $\gamma$, say 50000. We observe that, for our algorithm, the value of $\gamma$ doesn't influence the final result, provided it is large (say, $\geq$ 1000). Refer to the Figure \ref{fig:gamma} for the plot.


\subsubsection{Ablation study on Wisconsin Dataset}
\label{abl_wisc_app}

We will study the fluctuations in our performance measures with varying hyperparameters. 
We will turn to the \textbf{Wisconsin Breast Cancer} dataset for our ablation studies. We have shown in section \ref{gamma_app} that the hyperparameter $\gamma$ does not have much influence on the final clustering of the dataset. Two hyperparameters, k, which is a hyperparameter for the $k-$NN graph structure, and $\mu$, require tuning to achieve optimal performance of our algorithm.
For our experiments, we are varying k in $\{24,27,30,33,36\}$ , $\mu$ in $[4, 17]$ and p(noise level) in $\{0\%, 5\%, 10\%, 15\%, 20\%\}$.
For each pair (p,k), we are varing $\mu$ and reporting the mean of ARI and AMI. For each noise level, we observe that both the AMI and ARI values increase stochastically with $\mu$ and converge to a value. The fluctuations in ARI increase gradually with noise level for every value of k. Within each noise level however, ARI is most stable for the mid-range of k, which is 27 and 30, indicating that values of k from 27 to 30 have higher stabilities. Same observation can be made for AMI as well. All the fluctuations can be attributed to the randomness in adding noise and the optimization procedure of our objective function. The figures \ref{fig:Ablation_ari_wisc} and \ref{fig:Ablation_ami_wisc} correspond to the variation in ARI and AMI respectively for Wisconsin dataset.

\begin{figure*}[tb]
    \centering
    \begin{subfigure}[b]{0.32\textwidth}
        \centering
        \includegraphics[width=\linewidth]{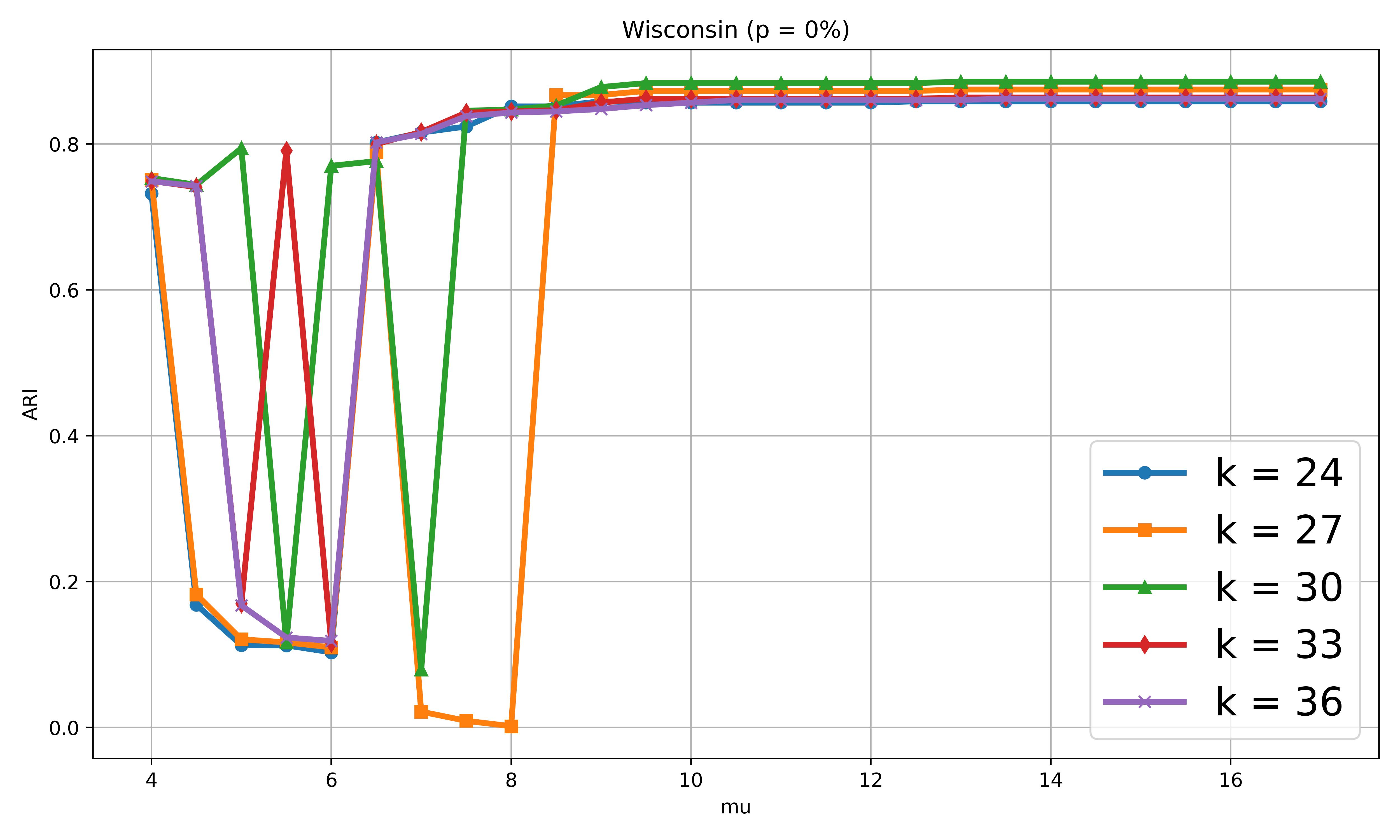}
        \caption{0 $\%$ noise}
    \end{subfigure}
    \hfill
    \begin{subfigure}[b]{0.32\textwidth}
        \centering
        \includegraphics[width=\linewidth]{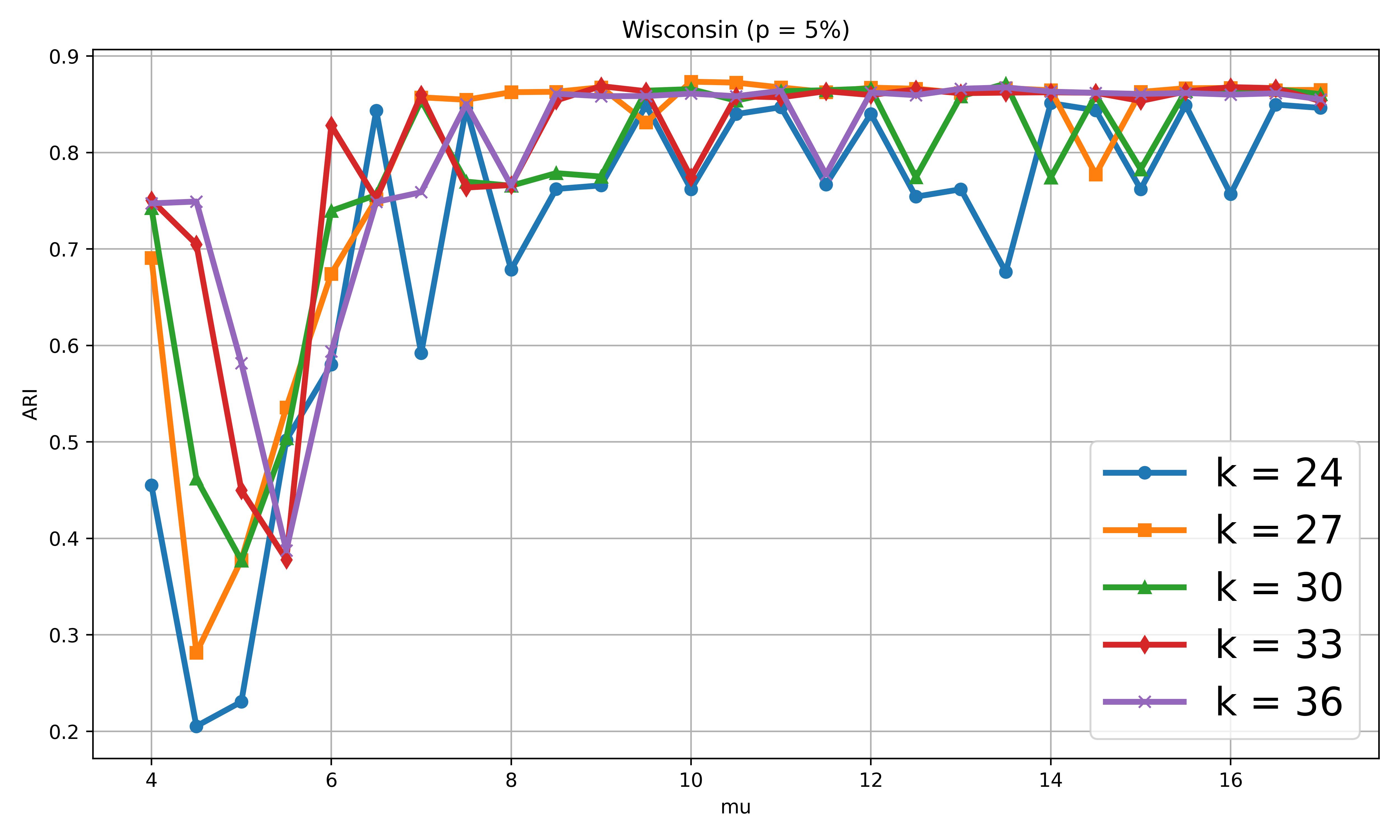}
        \caption{5 $\%$ noise}
    \end{subfigure}
    \hfill
    \begin{subfigure}[b]{0.32\textwidth}
        \centering
        \includegraphics[width=\linewidth]{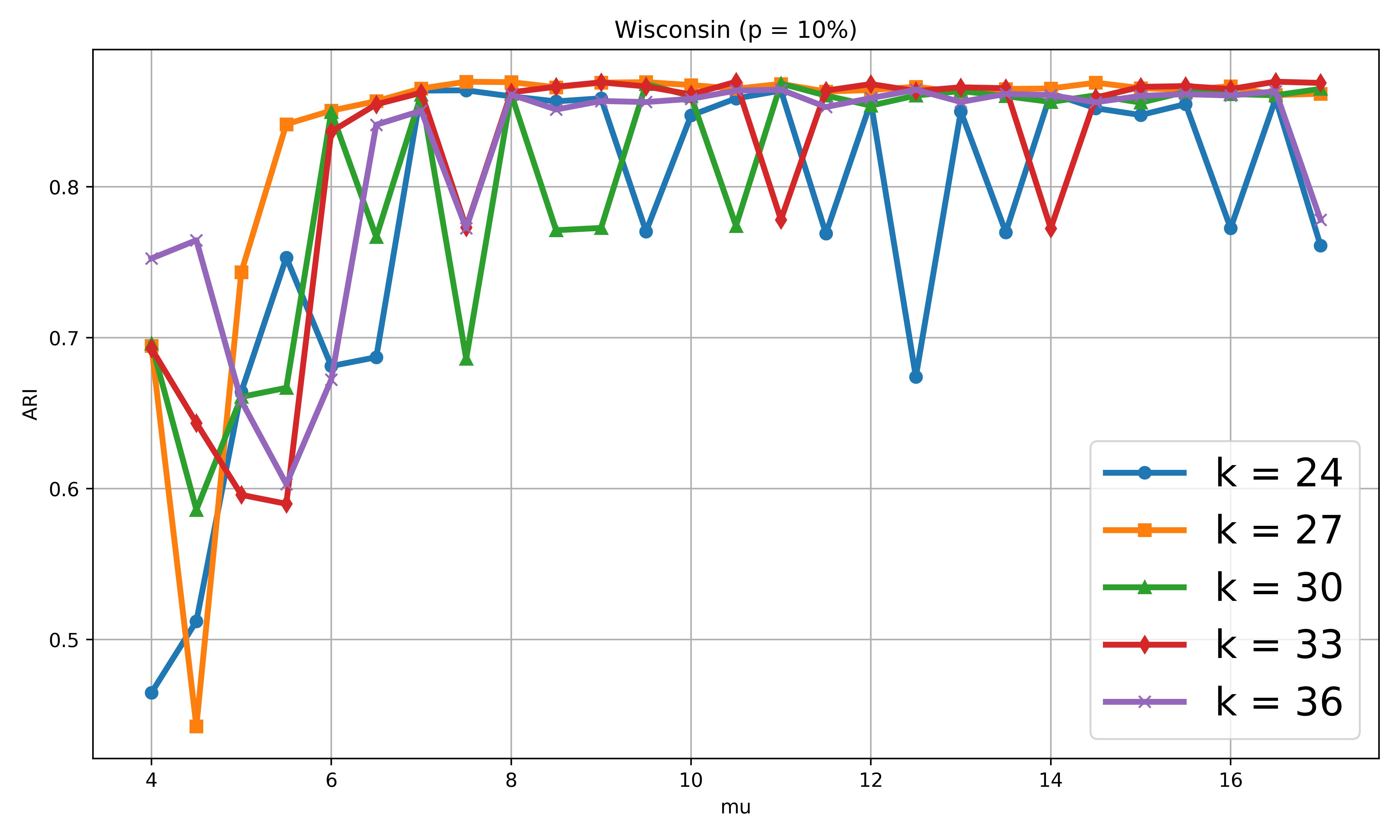}
        \caption{10 $\%$ noise}
    \end{subfigure}
    \vspace{10pt}
    \begin{subfigure}[b]{0.32\textwidth}
        \centering
        \includegraphics[width=\linewidth]{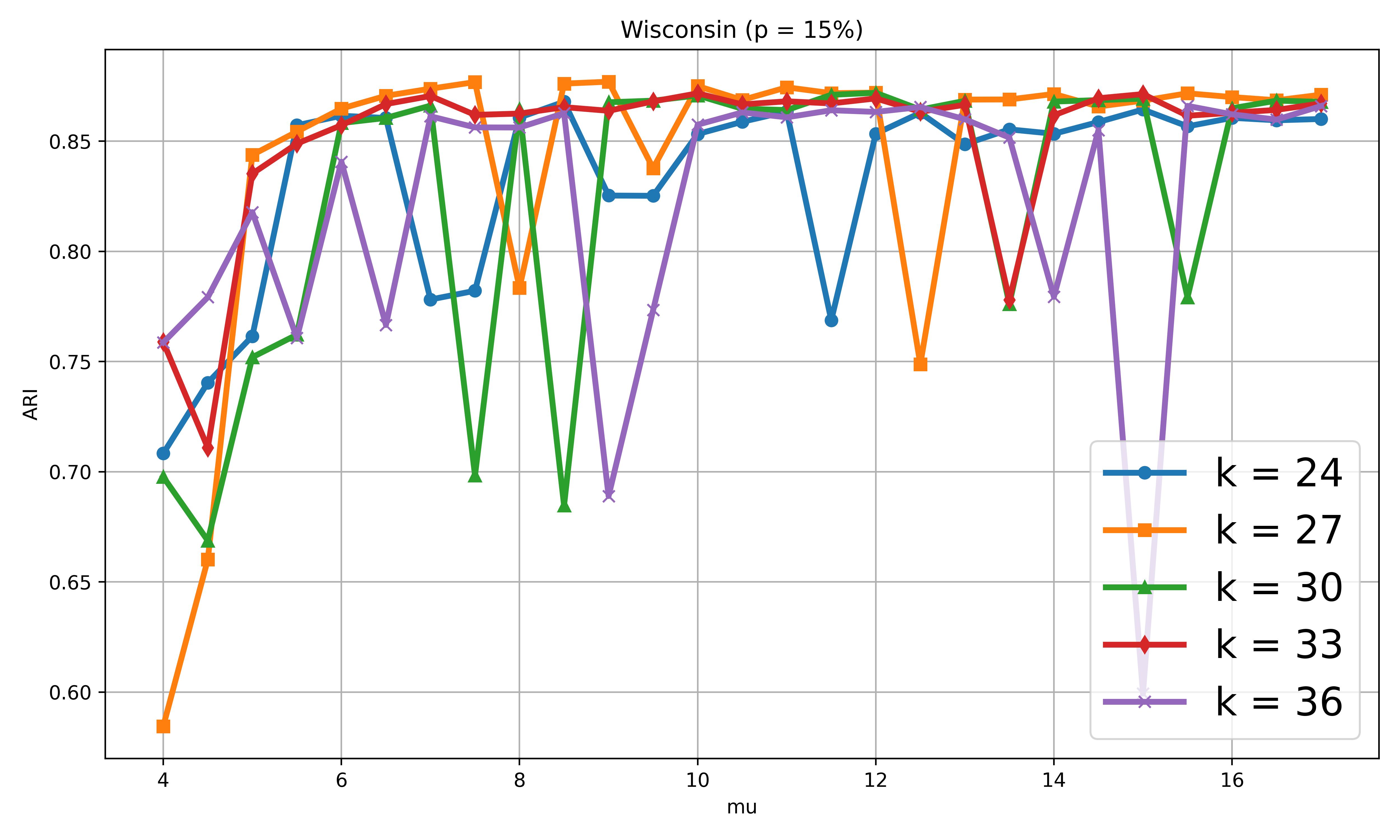}
        \caption{15 $\%$ noise}
    \end{subfigure}
    \begin{subfigure}[b]{0.32\textwidth}
        \centering
        \includegraphics[width=\linewidth]{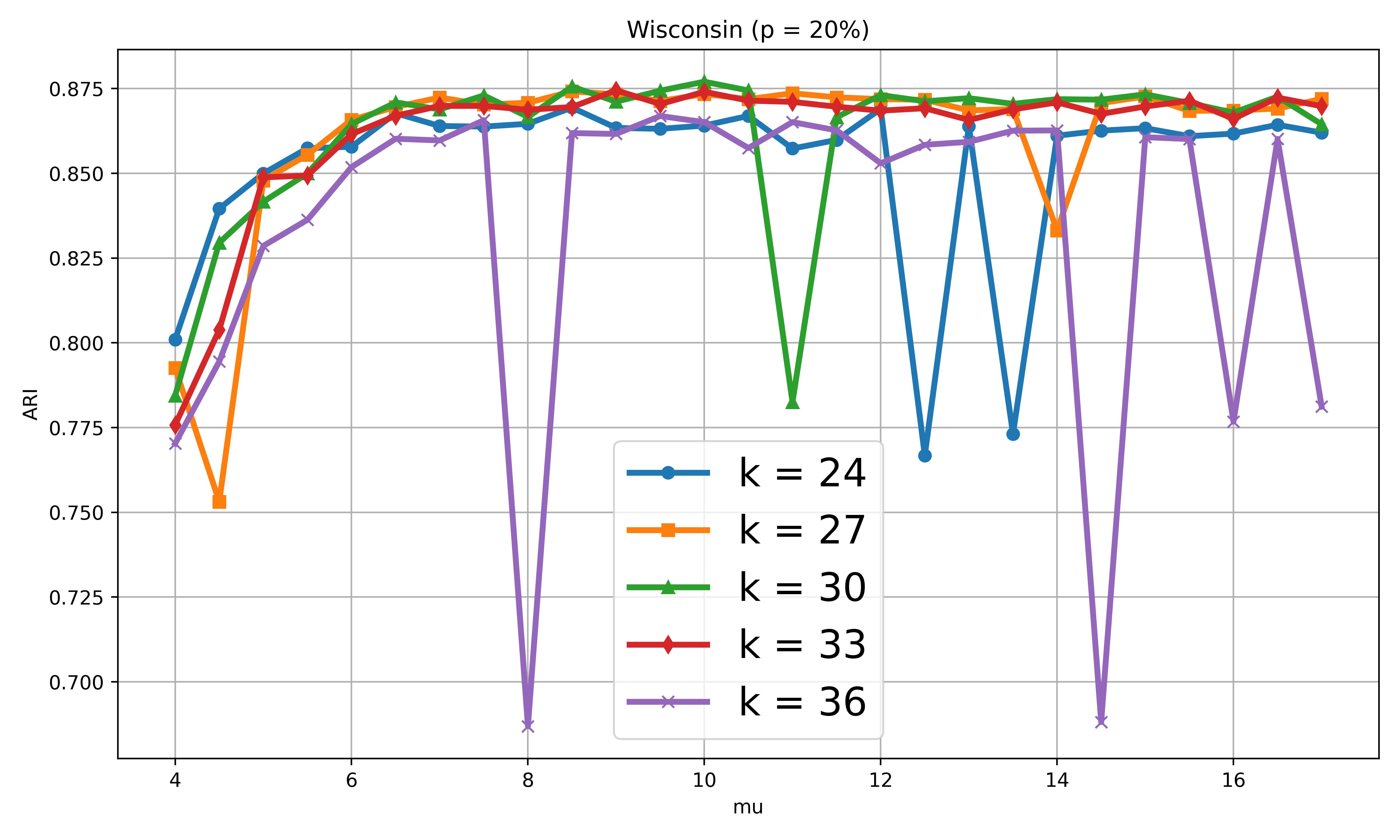}
        \caption{20 $\%$ noise}
    \end{subfigure}

    \caption{Ablation Studies of the ARI Values obtained from the Wisconsin Breast Cancer Dataset. Each subfigure corresponds to a different level of noise introduced into the dataset.}
    \label{fig:Ablation_ari_wisc}
\end{figure*}

\begin{figure*}[tb]
    \centering
    \begin{subfigure}[b]{0.32\textwidth}
        \centering
        \includegraphics[width=\linewidth]{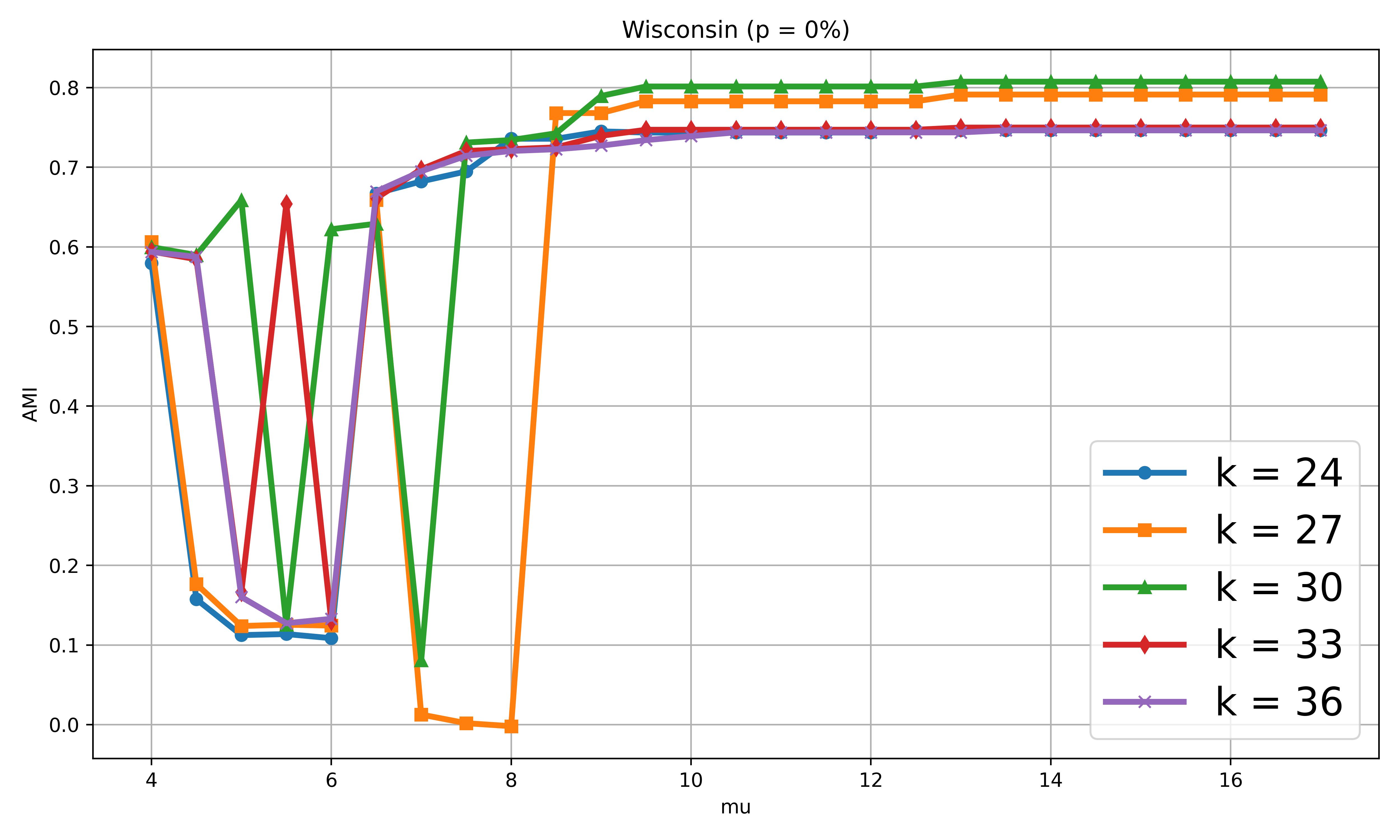}
        \caption{0 $\%$ noise}
    \end{subfigure}
    \hfill
    \begin{subfigure}[b]{0.32\textwidth}
        \centering
        \includegraphics[width=\linewidth]{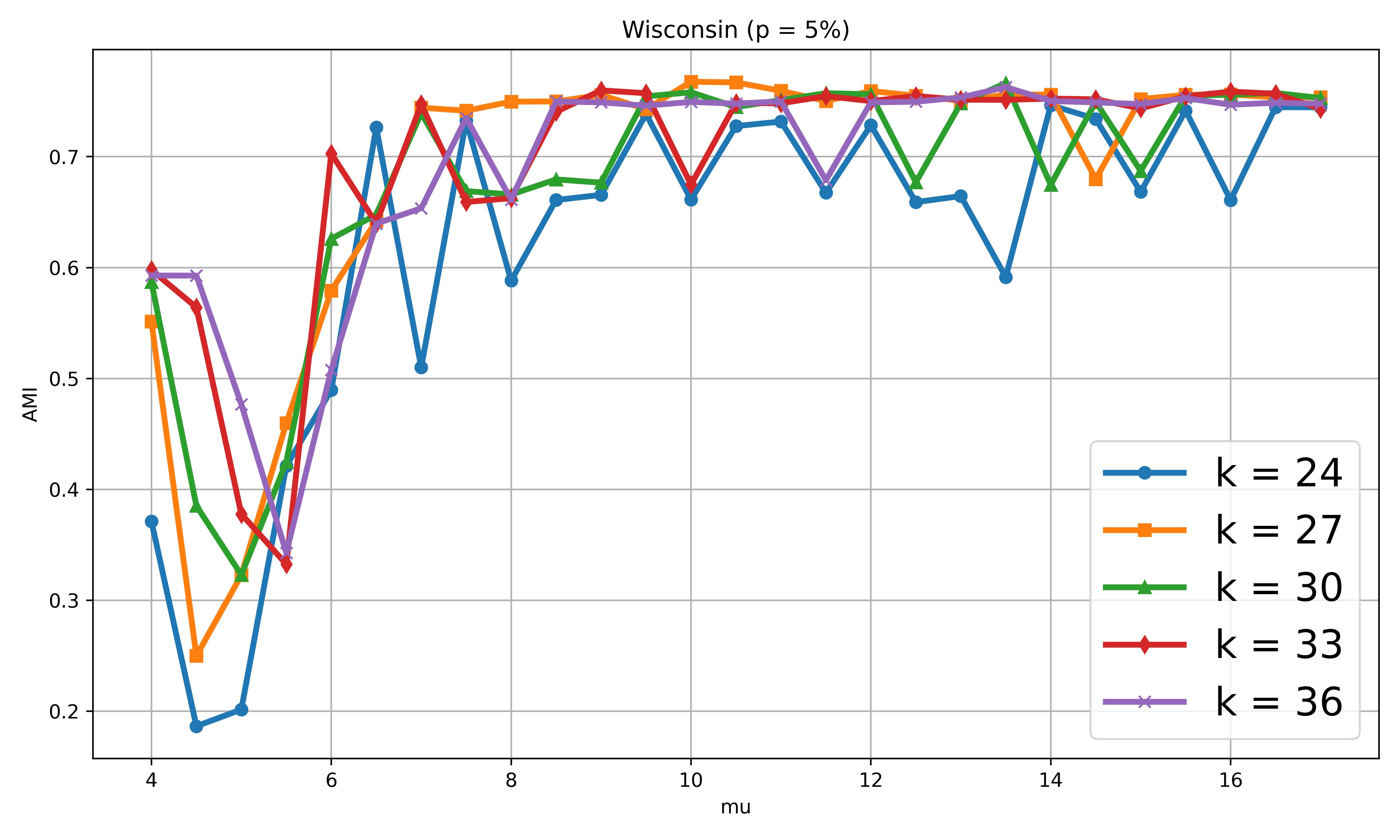}
        \caption{5 $\%$ noise}
    \end{subfigure}
    \hfill
    \begin{subfigure}[b]{0.32\textwidth}
        \centering
        \includegraphics[width=\linewidth]{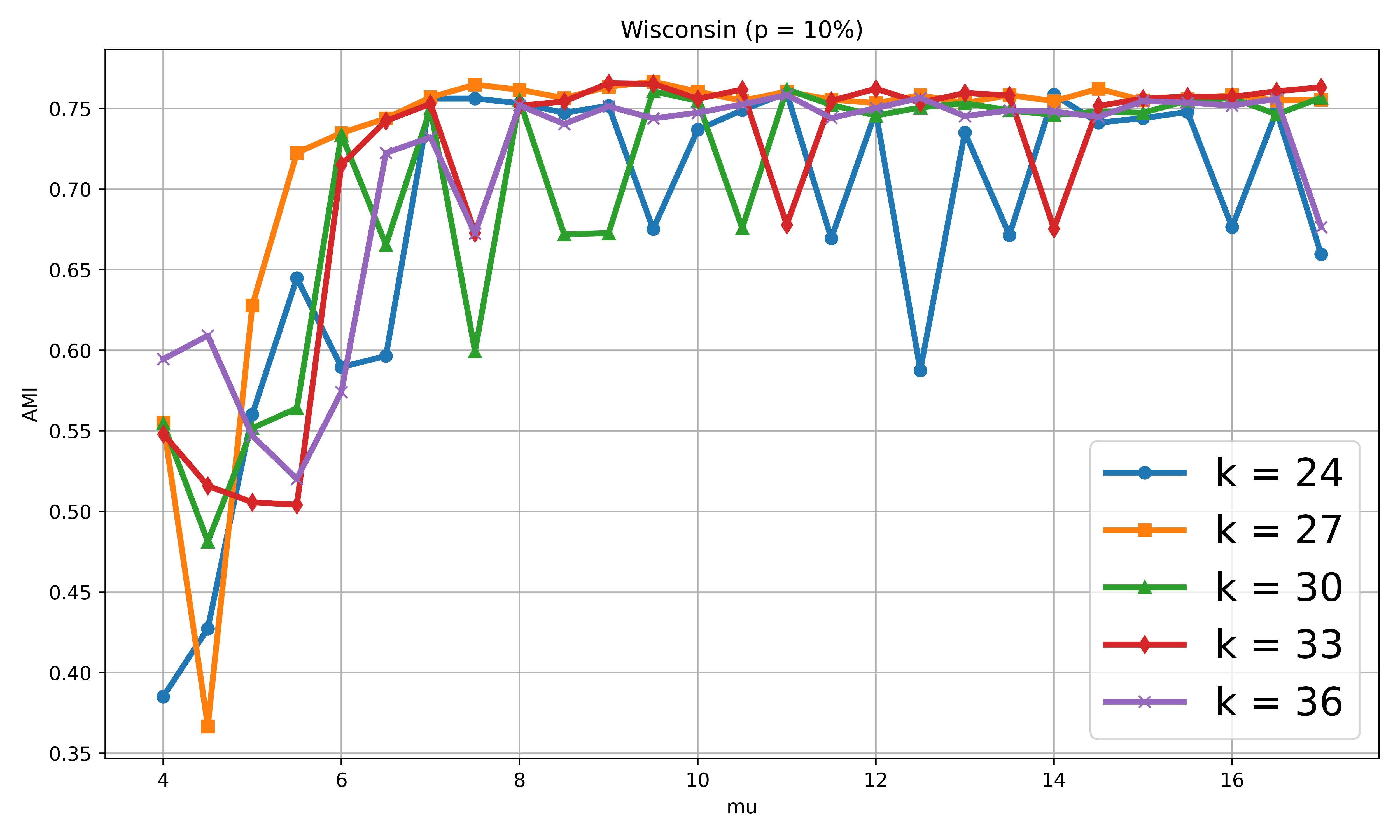}
        \caption{10 $\%$ noise}
    \end{subfigure}
    \vspace{10pt}
    \begin{subfigure}[b]{0.32\textwidth}
        \centering
        \includegraphics[width=\linewidth]{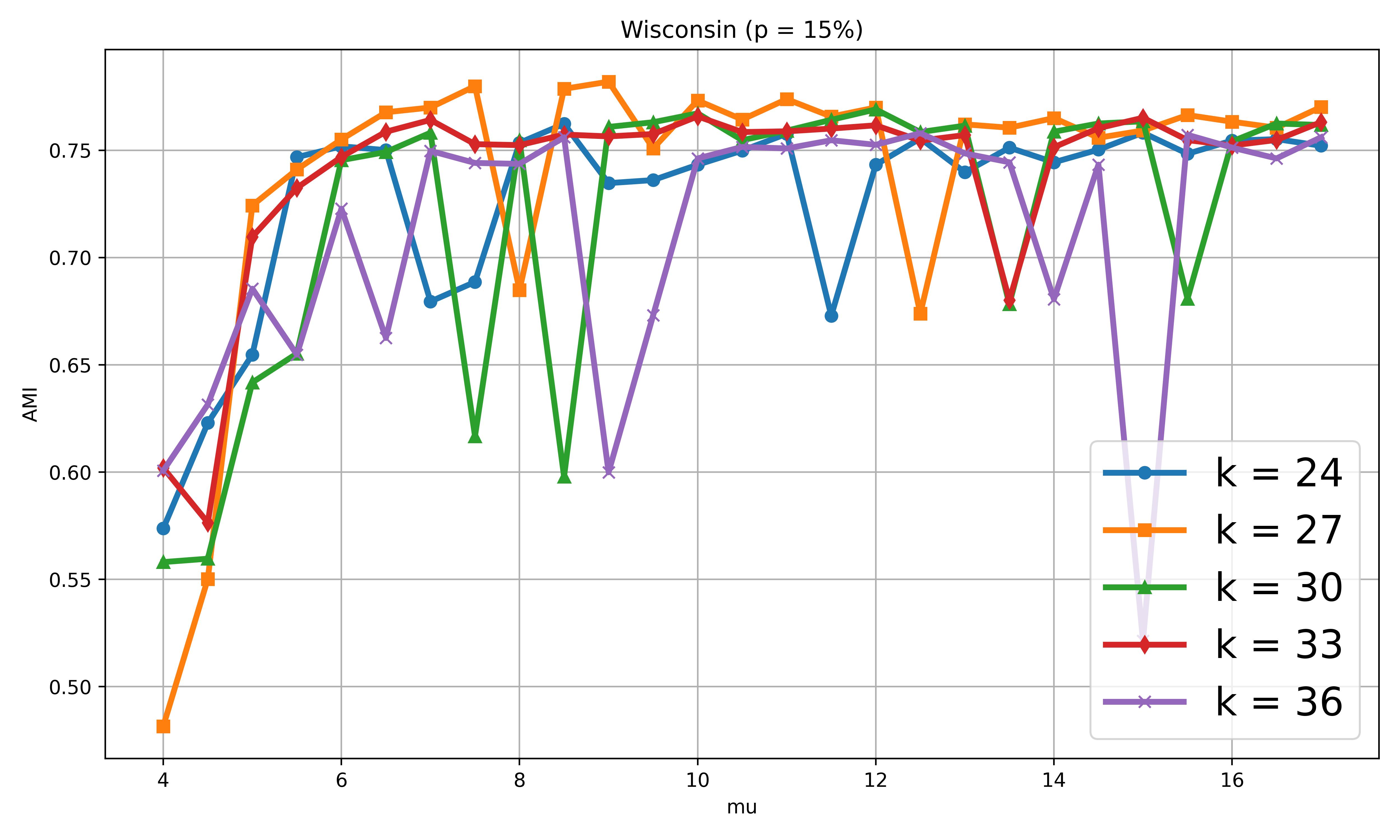}
        \caption{15 $\%$ noise}
    \end{subfigure}
    \begin{subfigure}[b]{0.32\textwidth}
        \centering
        \includegraphics[width=\linewidth]{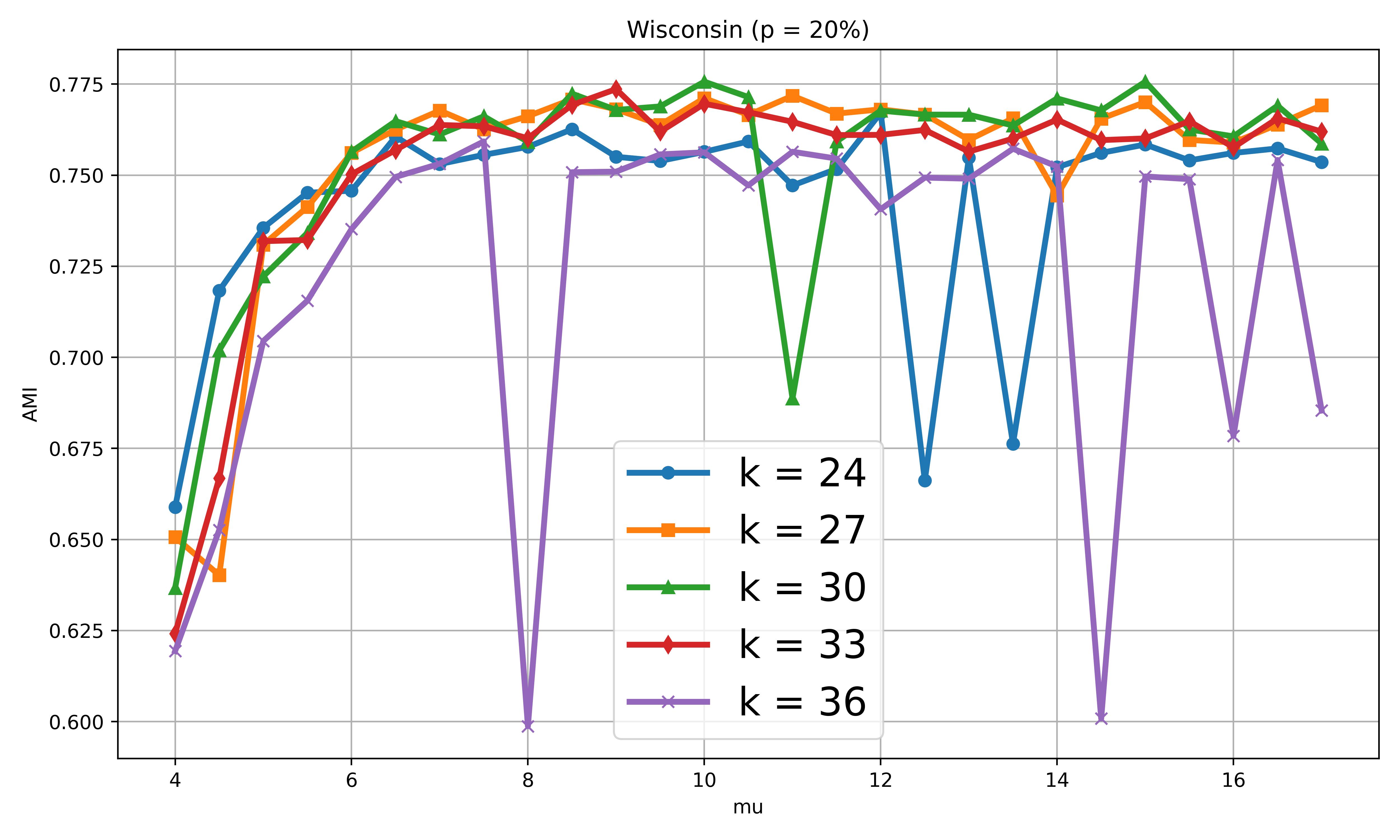}
        \caption{20 $\%$ noise}
    \end{subfigure}

    \caption{Ablation Studies of the AMI Values obtained from the Wisconsin Breast Cancer Dataset. Each subfigure corresponds to a different level of noise introduced into the dataset.}
    \label{fig:Ablation_ami_wisc}
\end{figure*}

\subsubsection{Ablation study on Newthyroid Dataset}
\label{abl_nwthy_app}
We will turn to the Newthyroid dataset for our ablation studies. We have shown in section \ref{gamma_app} that the hyperparameter $\gamma$ does not have much influence on the final clustering of the dataset. Two hyperparameters, $k$, which is a hyperparameter for the $k-$NN graph structure, and $\mu$, require tuning to achieve optimal performance of our algorithm. For our experiments, we are varying k in $\{31, 38, 45, 52, 59\}$ , $\mu$ in $[12, 110]$ and p(noise level) in $\{0\%, 5\%, 10\%, 15\%, 20\%\}$. For each pair ($p$,$k$), we are varying $\mu$ and reporting the mean of ARI and AMI. The ARI is almost always between 0.93 and 0.98 and very rarely drops to 0.92 at higher noise levels. The AMI values are also between 0.82 and 0.92. The Here we did not notice much fluctuations for both ARI and AMI values even for $20\%$ noise level. All the fluctuations can be attributed to the randomness of adding noise and the optimization procedure of our objective function. The fluctuations increase with increasing p, however the increase is only slight. The fluctuations decrease while choosing higher values of k (such as 52 in our case). The effect of changing $\mu$ is also not apparent from the plots indicating that there are well-separated clusters in this dataset. However, for much smaller values of $\mu$ (say below 10), the ARI and AMI values are expected to drop. The figures \ref{fig:Ablation_ari} and \ref{fig:Ablation_ami} correspond to the variation in ARI and AMI respectively for Brain dataset.
\begin{figure*}[tb]
    \centering
    \begin{subfigure}[b]{0.32\textwidth}
        \centering
        \includegraphics[width=\linewidth]{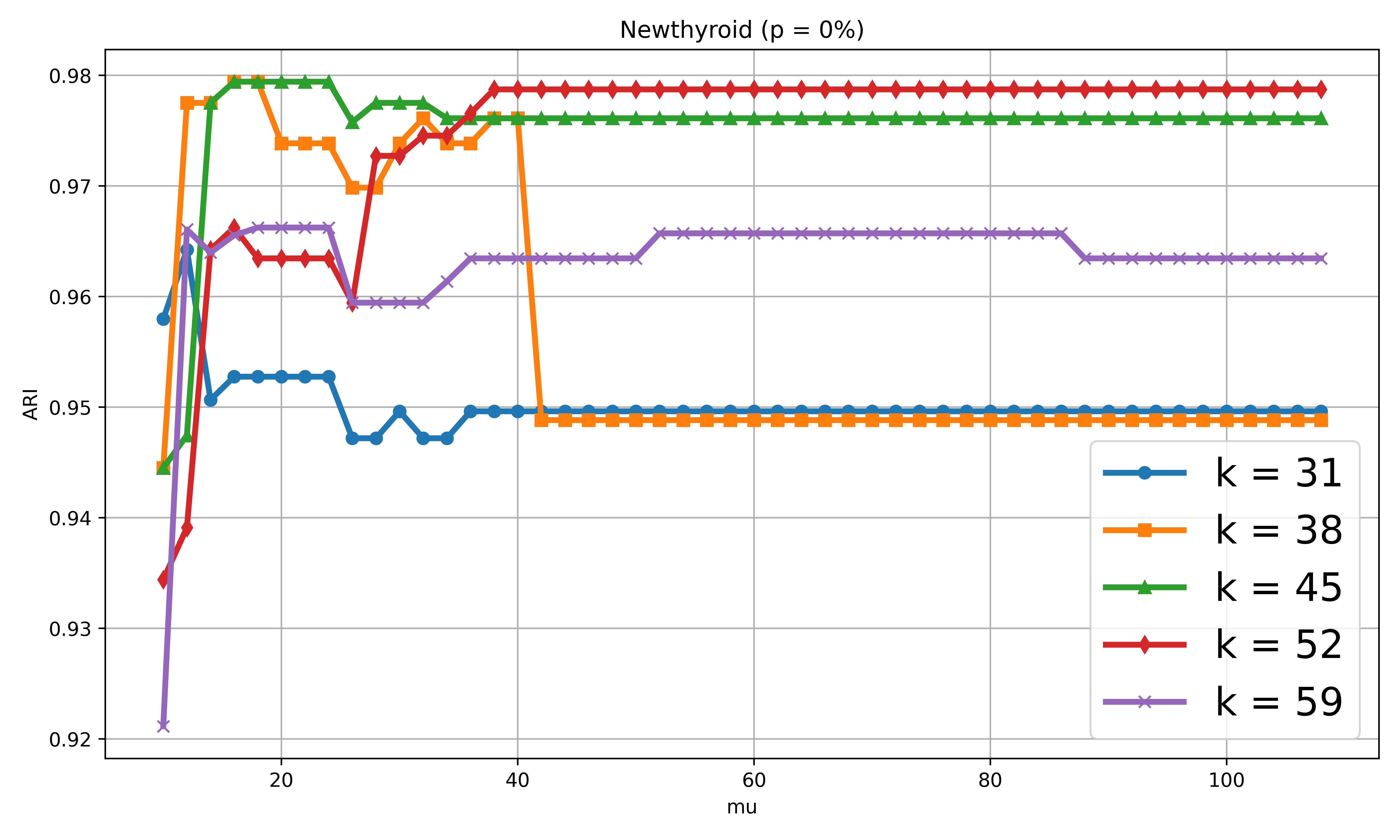}
        \caption{0 $\%$ noise}
    \end{subfigure}
    \hfill
    \begin{subfigure}[b]{0.32\textwidth}
        \centering
        \includegraphics[width=\linewidth]{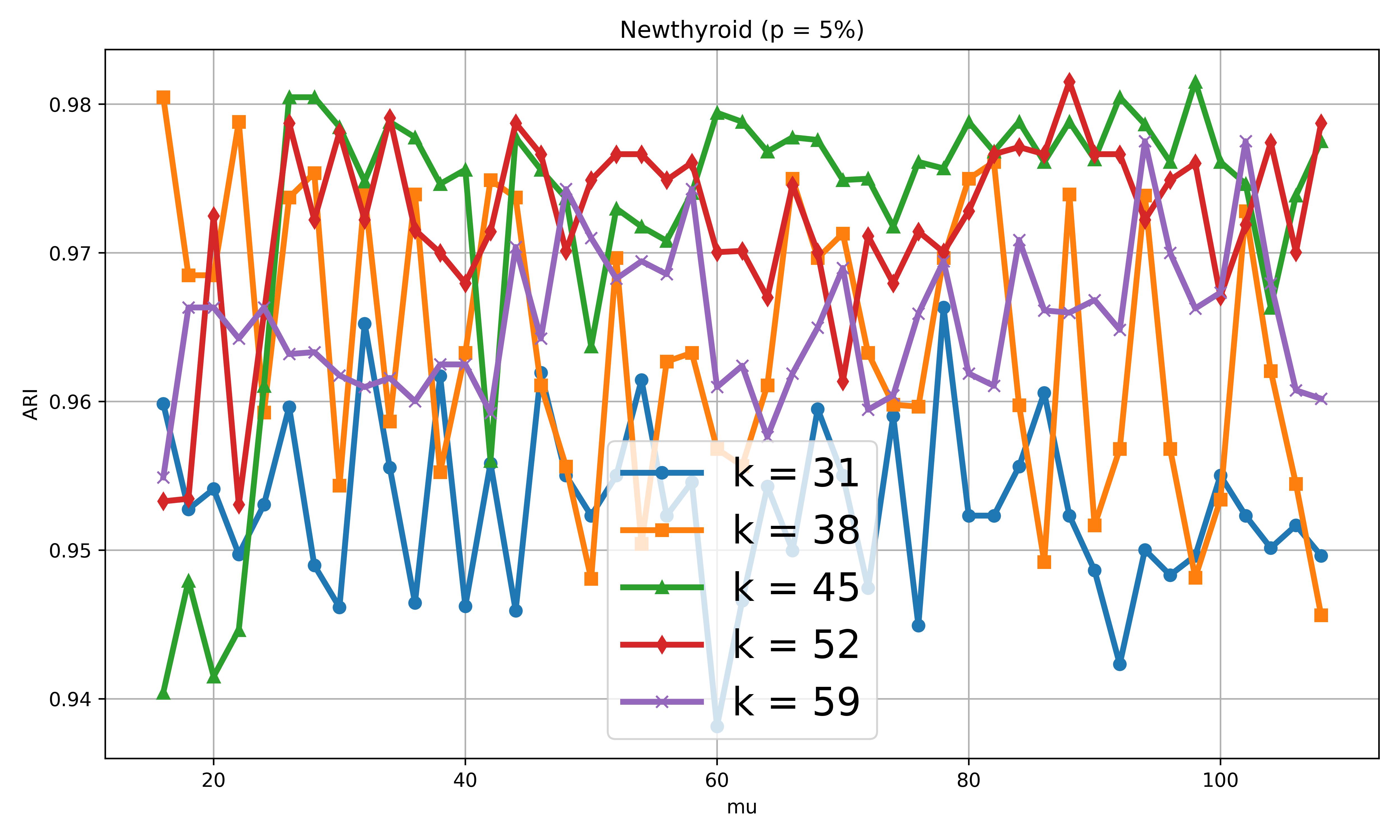}
        \caption{5 $\%$ noise}
    \end{subfigure}
    \hfill
    \begin{subfigure}[b]{0.32\textwidth}
        \centering
        \includegraphics[width=\linewidth]{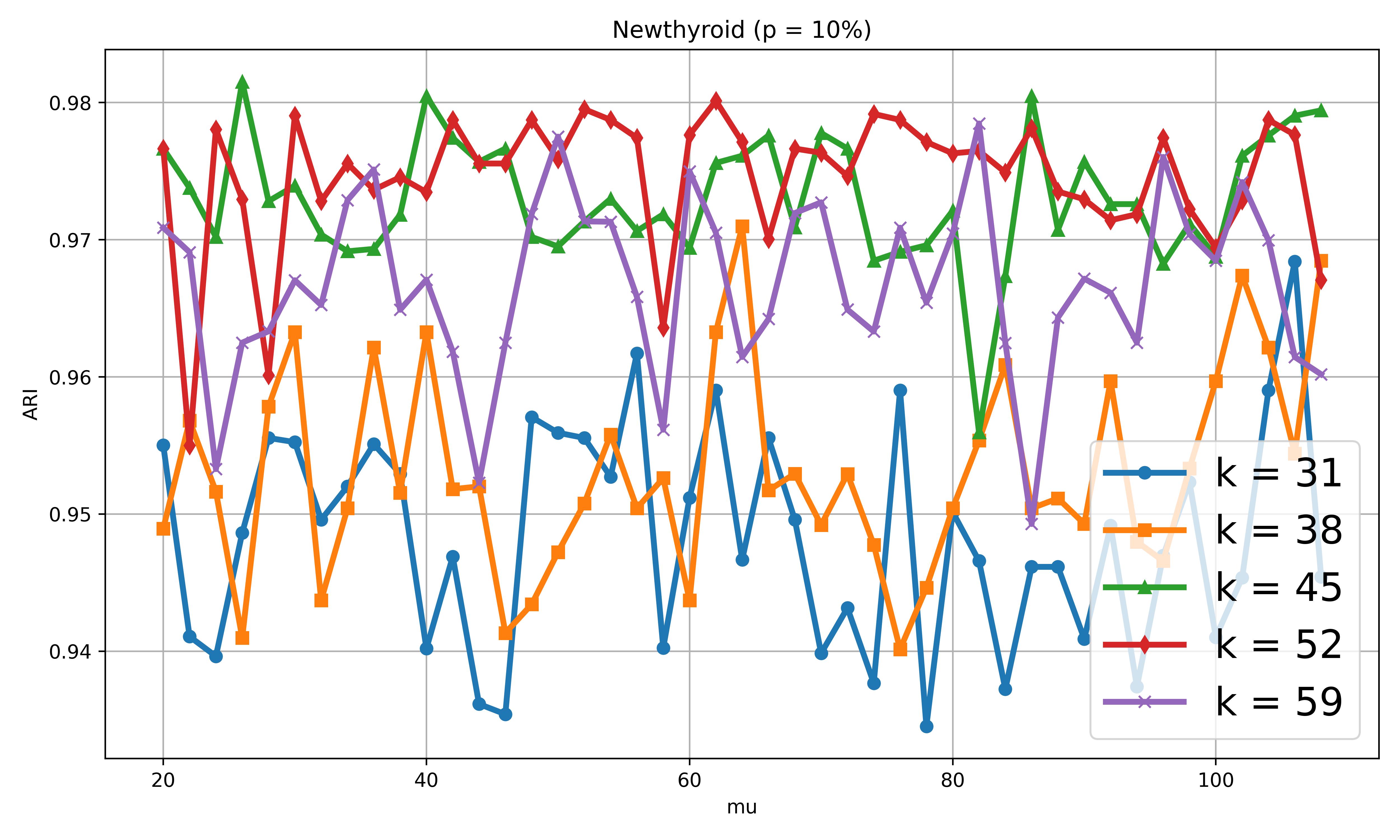}
        \caption{10 $\%$ noise}
    \end{subfigure}
    \vspace{10pt}
    \begin{subfigure}[b]{0.32\textwidth}
        \centering
        \includegraphics[width=\linewidth]{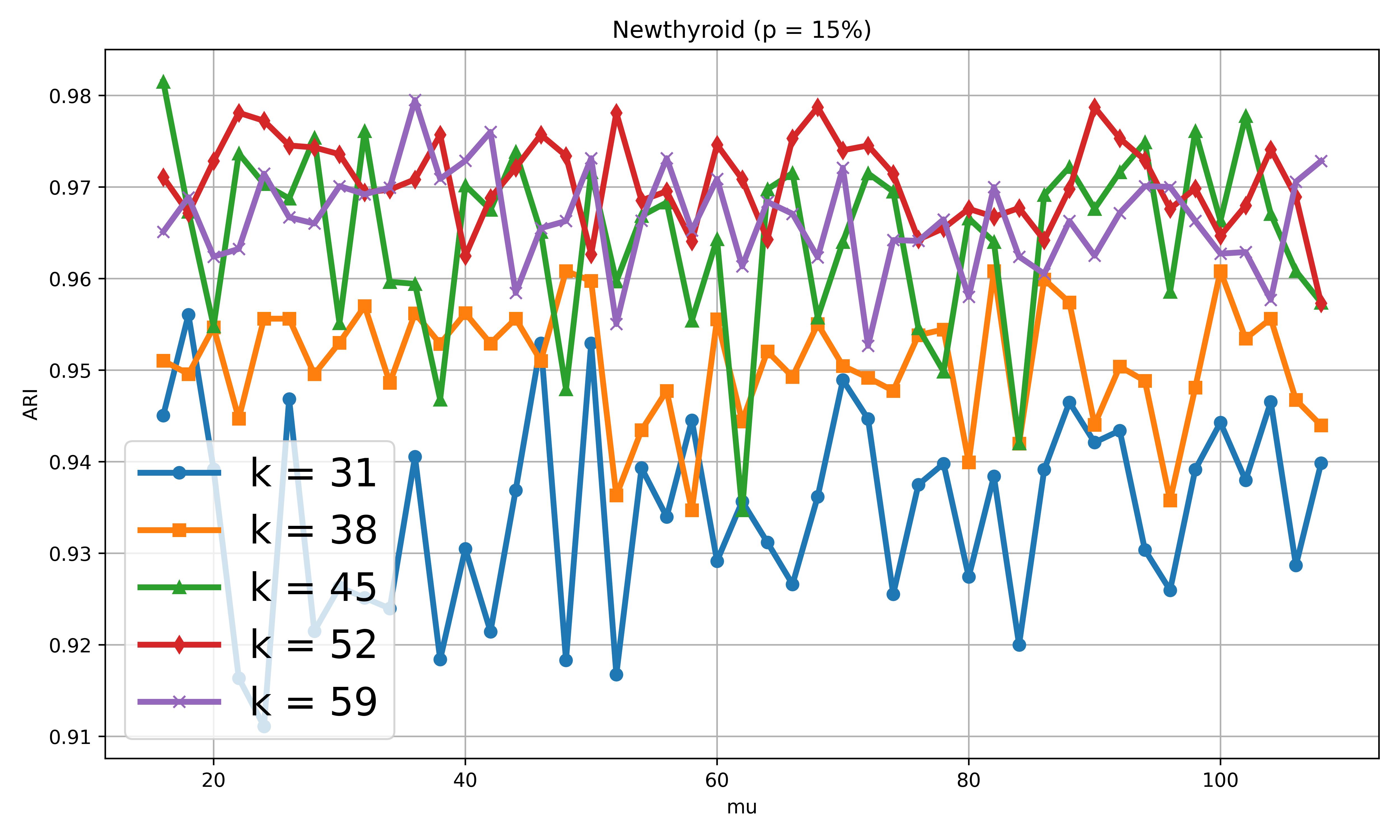}
        \caption{15 $\%$ noise}
    \end{subfigure}
    \begin{subfigure}[b]{0.32\textwidth}
        \centering
        \includegraphics[width=\linewidth]{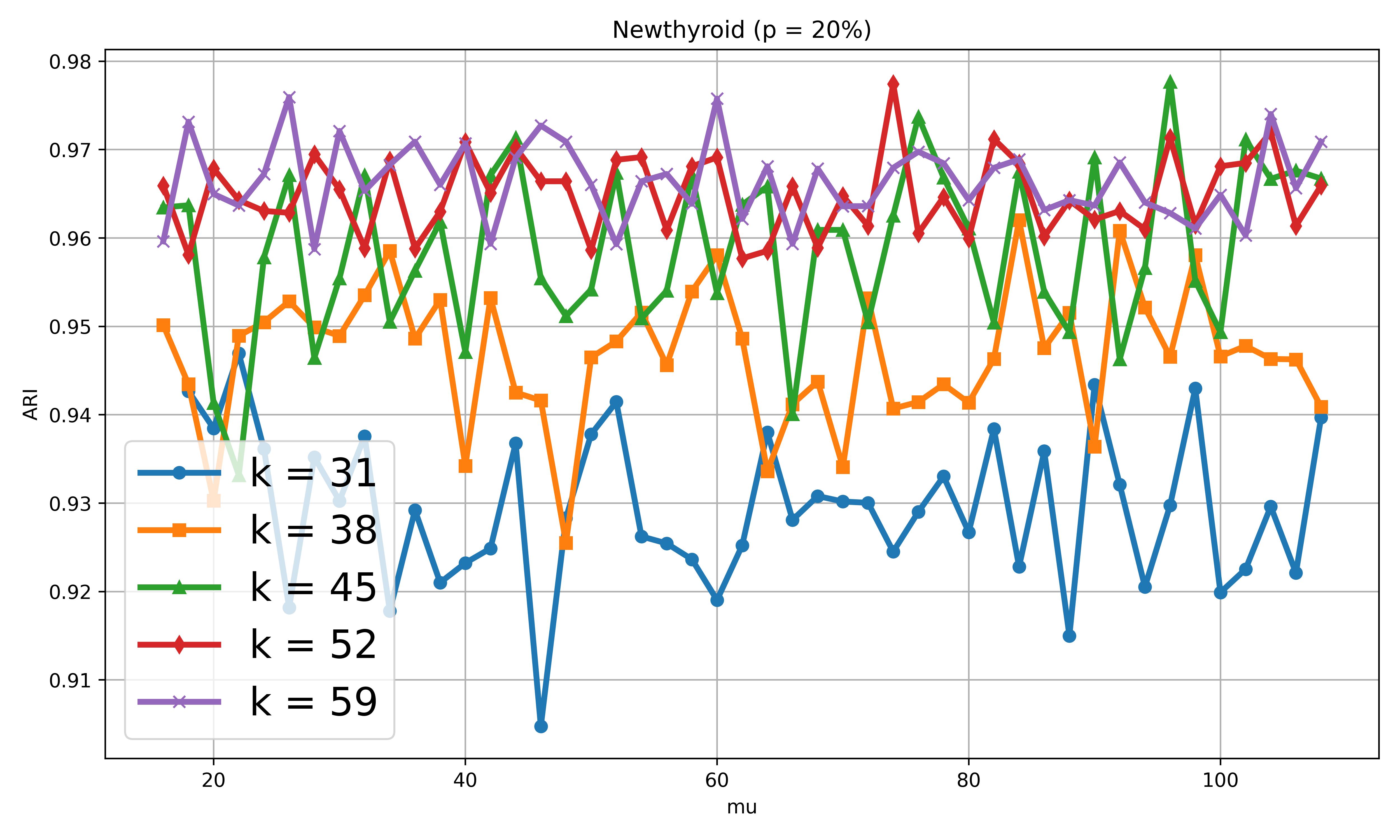}
        \caption{20 $\%$ noise}
    \end{subfigure}

    \caption{Ablation Studies of the ARI Values obtained from the NewThyroid Dataset. Each subfigure corresponds to a different level of noise introduced into the dataset.}
    \label{fig:Ablation_ari}
\end{figure*}

\begin{figure*}[tb]
    \centering

    \begin{subfigure}[b]{0.32\textwidth}
        \centering
        \includegraphics[width=\linewidth]{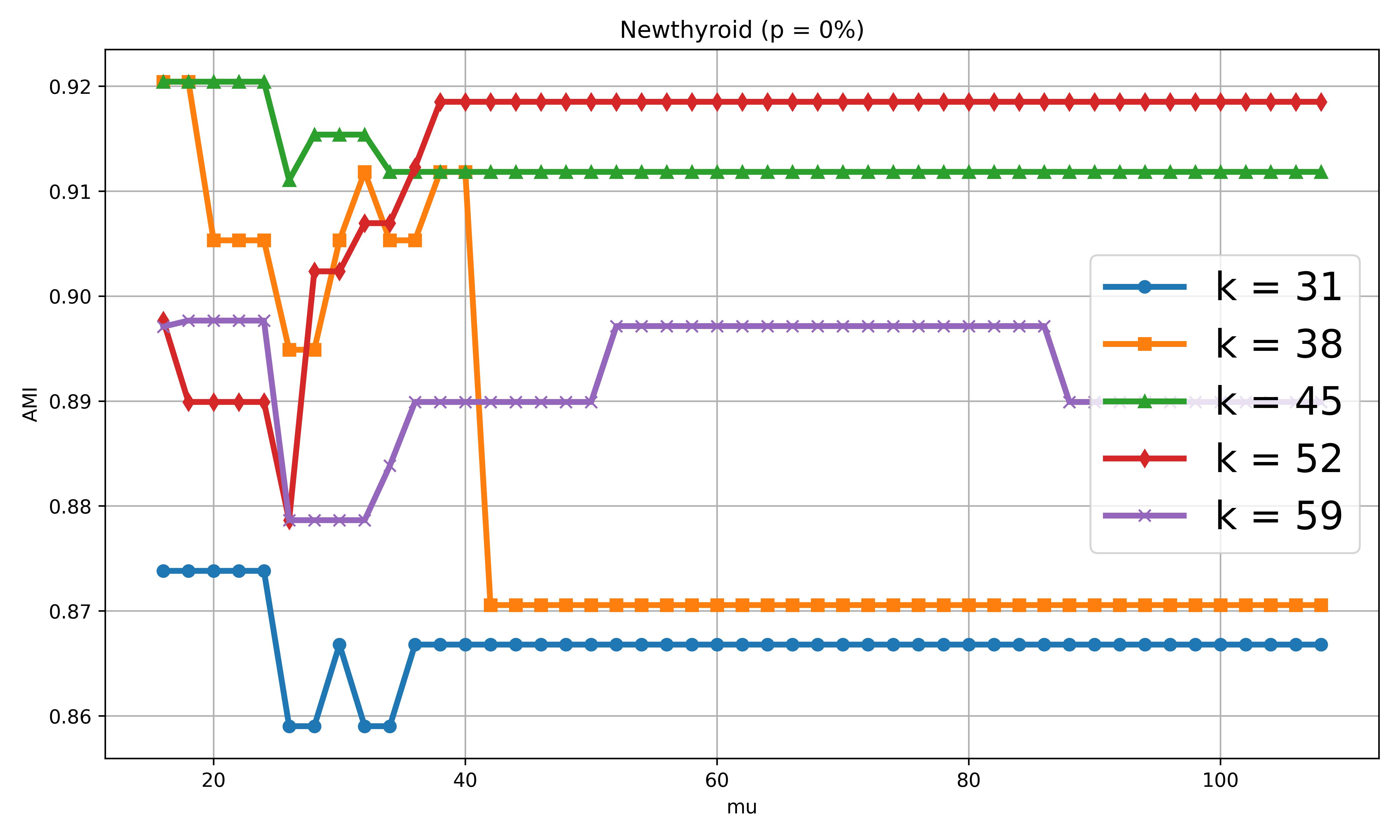}
        \caption{0 $\%$ noise}
    \end{subfigure}
    \hfill
    \begin{subfigure}[b]{0.32\textwidth}
        \centering
        \includegraphics[width=\linewidth]{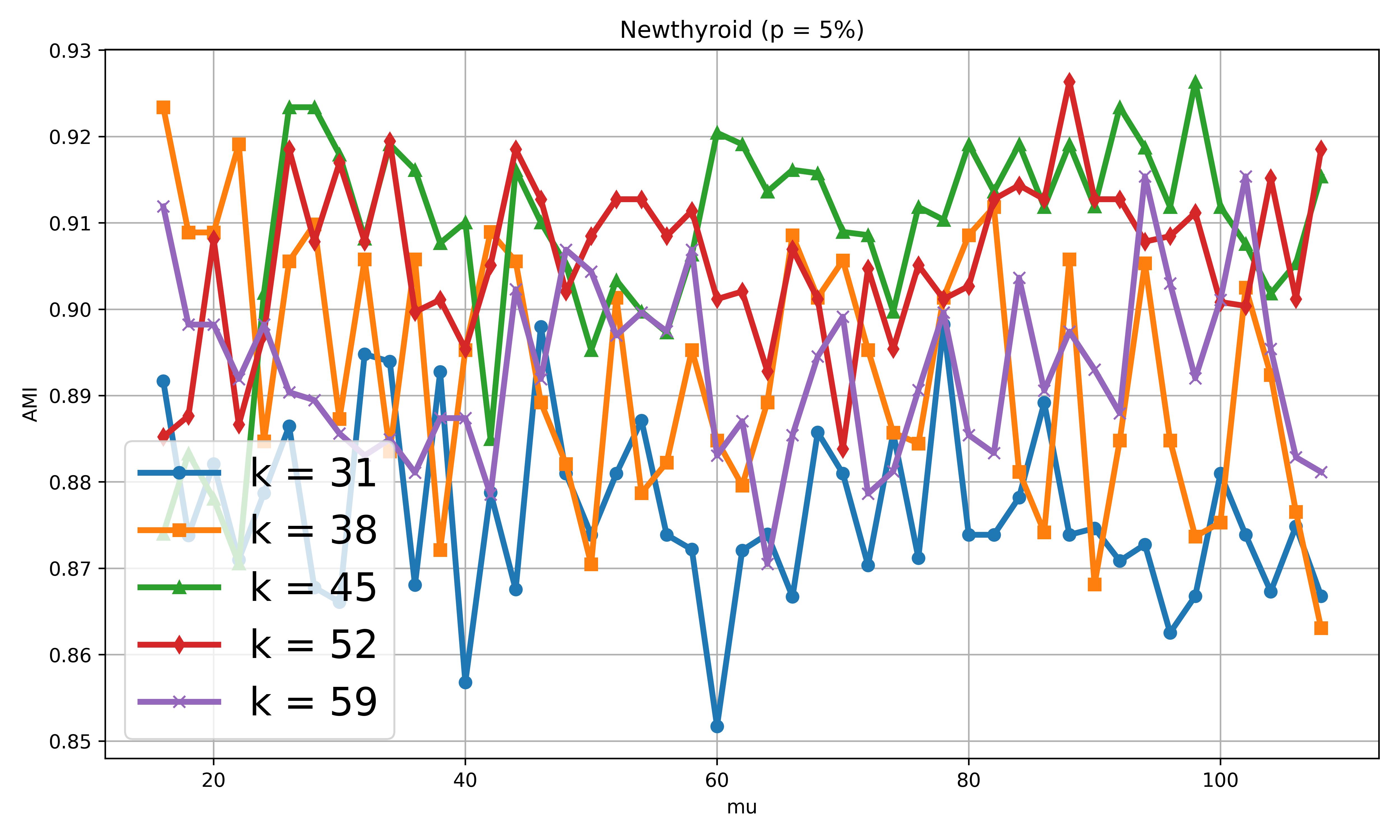}
        \caption{5 $\%$ noise}
    \end{subfigure}
    \hfill
    \begin{subfigure}[b]{0.32\textwidth}
        \centering
        \includegraphics[width=\linewidth]{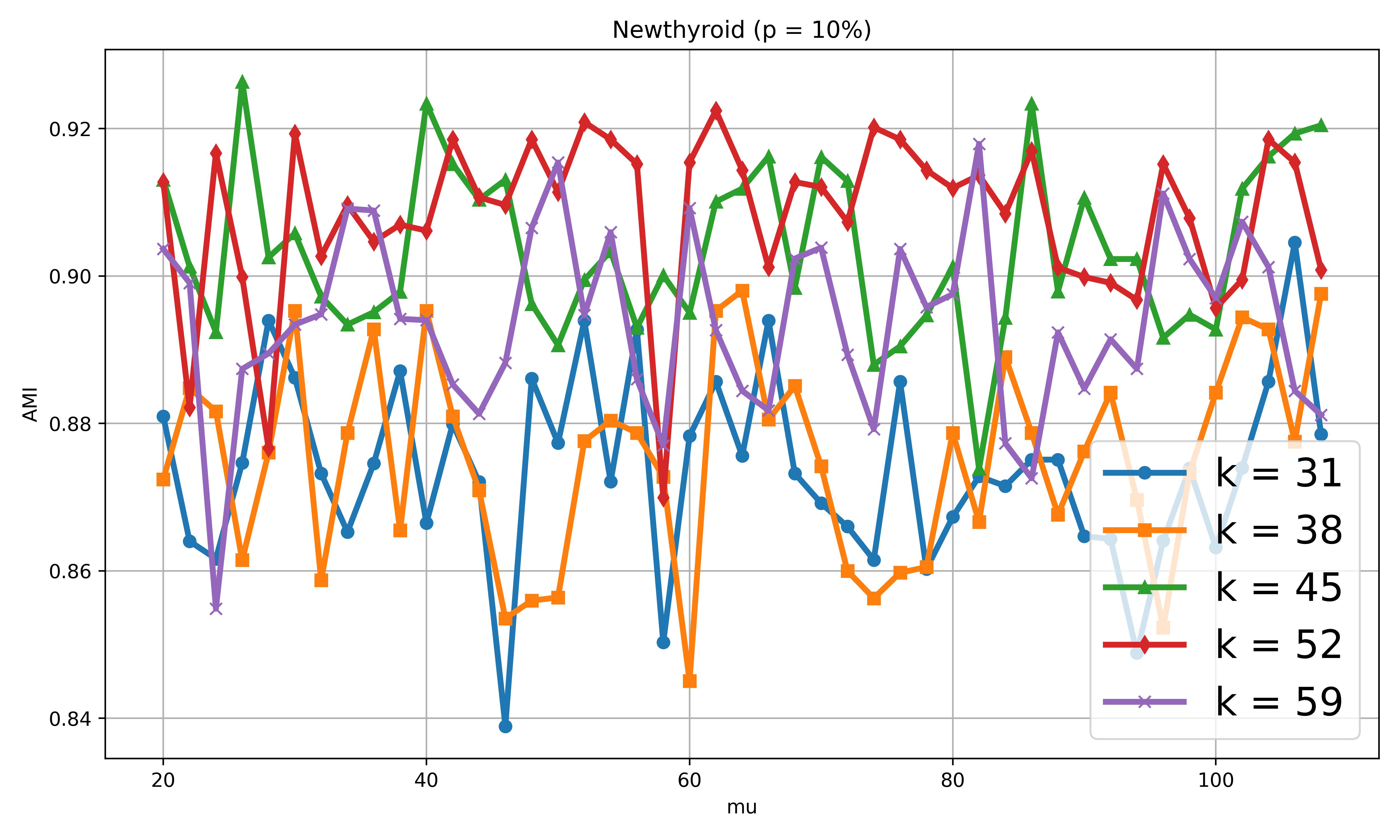}
        \caption{10 $\%$ noise}
    \end{subfigure}
    \vspace{10pt}
    \begin{subfigure}[b]{0.32\textwidth}
        \centering
        \includegraphics[width=\linewidth]{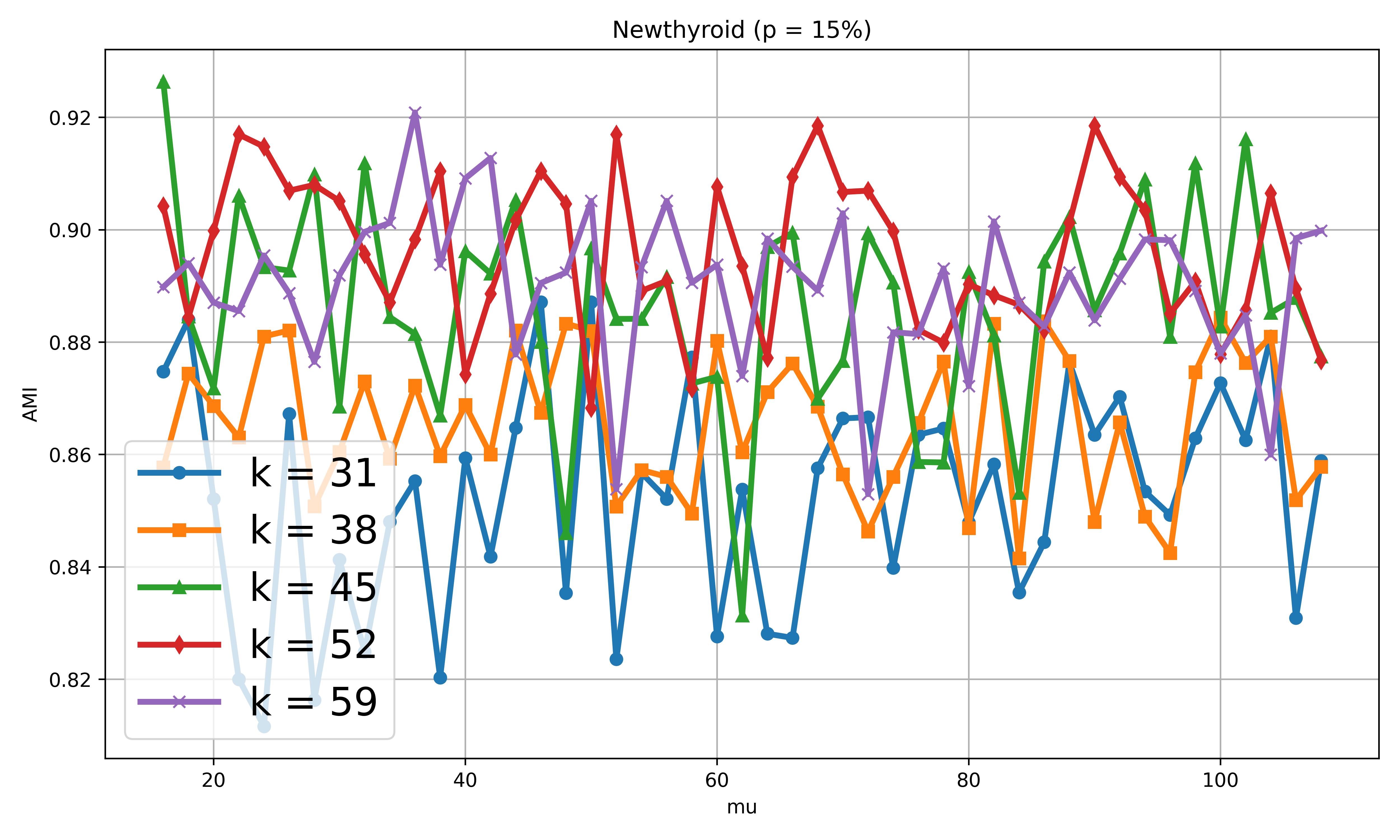}
        \caption{15 $\%$ noise}
    \end{subfigure}
    \begin{subfigure}[b]{0.32\textwidth}
        \centering
        \includegraphics[width=\linewidth]{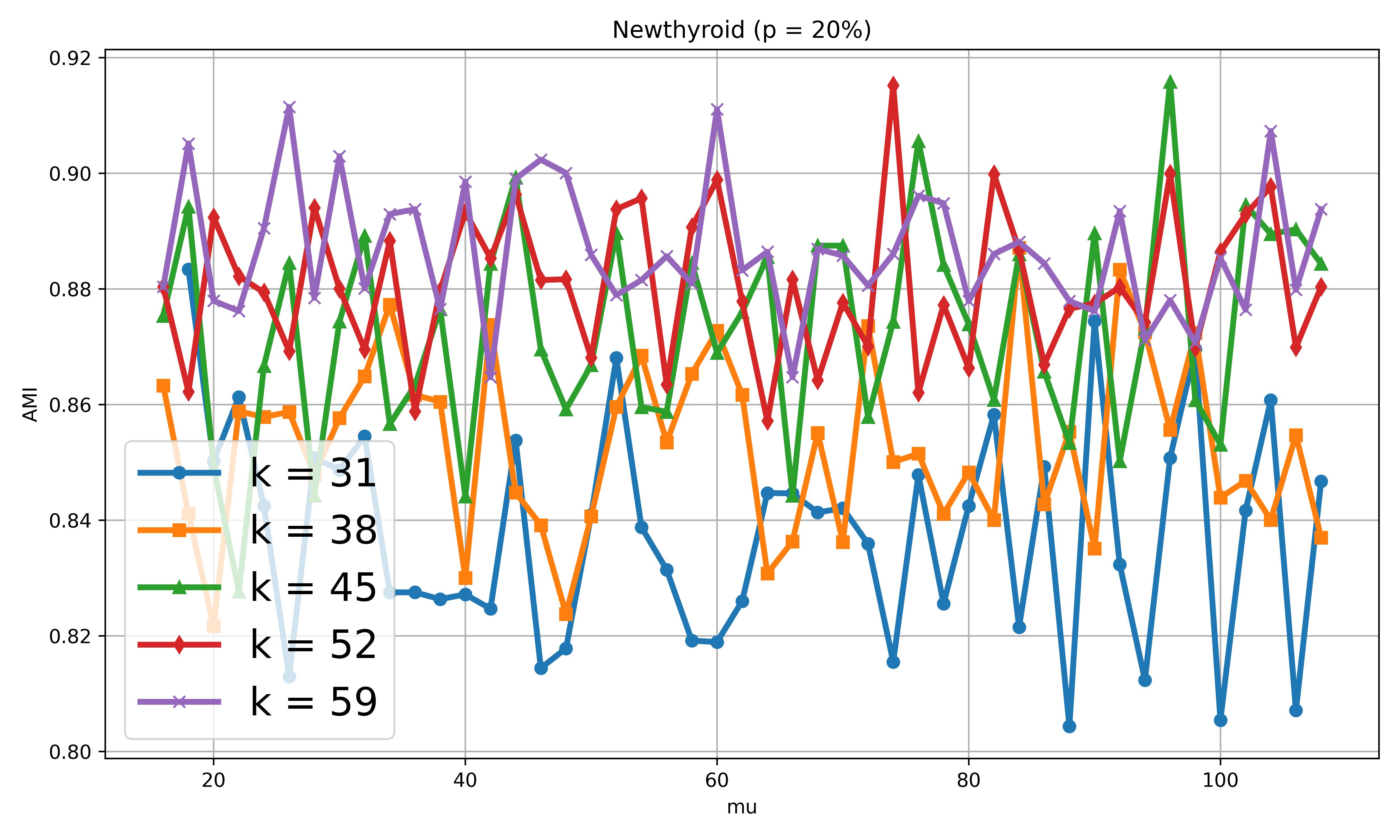}
        \caption{20 $\%$ noise}
    \end{subfigure}

    \caption{Ablation Studies of the AMI Values obtained from the NewThyroid Dataset. Each subfigure corresponds to a different level of noise introduced into the dataset.}
    \label{fig:Ablation_ami}
\end{figure*}

\subsection{Table of Plots for Simulated Datasets}
\begin{figure*}[h!]
    \centering
    \begin{tabular}{m{3cm}|ccc}
        \textbf{Algorithm} & \textbf{Blobs} & \textbf{Circles} & \textbf{Moons} \\
        \hline
        \textbf{Dataset} &
        \includegraphics[width=0.2\textwidth]{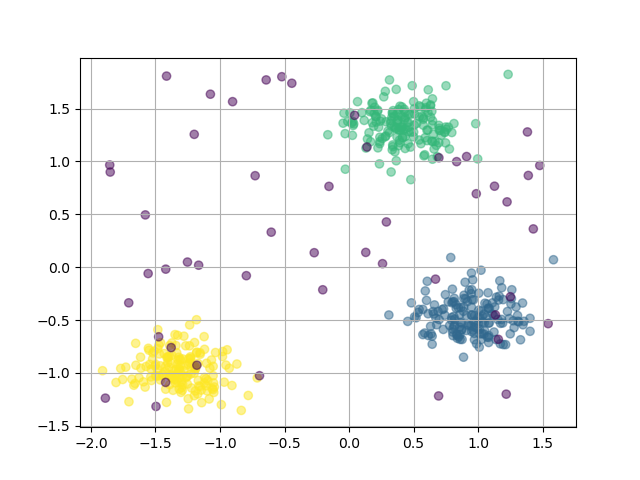} &
        \includegraphics[width=0.2\textwidth]{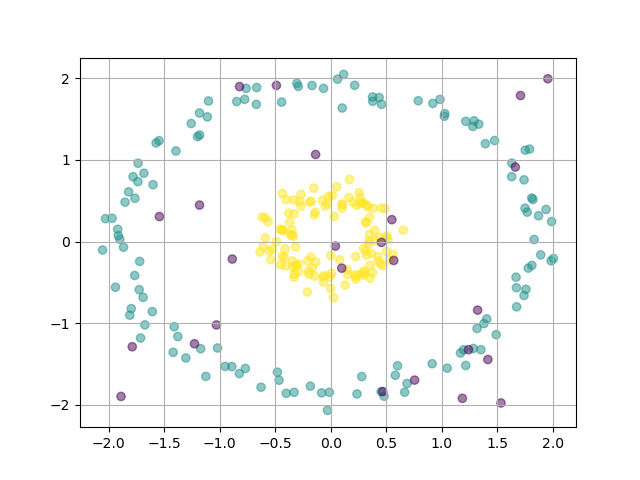} &
        \includegraphics[width=0.19\textwidth]{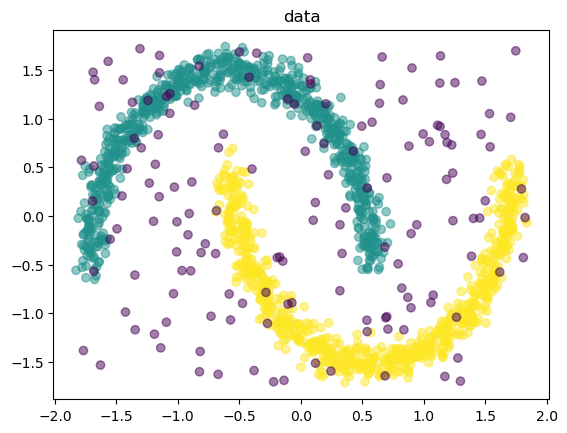} \\
    
        \textbf{COMET} &
        \includegraphics[width=0.2\textwidth]{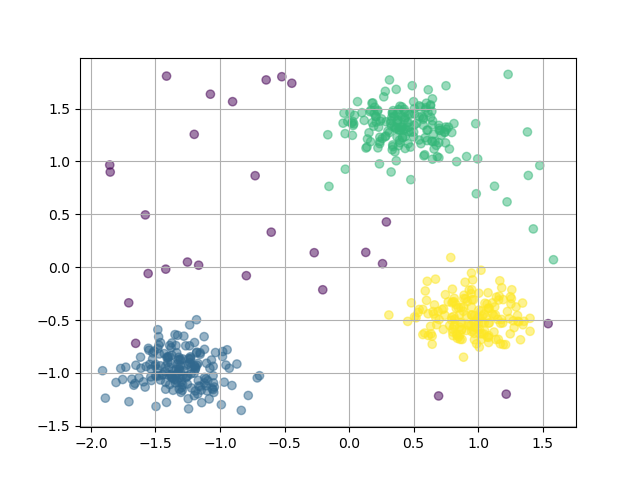} &
        \includegraphics[width=0.2\textwidth]{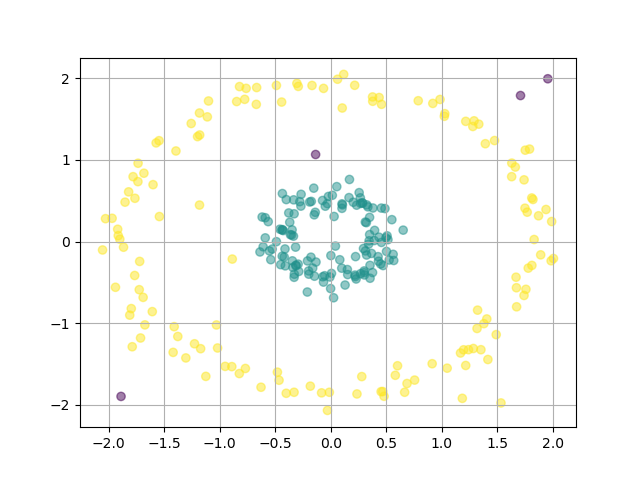} &
        \includegraphics[width=0.2\textwidth]{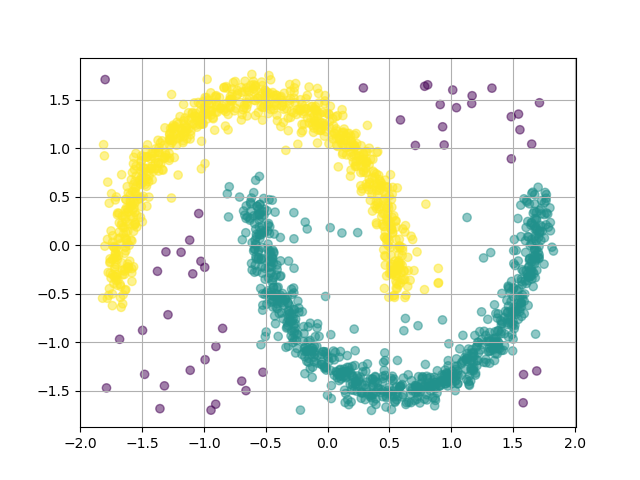} \\
        
        \textbf{KM} &
        \includegraphics[width=0.2\textwidth]{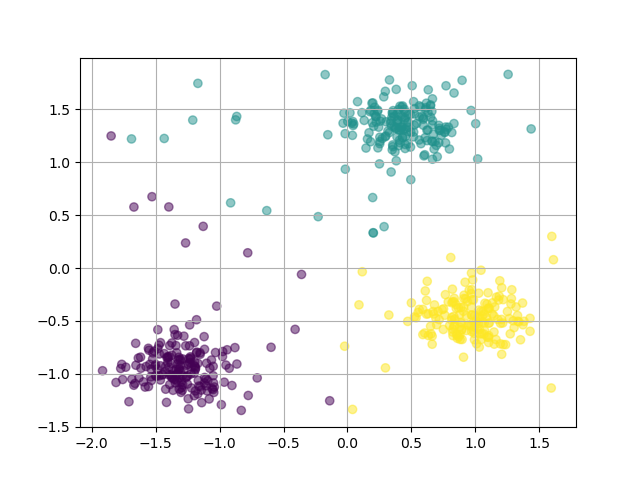} &
        \includegraphics[width=0.2\textwidth]{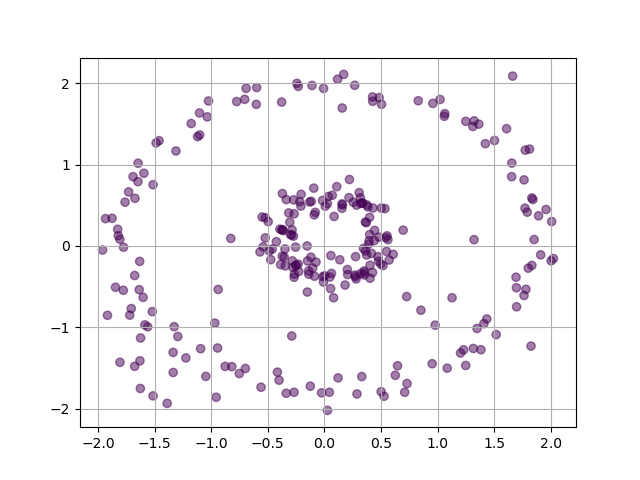} &
        \includegraphics[width=0.2\textwidth]{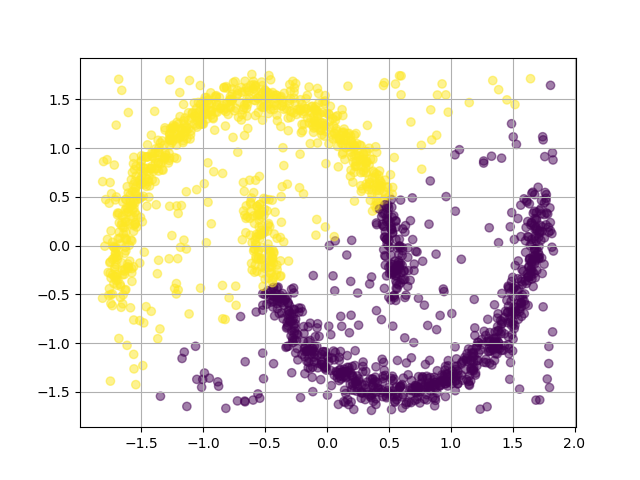} \\
        
        \textbf{MKM} &
        \includegraphics[width=0.2\textwidth]{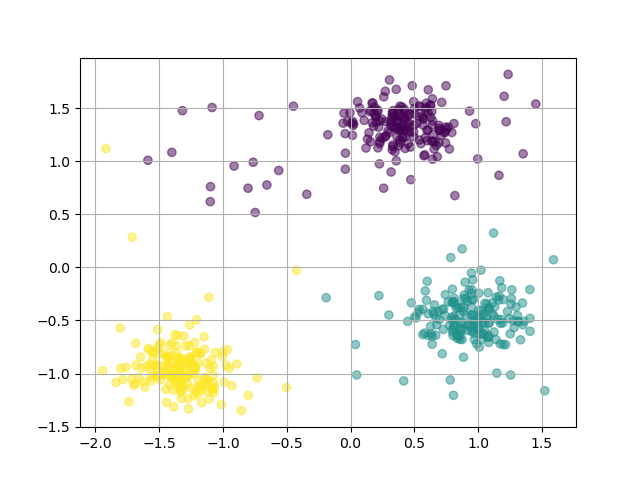} &
        \includegraphics[width=0.2\textwidth]{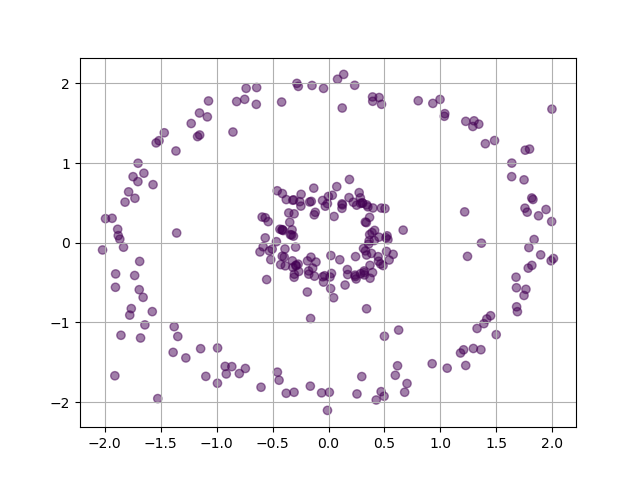} &
        \includegraphics[width=0.2\textwidth]{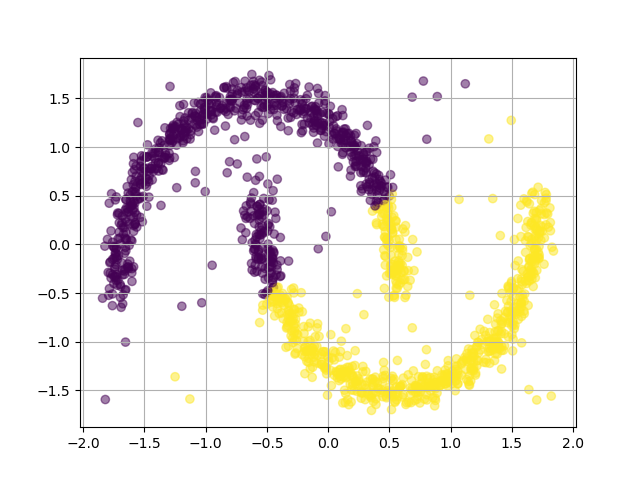} \\
        
        \textbf{CC} &
        \includegraphics[width=0.2\textwidth]{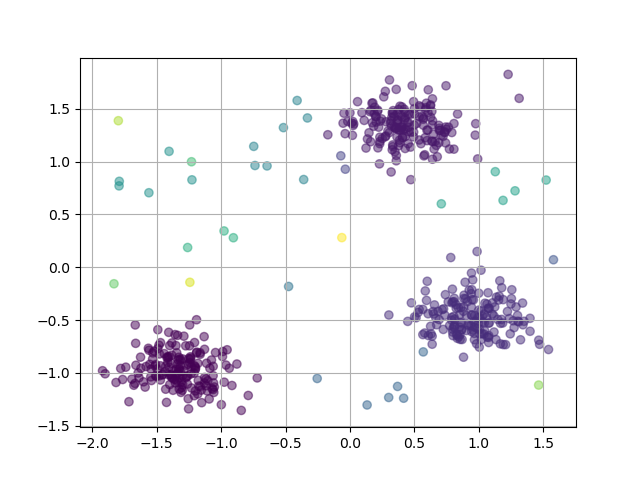} &
        \includegraphics[width=0.2\textwidth]{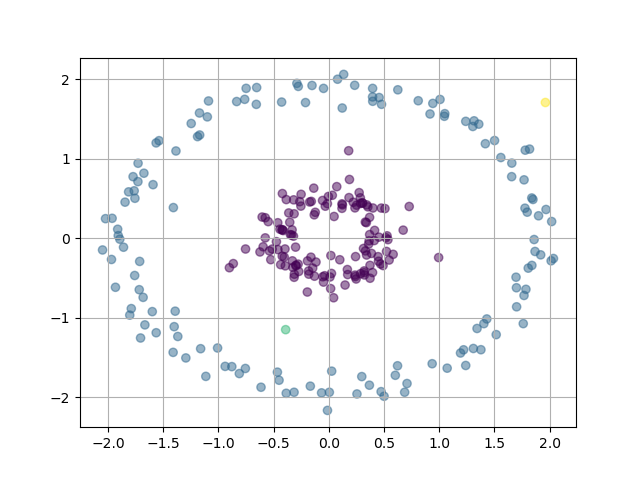} &
        \includegraphics[width=0.2\textwidth]{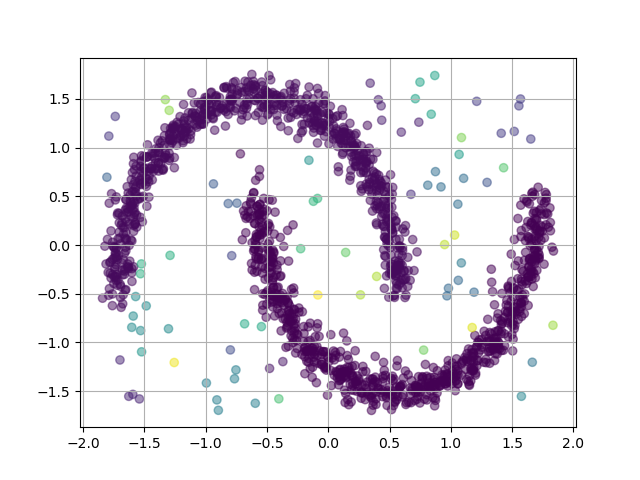} \\

        \textbf{RConv} &
        \includegraphics[width=0.2\textwidth]{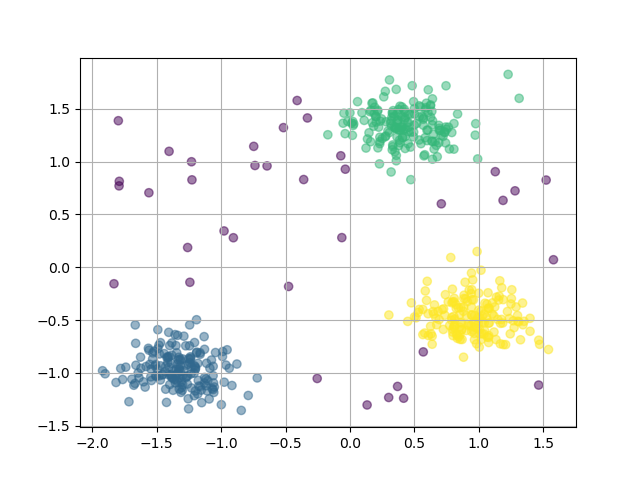} &
        \includegraphics[width=0.2\textwidth]{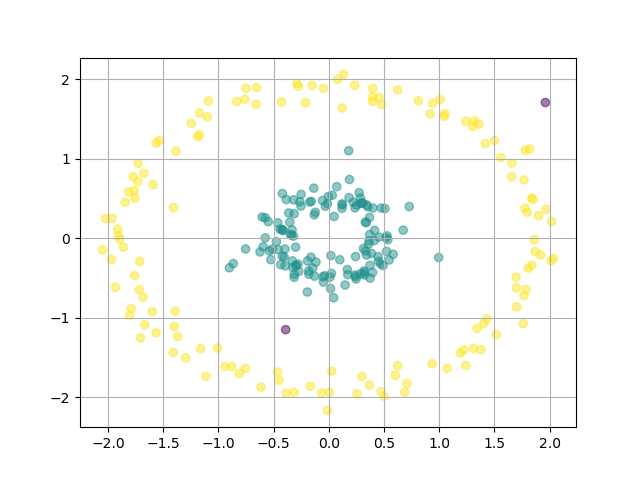} &
        \includegraphics[width=0.2\textwidth]{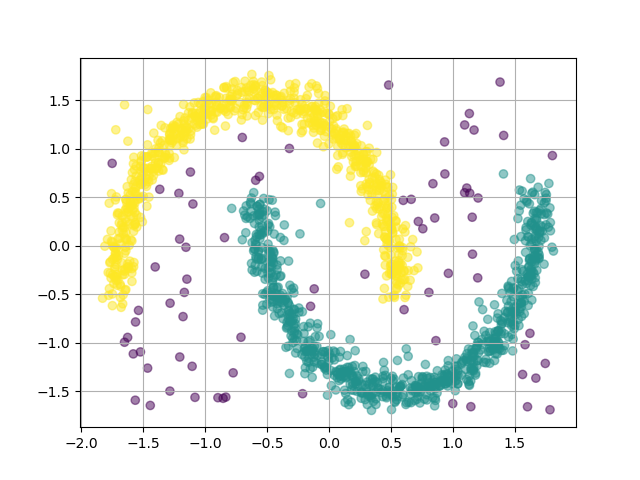} \\

        \textbf{RCC} &
        \includegraphics[width=0.2\textwidth]{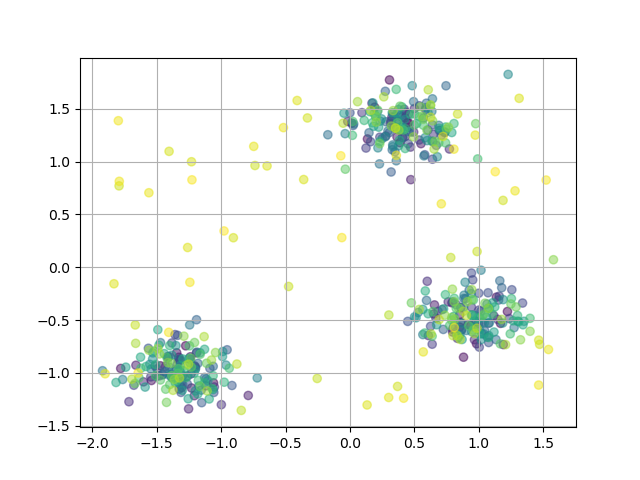} &
        \includegraphics[width=0.2\textwidth]{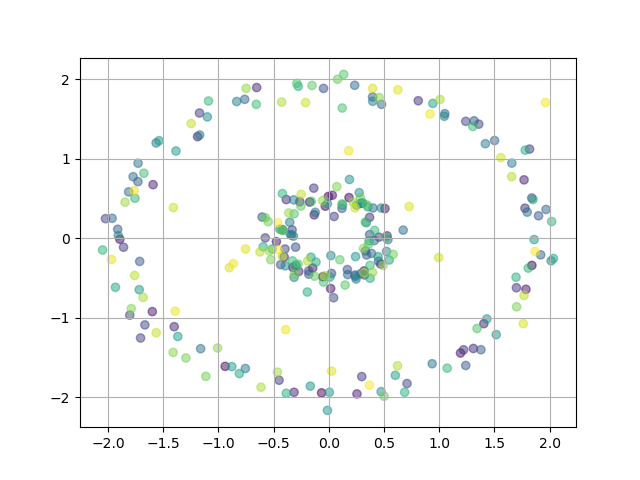} &
        \includegraphics[width=0.2\textwidth]{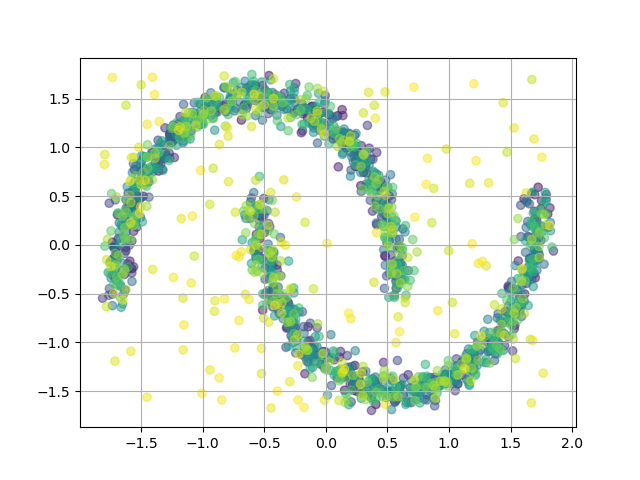} \\

        \textbf{RBKM} &
        \includegraphics[width=0.2\textwidth]{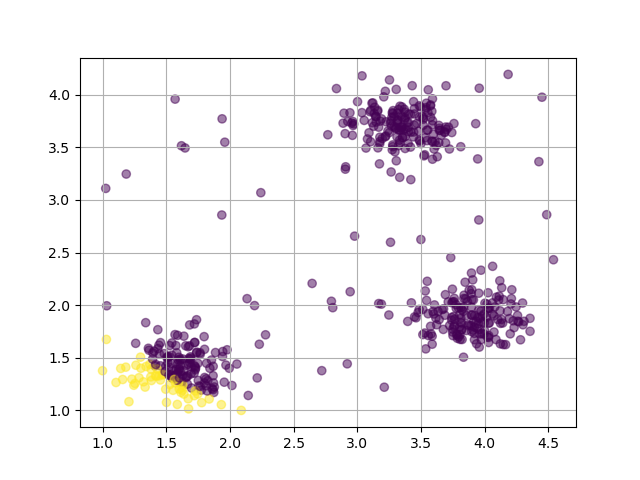} &
        \includegraphics[width=0.2\textwidth]{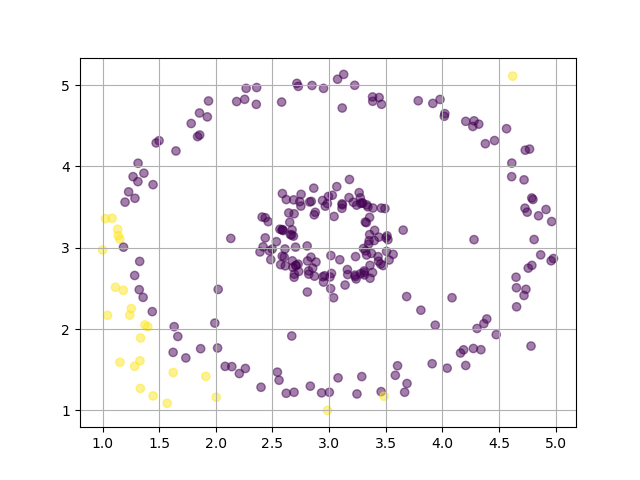} &
        \includegraphics[width=0.2\textwidth]{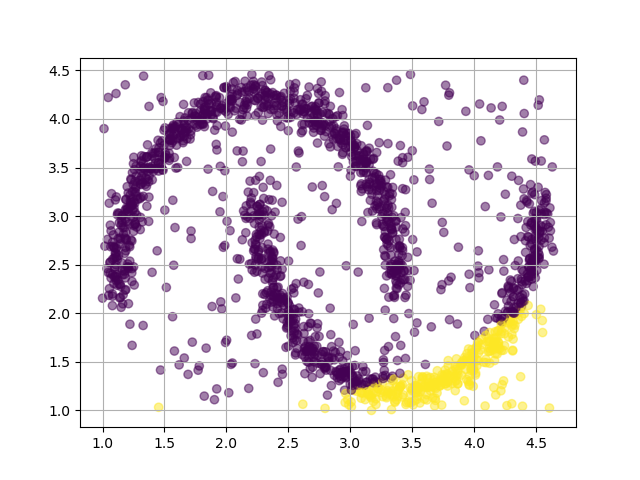} \\
    \end{tabular}
    
    \caption{Clustering results for 7 algorithms across 3 datasets. Each row corresponds to an algorithm, and each column corresponds to a dataset.}
    
    \label{fig:cluster_grid}
\end{figure*}

\end{document}